\documentclass{article}

\usepackage{amsmath,amsfonts,bm}

\def\1{\bm{1}}

\def\vtheta{{\bm{\theta}}}

\def\vx{{\bm{x}}}
\def\vy{{\bm{y}}}

\DeclareMathAlphabet{\mathsfit}{\encodingdefault}{\sfdefault}{m}{sl}
\SetMathAlphabet{\mathsfit}{bold}{\encodingdefault}{\sfdefault}{bx}{n}

\def\gM{{\mathcal{M}}}
\def\gN{{\mathcal{N}}}
\def\gO{{\mathcal{O}}}
\def\gP{{\mathcal{P}}}

\def\sR{{\mathbb{R}}}

\newcommand{\first}[1]{\textbf{\textcolor{red}{#1}}}
\newcommand{\second}[1]{\textbf{\textcolor{violet}{#1}}}
\newcommand{\third}[1]{\textbf{\textcolor{black}{#1}}}

\PassOptionsToPackage{numbers, compress}{natbib}
\usepackage[final]{neurips_2024}

\usepackage[utf8]{inputenc} % allow utf-8 input
\usepackage[T1]{fontenc}    % use 8-bit T1 fonts
\usepackage[hidelinks]{hyperref}       % hyperlinks
\usepackage{url}            % simple URL typesetting
\usepackage{booktabs}       % professional-quality tables
\usepackage{amsfonts}       % blackboard math symbols
\usepackage{nicefrac}       % compact symbols for 1/2, etc.
\usepackage{microtype}      % microtypography
\usepackage{xcolor}         % colors

\usepackage[inline]{enumitem}

\usepackage{caption}
\usepackage{subcaption}

\usepackage{amsmath}
\usepackage{amssymb}
\usepackage{mathtools}
\usepackage{amsthm}
\usepackage{multirow}
\usepackage{dsfont}
\usepackage[capitalize,noabbrev]{cleveref}
\usepackage{thmtools} 
\usepackage{thm-restate}

\theoremstyle{plain}
\newtheorem{theorem}{Theorem}[section]
\newtheorem{proposition}[theorem]{Proposition}

\theoremstyle{definition}
\newtheorem{definition}[theorem]{Definition}

\theoremstyle{remark}
\newtheorem{remark}[theorem]{Remark}

\usepackage{xcolor}
\hypersetup{
    colorlinks,
    linkcolor={red!50!black},
    citecolor={blue!50!black},
    urlcolor={blue!80!black}
}

\usepackage{pdfpages}

\setcitestyle{square,numbers,comma}

\newcommand{\xhdr}[1]{\vspace{0.7mm}\noindent{\bf #1.}}
\newcommand{\ourmethod}{\textsc{DiGRAF}}
\newcommand{\ourmethodGNN}{\textsc{DiGRAF}}
\newcommand{\ourmethodshared}{\textsc{DiGRAF (W/O Adap.)}}

\newcommand{\revision}[1]{{\color{black}#1}}
\usepackage{wrapfig}

\title{\ourmethod: Diffeomorphic Graph-Adaptive %and Expressive 
Activation Function}

\author{%
  Krishna Sri Ipsit Mantri\thanks{Equal contribution.} \\ %Department of Computer Science\\
  Purdue University\\
  \texttt{mantrik@purdue.edu} \\ \And Xinzhi (Aurora) Wang$^{*}$  \\ %Department of Computer Science\\
  Purdue University\\
  \texttt{wang6171@purdue.edu} \\ \And Carola-Bibiane Schönlieb \\   %Department of Applied Mathematics\\
  University of Cambridge\\
  \texttt{cbs31@cam.ac.uk} \\ \And Bruno Ribeiro \\ %Department of Computer Science\\
  Purdue University\\
  \texttt{ribeirob@purdue.edu} \\ \And Beatrice Bevilacqua$^\dagger$ \\   %Department of Computer Science\\
  Purdue University\\
  \texttt{bbevilac@purdue.edu} \\ \And Moshe Eliasof\thanks{Equal supervision.} \\
  University of Cambridge\\
  \texttt{me532@cam.ac.uk} \\
}

\begin{document}

\maketitle

\begin{abstract}
In this paper, we propose a novel activation function tailored specifically for graph data in Graph Neural Networks (GNNs). Motivated by the need for graph-adaptive and flexible activation functions, we introduce \ourmethod, leveraging Continuous Piecewise-Affine Based (CPAB) transformations, which we augment with an additional GNN to learn a graph-adaptive diffeomorphic activation function in an end-to-end manner. In addition to its graph-adaptivity and flexibility, \ourmethod\ also possesses properties that are widely recognized as desirable for activation functions, such as differentiability, boundness within the domain, and computational efficiency. 
We conduct an extensive set of experiments across diverse datasets and tasks, demonstrating a consistent and superior performance of \ourmethod\ compared to traditional and graph-specific activation functions, highlighting its effectiveness as an activation function for GNNs. \revision{Our code is available at \url{https://github.com/ipsitmantri/DiGRAF}.}

\end{abstract}

\small

\section{Introduction}
\label{sec:introduction}
Graph Neural Networks (GNNs) have found application across diverse domains, including social networks, recommendation systems, bioinformatics, and chemical analysis~\citep{wu2022graph, zhang2021graph, reiser2022graph}.
Recent advancements in GNN research have predominantly focused on exploring the design space of key architectural elements, ranging from expressive GNN layers~\citep{morris2023wl,frasca2022understanding,zhang2023complete,zhang2023rethinking,puny2023equivariant}, to pooling layers~\citep{ying2018hierarchical,lee2019self,bianchi2020spectral,wang2020haar}, and positional and structural encodings~\citep{dwivedi2023benchmarking,rampavsek2022recipe,eliasof2023graph}. 
Despite the exploration of these architectural choices, a common trend persists where most GNNs default to employing standard activation functions, such as ReLU~\citep{fukushima1969visual}, among a few others.

Activation functions play a crucial role in neural networks, as they are necessary for modeling non-linear input-output mappings. Importantly, different activation functions exhibit distinct behaviors, and the choice of the activation function can significantly influence the performance of the neural network~\citep{nwankpa2018activation}. 
It is well-known~\cite{opschoor2020deep, schwab2023deep} that, from a theoretical point of view, non-convex and highly oscillatory activation functions offer better approximation power. However, due to their strong non-convexity, they amplify optimization challenges~\cite{jain2017non}.   

Therefore, as a middle-ground between practice and theory, it has been suggested that a successful activation function should possess the following properties:  
\begin{enumerate*}[label=(\arabic*)]
    \item be differentiable everywhere~\citep{dubey2022activation,mishra2021non},
    \item have non-vanishing gradients~\citep{dubey2022activation};
    \item be bounded to improve the training stability~\citep{liew2016bounded,dubey2022activation};
    \item be zero-centered to accelerate convergence~\citep{dubey2022activation}; and
    \item be efficient and not increase the complexity of the neural network~\citep{kunc2024three}.
\end{enumerate*}
In the context of graph data, the activation function should arguably also be what we define as \emph{graph-adaptive}, that is, tailored to the input graph and capable of capturing the unique properties of graph-structured data, such as degree differences or size changes. This adaptivity ensures that the activation function can effectively leverage the structural information present in the graph data, potentially leading to improved performance in graph tasks.

Recent work in graph learning has investigated the impact of activation functions specifically designed for graphs, such as \citet{iancu2020graph} that proposes graph-adaptive max and median activation filters, and \citet{zhang2022graph} that introduces GReLU, which learns piecewise linear activation functions with a graph-adaptive mechanism. Despite the potential demonstrated by these approaches, the proposed activation functions still have predefined fixed structures (max and median functions in \citet{iancu2020graph} and piecewise linear in \citet{zhang2022graph}), restricting the flexibility of the activation functions that can be learned. Additionally, in the case of GReLU, the learned activation functions inherit the drawback of points of non-differentiability, which are undesirable according to the properties mentioned above. As a consequence, to the best of our knowledge, none of the existing activation functions prove to be consistently beneficial across different graph datasets and tasks.
Therefore, \emph{our objective is to design a flexible activation function tailored for graph data, offering consistent performance gains}. This activation function should possess many, if not all, of the properties recognized as beneficial for activation functions, with an emphasis on blueprint flexibility, as well as task and input adaptivity.

\begin{figure}[t]
    \centering \includegraphics[width=\columnwidth]{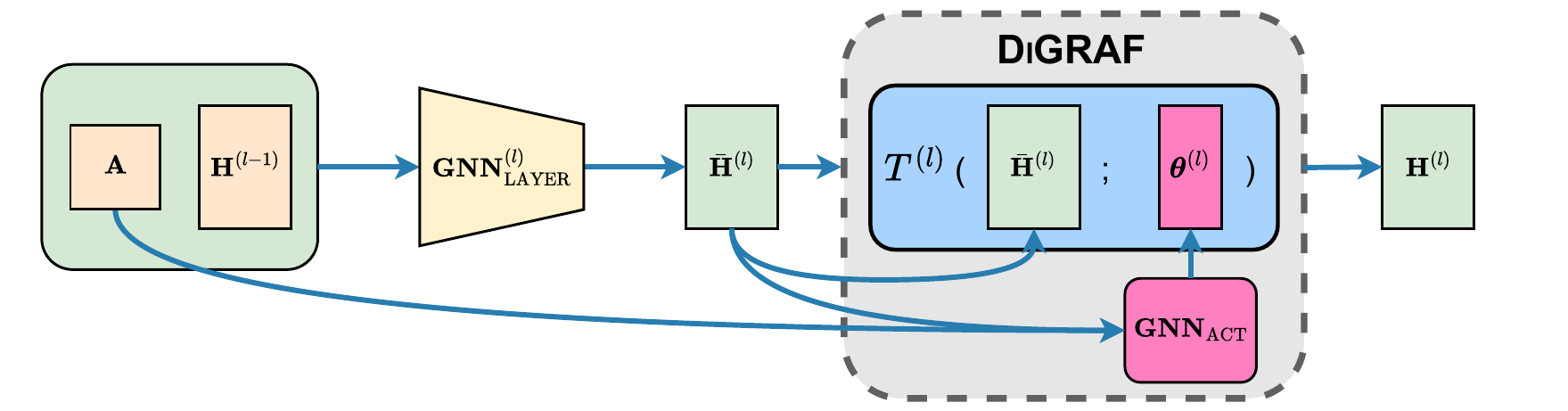}
    \caption{Illustration of \ourmethod. Node features $\mathbf{H}^{(l-1)}$ and adjacency matrix $\mathbf{A}$ are fed to a $\textsc{GNN}_{\textsc{layer}}^{(l)}$ to obtain updated intermediate node features $\bar{\mathbf{H}}^{(l)}$,
    which are passed to our activation function layer, \ourmethod.
    First, an additional GNN network $\textsc{GNN}_{\textsc{act}}$ takes $\bar{\mathbf{H}}^{(l)}$ and $\mathbf{A}$ as input to determine the activation function parameters $\vtheta^{(l)}$. These are used to parameterize the transformation  $T^{(l)}$, which operates on $\bar{\mathbf{H}}^{(l)}$ to produce the activated node features $\mathbf{H}^{(l)}$.}
    \label{fig:structure}
\end{figure}

\paragraph{Our Approach: \ourmethod.} In this paper, we leverage the success of learning diffeomorphisms, particularly through Continuous Piecewise-Affine Based transformations (CPAB)~\citep{freifeld2015highly,freifeld2017transformations}, to devise an activation function tailored for graph-structured data. Diffeomorphisms, characterized as bijective, differentiable, and invertible mappings with a differentiable inverse, inherently possess many desirable properties of activation functions, like differentiability, boundedness within the input-output domain, and stability to input perturbations. 
To augment our activation function with graph-adaptivity, we employ an additional GNN to derive the parameters of the learned diffeomorphism. This integration yields our node permutation equivariant activation function, dubbed \ourmethod\ -- {\bf DI}ffeomorphism-based {\bf GR}aph {\bf A}ctivation {\bf F}unction, illustrated in \Cref{fig:structure}, that dynamically adapts to different graphs, providing a flexible framework capable of learning activation functions for specific tasks and datasets in an end-to-end manner.  This comprehensive set of characteristics positions \ourmethod\ as a promising approach for designing activation functions for GNNs.

To evaluate the efficacy of \ourmethod, we conduct an extensive set of experiments on a diverse set of datasets across various tasks, including node classification, graph classification, and regression. Our evaluation compares the performance of \ourmethod\ with three types of baselines: traditional activation functions, activation functions with trainable parameters, and graph activation functions.
Our experimental results demonstrate that \ourmethod\ repeatedly exhibits better downstream performance than other approaches, reflecting the theoretical understanding and rationale underlying its design and the properties it possesses. Importantly, while existing activation functions offer different behavior in different datasets, \ourmethod\ maintains consistent performance across diverse experimental evaluations, further highlighting its effectiveness.

\paragraph{Main contributions.} The contributions of this work are summarized as follows:
\begin{enumerate*}[label=(\arabic*), leftmargin=*]
    \item We introduce a learnable graph-adaptive activation function based on flexible and efficient diffeomorphisms -- \ourmethod, which we show to have properties advocated in literature;
    \item an analysis of such properties, reasoning about the design choices of our method; and, 
    \item a comprehensive experimental evaluation of \ourmethod\ and other activation functions. 
\end{enumerate*}

\section{Related Work}
\label{sec:relatedwork}
\paragraph{Diffeomorphisms in Neural Networks.} % A way of getting transformation in the computer vision, introduce more in the preliminaries
%\textcolor{red}{first describe what diffeomorphism is, cite a few papers that do it and then move to cpab as a specific framework for learning diffeomorphisms.}
%
% A diffeomorphism is a bijection mapping function that is invertible. Both the function and its inverse are continuously differentiable.  
%
% Early work on diffeomorphism, most of them are based on complicated usually infinity dimensional.
%
% Recent work using Monte Carlo Markov Chain (MCMC), powerful inference tools but owing to their complexity, still challenging (Current representations of highly-flexibility difeomorphism spaces are overly complicated)
%
%
% CPAB takes more practical approach and strat from a finite-dimensional space, in which discretizing the representation is unneeded, while computations require no approximations in the 1D case. Algorithm is linear complexity and are parallelizable yielding sub-linear running times in practice. 
%
A bijection mapping function $f: \gM \to \gN$, given two differentiable manifolds $\gM$ and $\gN$, is termed a \emph{diffeomorphism} if its inverse $f^{-1}: \gN \to \gM$ is also differentiable. The challenge in learning diffeomorphisms arises from their computational complexity: early research is often based on complicated infinite dimensional spaces~\citep{NEURIPS2018_68148596}, and later advancements have turned to Markov Chain Monte Carlo methods, which still suffer from large computational complexity \citep{allassonniere2010construction, allassonniere2015bayesian, zhang2016bayesian}.
To address these drawbacks, \citet{freifeld2015highly, freifeld2017transformations} introduced the Continuous Piecewise-Affine Based transformation (CPAB) approach, offering a more pragmatic solution to learning diffeomorphisms by starting from a finite-dimensional space, 
and allowing for exact diffeomorphism computations in the case of 1D diffeomorphisms -- an essential trait in our case, given that activation functions are 1D functions. CPAB has linear complexity and is parallelizable, which can lead to sub-linear complexity in practice~\citep{freifeld2017transformations}.
Originally designed for alignment and regression tasks by learning diffeomorphisms, in recent years, CPAB was found to be effective in addressing numerous applications using neural networks, posing it as a suitable framework for learning transformation. For instance, \citet{detlefsen2018deep} learns CPAB transformations to improve the flexibility of spatial transformer layers,  \citet{martinez2022closed} combines CPAB with neural networks for temporal alignment, \citet{weber2023regularization} introduces a novel loss function that eliminates the need for CPAB deformation regularization in time-series analysis, and \citet{wang2024continuous} utilizes CPAB to model complex spatial transformation for image animation and motion modeling.

\paragraph{General-Purpose Activation Functions.}
In the last decades, the design of activation functions has seen extensive exploration, resulting in the introduction of numerous high-performing approaches, as summarized in \citet{dubey2022activation,kunc2024three}. The focus has gradually shifted from traditional, static activation functions such as ReLU \citep{fukushima1969visual}, Sigmoid \citep{kunc2024three}, Tanh \citep{hochreiter1997long}, and ELU~\citep{clevert2015fast},  to learnable functions.
In the landscape of learnable activation functions, the Maxout~\citep{pmlr-v28-goodfellow13} unit selects the maximum output from learnable linear functions, and PReLU~\citep{he2015delving} extends ReLU by learning a negative slope. Additionally, the Swish function~\citep{ramachandran2017searching} augments the SiLU function~\citep{elfwing2018sigmoid}, a Sigmoid-weighted linear unit, with a learnable parameter controlling the amount of non-linearity. The recently proposed AdAct~\citep{maiti2024adact} learns a weighted combination of several activation functions, and DiTAC \citep{chelly2024ditac} learns a diffeomorphic activation function for CNNs. However, these activation functions are not input-adaptive, a desirable property in GNNs.

\begin{figure}
     \centering
     \begin{subfigure}[t]{0.47\textwidth}
         \centering
         \includegraphics[width=\textwidth]{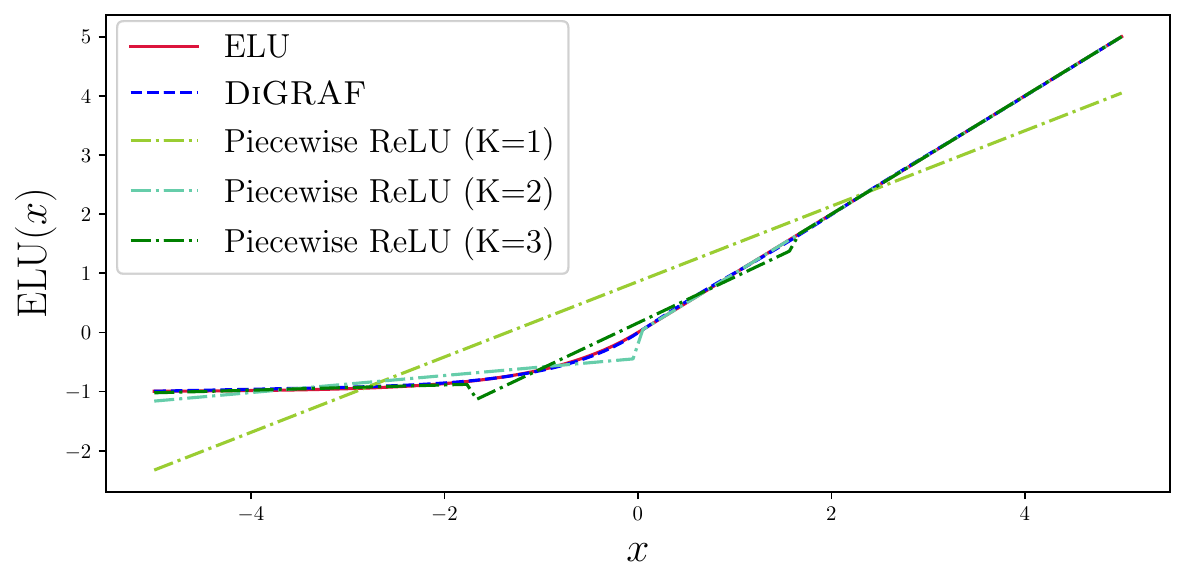}
         \caption{ELU}
         \label{fig:approximate:ELU}
     \end{subfigure}
     \hfill
     \centering
     \begin{subfigure}[t]{0.47\textwidth}
         \centering
         \includegraphics[width=\textwidth]{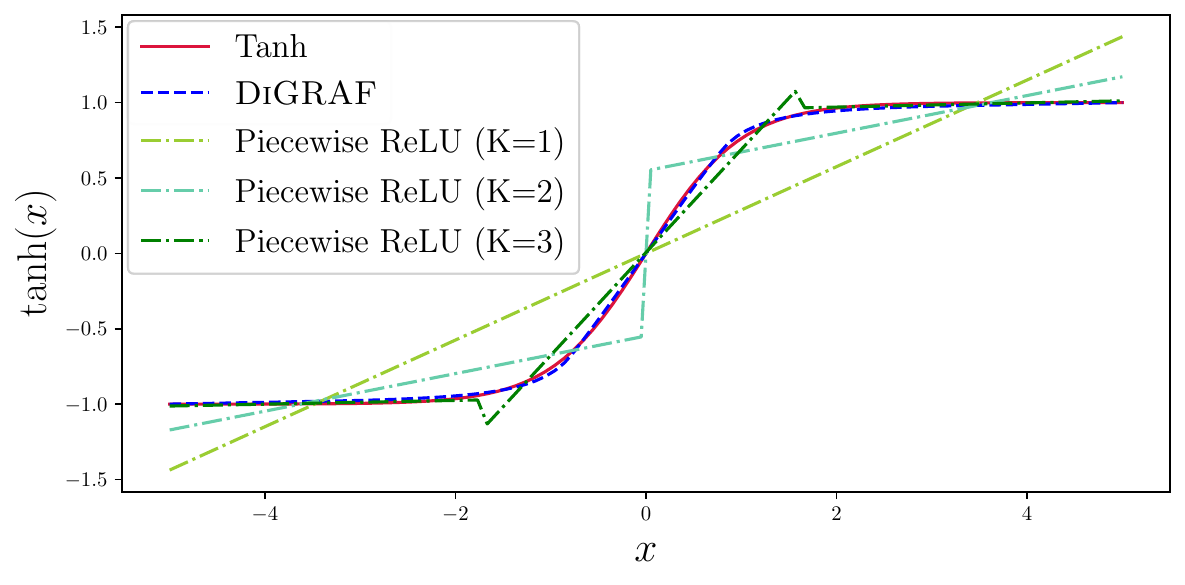}
         \caption{Tanh}
         \label{fig:approximate:tanh}
     \end{subfigure}
        \caption{Approximation of traditional activation functions using CPAB and Piecewise ReLU with varying segment counts $K \in \{1, 2, 3\}$ on a closed interval $\Omega=[-5,5]$, demonstrating the advantage of utilizing CPAB and its flexibility to model various activation functions.}
        \label{fig:approximate:flexibility}
        %\vspace{-2pt}
\end{figure}

\paragraph{Graph Activation Functions.}
Typically, GNNs are coupled with conventional activation functions~\citep{kipf2017semisupervised, velivckovic2017graph, xu2019how}, which were not originally tailored for graph data, graph tasks, or GNN models. \revision{This implies that these activation functions do not inherently adapt to the structure of the input graph, which was found to be an important property in other GNN components, such as graph normalization \citep{eliasof2024granola}.} Recent works have suggested various approaches to bridge this gap. 
Early works such as \citet{scardapane2018improving} propose learning activation functions based on graph kernels, and \citet{iancu2020graph} introduces Max and Median filters, which operate on local neighborhoods in the graph, thereby offering adaptivity to the input graphs. A notable advancement in graph-adaptive activation functions is GReLU~\citep{zhang2022graph}, a parametric piecewise affine activation function achieving graph adaptivity by learning parameters through a hyperfunction that takes into account the node features and the connectivity of the graph.
While these approaches demonstrate the potential to enhance GNN performance compared to standard activation functions, they are constrained by their blueprint, often relying on piecewise  ReLU composition, which can be performance-limiting~\citep {khalife2023power}. Moreover, a fixed blueprint limits flexibility, i.e., the ability to express a variety of functions. As we show in \Cref{fig:approximate:flexibility}, attempts to approximate traditional activation functions such as ELU and Tanh using piecewise ReLU composition with different segment counts ($K = 1$, $2$, and $3$),
reveal limited approximation power. On the contrary, our \ourmethod, which leverages CPAB, yields significantly better approximations. Furthermore, we demonstrate the approximation power of activations learned with the CPAB framework in our \ourmethod\ in \Cref{app:funcapprox}.

\section{Mathematical Background and Notations}
\label{sec:background}

In this paper, we utilize the definitions from CPAB — a framework for efficiently learning flexible diffeomorphisms~\citep{freifeld2015highly, freifeld2017transformations}, alongside basic graph learning notations, to develop activation functions for GNNs. Consequently, this section outlines the essential details needed to understand the foundations of our \ourmethod.

\subsection{CPAB Diffeomorphisms}\label{sec:methods:CPAB_definition}
Let $\Omega = [a, b] \subset \sR$ be a closed interval, where $a < b$. We discretize $\Omega$ using a tessellation $\gP$ with $\gN_\gP$ intervals, which, in practice, is oftentimes an equispaced 1D meshgrid with $\gN_\gP$ segments~\citep{freifeld2017transformations} (see \Cref{sec:appendix:thetaRelations} for a formal definition of tessellation). Our goal in this paper is to learn a diffeomorphism $f:\Omega \to \Omega$ that we will use as an activation function. Formally, a diffeomorphism is defined as follows:

\begin{definition}[Diffeomorphism on a closed interval $\Omega$]\label{sec:def:diffeomorphism}
A diffeomorphism on a closed interval $\Omega \subset \sR$ is any function $f : \Omega \to \Omega$ that is \begin{enumerate*}[label=(\arabic*)]
    \item bijective,
    \item differentiable, and
    \item has a differentiable inverse $f^{-1}$.
\end{enumerate*} 
\end{definition}

To instantiate a CPAB diffeomorphism $f$, we define a continuous piecewise-affine (CPA) velocity field $v^\vtheta$ parameterized by $\vtheta \in \sR^{\gN_\gP-1}$. We display examples of velocity fields $v^\vtheta$ for various instances of $\vtheta$ in \Cref{fig:velocity_field:velocity} to demonstrate the distinct influence of $\vtheta$ on $v^\vtheta$.
Formally, a velocity field $v^\vtheta$ is defined as follows:

\begin{definition}[CPA velocity field $v^{\vtheta}$ on $\Omega$]\label{sec:def:CPA_Velocity}
Given a tessellation $\gP$ with $\gN_\gP$ intervals on a closed domain $\Omega$, any velocity field $v^{\vtheta} : \Omega \to \sR$ is termed continuous and piecewise-affine if (1)  $v^{\vtheta}$ is continuous, and (2) $v^{\vtheta}$ is an affine transformation on each interval of $\gP$.
\end{definition}

The CPA velocity field $v^\vtheta$ defines a differentiable trajectory $\phi^\vtheta(x, t) : \Omega \times \sR \to \Omega$ for each $x \in \Omega$. The trajectories are computed by integrating the velocity field $v^{\vtheta}$  to time $t$, and are used to construct the CPAB diffeomorphism. We visualize the resulting diffeomorphism in \cref{fig:velocity_field:diffeomorphism} with matching colors denoting corresponding pairs of $v^{\vtheta}$ and $f^\vtheta(x)$. 
Mathematically,

\begin{definition}[CPAB Diffeomorphism] \label{sec:def:CPAB_Diffeomorphism}
Given a CPA velocity field $v^{\vtheta}$, the CPAB diffeomorphism $f$ at point $x$, is defined as:
\begin{equation}
\label{eq:cpab_diffeo}
    f^\vtheta(x) \triangleq \phi^\vtheta(x, t=1) 
\end{equation}
such that $\phi^\vtheta(x, t=1)$ solves the integral equation:
\begin{equation}
\label{eq:cpab_integration}
    \phi^\vtheta(x, t) = x + \int\limits_{0}^{t} v^\vtheta\big(\phi^\vtheta(x, \tau)\big)\, d\tau.
\end{equation}
In arbitrary dimensions, computing  \Cref{sec:def:CPAB_Diffeomorphism} required using an ordinary differential equation solver and can be expensive. However, for 1D diffeomorphisms, as in our \ourmethod, there are closed-form solutions to the CPAB diffeomorphism and its gradients  \cite{freifeld2017transformations}, offering an efficient framework for learning activation functions.
\end{definition}

\begin{figure}[t]
     \centering
     \hfill
     \begin{subfigure}[t]{0.61\textwidth}
         \centering
         \captionsetup{margin={116px, 0px}}
         \includegraphics[width=\textwidth]{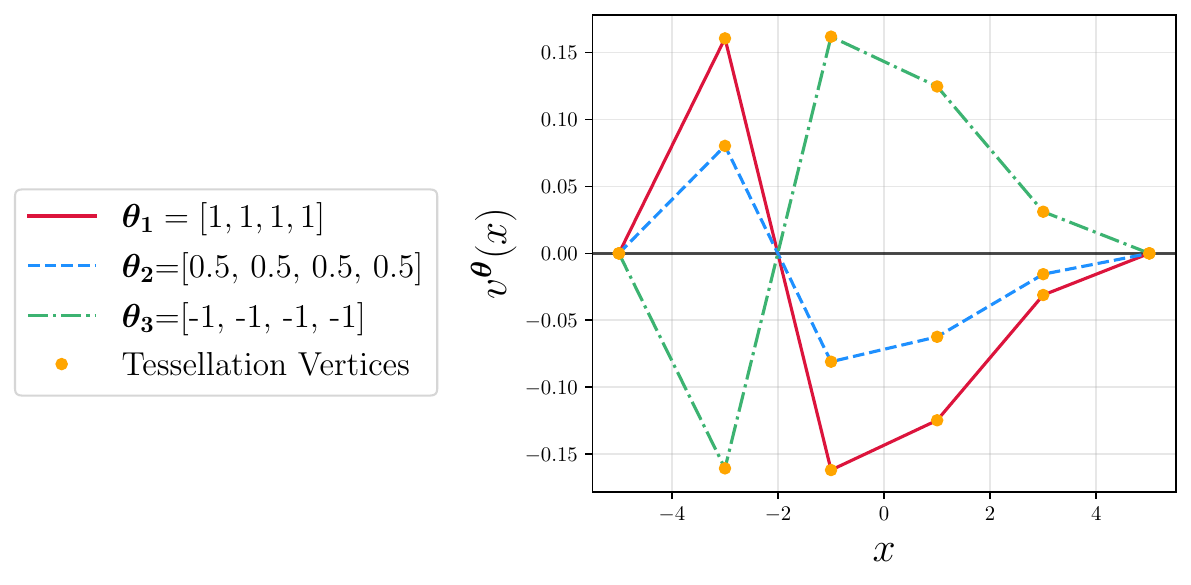}
         \caption{}
         \label{fig:velocity_field:velocity}
     \end{subfigure}
     \hfill
     \centering
     \begin{subfigure}[t]{0.38\textwidth}
         \centering
         \captionsetup{margin={18px, 0px}}
         \includegraphics[width=\textwidth]{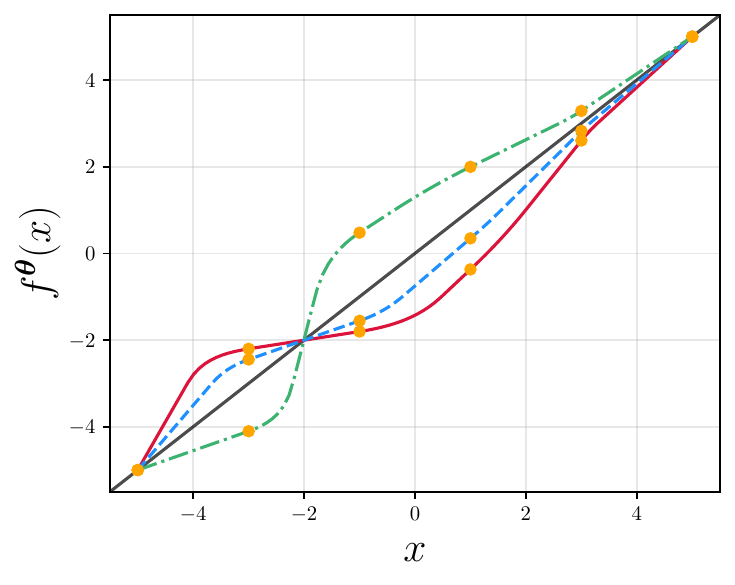}
         \caption{}
         \label{fig:velocity_field:diffeomorphism}
     \end{subfigure}
     \hfill
        \caption{An example of CPA velocity fields $v^{\vtheta}$ defined on the interval $\Omega = [-5, 5]$ with a tessellation $\gP$ consisting of five subintervals. The three different parameters, $\vtheta_1$, $\vtheta_2$, and $\vtheta_3$ define three distinct CPA velocity fields (\Cref{fig:velocity_field:velocity}) resulting in separate CPAB diffeomorphisms $f^\vtheta(x)$ (\cref{fig:velocity_field:diffeomorphism}).}
        \label{fig:velocity_field}
\end{figure}

\subsection{Graph Learning Notations}
Consider a graph $G = (V, E)$ with $N \in \mathbb{N}$ nodes, where $V = \{1, \dots, N\}$ is the set of nodes and $E \subseteq V \times V$ is the set of edges. Let $\mathbf{A} \in \{0, 1\}^{N \times N}$ be the adjacency matrix of $G$, and $\mathbf{X} \in \mathbb{R}^{N \times F}$ the node feature matrix, where $F$ is the number of input features. We denote the feature vector of node $v \in V$ as $\mathbf{x}_v \in \mathbb{R}^F$, which corresponds to the $v$-th row of $\mathbf{X}$. The input node features $\mathbf{X}$ are transformed into the initial node representations $\mathbf{H}^{(0)} \in \mathbb{R}^{N \times C}$, using an embedding function $\mathrm{emb}: \mathbb{R}^F \to \mathbb{R}^{C}$ to $\mathbf{X}$, where $C$ is the hidden dimension, that is
\begin{equation}
\mathbf{H}^{(0)} = \mathrm{emb}(\mathbf{X}).
\end{equation}
The initial features $\mathbf{H}^{(0)}$ are fed to a GNN comprised of  $L \in \mathbb{N}$ layers, where each layer $l \in \{1, \dots, L\}$ is followed by an activation function $\sigma^{(l)}(\cdot; \vtheta^{(l)}): \mathbb{R} \to \mathbb{R}$, and $\vtheta^{(l)}$ is a set of possibly learnable parameters of $\sigma^{(l)}$. Specifically, the intermediate output of the $l$-th GNN layer is denoted as:
\begin{equation}
\label{eq:Model:GNN_output}
\bar{\mathbf{H}}^{(l)} = \textsc{GNN}_{\textsc{layer}}^{(l)}(\mathbf{H}^{(l-1)}, \mathbf{A})
\end{equation}
where $\bar{\mathbf{H}}^{(l)} \in \mathbb{R}^{N \times C}$.
The activation function $\sigma^{(l)}$ is then applied \emph{element-wise} to $\bar{\mathbf{H}}^{(l)}$, yielding node features $h^{(l)}_{u, c} = \sigma^{(l)}(\bar{h}^{(l)}_{u, c}; \vtheta^{(l)})$ $\forall u \in V$, $\forall c \in [C]$ . 
Therefore, the application of $\sigma^{(l)}$ can be equivalently written as:
 \begin{equation}
\label{eq:Model:T_output}
\mathbf{H}^{(l)} = \sigma^{(l)}(\bar{\mathbf{H}}^{(l)}; \vtheta^{(l)}).
\end{equation}
In the following section, we will show how this abstraction is translated to our \ourmethod.

\section{\ourmethod}
\label{sec:methods}

In this section, we formalize our approach, \ourmethod, illustrated in \Cref{fig:structure}, which leverages diffeomorphisms to learn adaptive and flexible graph activation functions.

\subsection{A CPAB Blueprint for Graph Activation Functions}
Our approach builds on the highly flexible \textsc{CPAB} framework~\citep{freifeld2015highly,freifeld2017transformations} and extends it by incorporating Graph Neural Networks (GNNs) to enable the learning of adaptive graph activation functions. While the original \textsc{CPAB} framework was designed for grid deformation and alignment tasks, typically in 1D, 2D, or 3D spaces, we propose a novel application of CPAB in the context of learning activation functions, as described below.

\begin{wrapfigure}[21]{r}{0.5\textwidth}
\vspace{-8pt}
  \centering
    \includegraphics[width=0.45\textwidth]{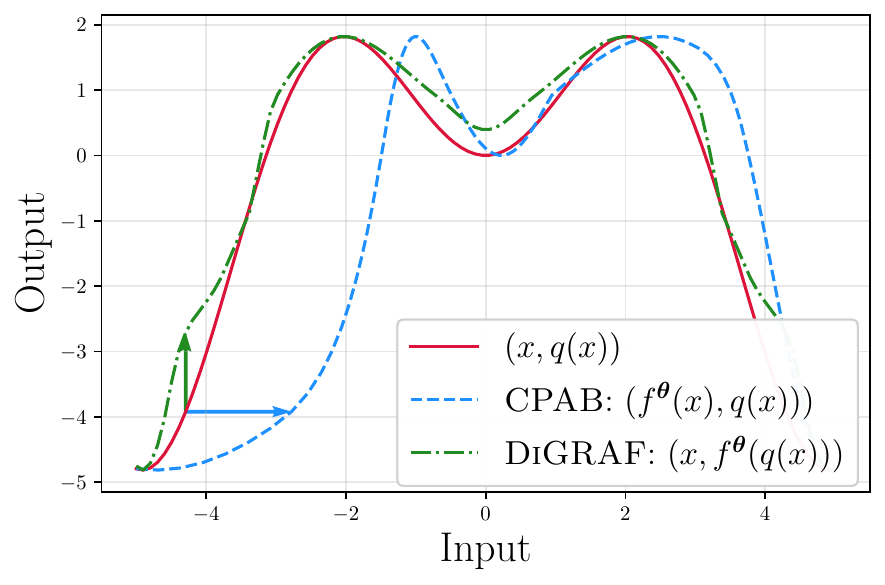}
    \caption{Different transformation strategies. The input function (red), CPAB transformation (blue), and \ourmethod\ transformation (green), within $\Omega = [-5, 5]$ using the same $\vtheta$. While CPAB stretches the input, \ourmethod\ stretches the output, showcasing the distinctive impact of each approach. }
    \label{fig:transform_grid}
\end{wrapfigure}
In \ourmethod, we treat a node feature (single channel) as a one-dimensional (\textsc{1D}) point. Given the node features matrix $\bar{\mathbf{H}} \in \mathbb{R}^{N \times C}$, we apply \ourmethod\ per entry in $\bar{\mathbf{H}}$, in accordance with the typical element-wise computation of activation functions. We mention that, while CPAB was originally designed to learn grid deformations, it can be utilized as an activation function blueprint by considering a conceptual shift that we demonstrate in \Cref{fig:transform_grid}. 
Given an input function (shown in red in the figure),  CPAB deforms grid coordinates, i.e., it transforms it along the horizontal axis, as shown in the blue curve. In contrast, \ourmethod\ transforms the original data points along the vertical axis, resulting in the green curve. This conceptual shift can be seen visually from the arrows showing the different dimensions of transformations. 
We therefore refer to the vertical transformation of the data as their activations. Formally, we define the transformation function $T^{(l)}$ as the element-wise application of the diffeomorphism $f^\vtheta$ from \Cref{eq:cpab_diffeo}:
\begin{equation}
\label{eq:TandPhi}
T^{(l)}(\bar{h}^{(l)}_{u, c}; \vtheta^{(l)}) \triangleq f^{\vtheta^{(l)}}(\bar{h}^{(l)}_{u, c}),
\end{equation}
where $\vtheta^{(l)}$ denotes learnable parameters of the transformation function $T^{(l)}$, that parameterize the underlying CPA velocity field as discussed in  \Cref{sec:background}. In \Cref{sec:methods:models}
, we discuss the learning of  $\vtheta^{(l)}$ in \ourmethod.

The transformation $T^{(l)}: \Omega \to \Omega$ described in \Cref{eq:TandPhi} is based on CPAB and therefore takes as input values within a domain $\Omega=[a,b]$, and outputs a value within that domain, where $a<b$ are hyperparameters. In practice, we take $a=-b$, such that the activation function can be symmetric and centered around 0, a property known to be desirable for activation functions \cite{dubey2022activation}. 
For any entry in the intermediate node features $\bar{\mathbf{H}}^{(l)}$(\Cref{eq:Model:GNN_output}) that is outside the domain $\Omega$, we use the identity function. Therefore, a \ourmethod~activation function reads:
\begin{equation}
    \label{eq:ourmethod}
    \ourmethod(\bar{h}^{(l)}_{u, c}, {\vtheta}^{(l)}) = \begin{cases}
        T^{(l)}(\bar{h}^{(l)}_{u, c}; {{\vtheta}}^{(l)}), & \text{If }    \bar{h}^{(l)}_{u, c}\in \Omega \\
        \bar{h}^{(l)}_{u, c}, & \text{Otherwise}
    \end{cases}
\end{equation}
In practice, \ourmethod\ is applied element-wise in parallel over all entries, and we use the following notation, which yields the output features post the activation of the $l$-th GNN layer:
\begin{equation}
\label{eq:application_digraf}
    \mathbf{H}^{(l)} =  \ourmethod(\bar{\mathbf{H}}^{(l)}, {\vtheta}^{(l)}).
\end{equation}

\subsection{Learning Diffeomorphic Velocity Fields}
\label{sec:methods:models}
\ourmethod, defined in \Cref{eq:ourmethod}, introduces graph-adaptivity into the transformation function $T^{(l)}$ by employing an additional GNN, denoted as $\textsc{GNN}_{\textsc{act}}$, that returns the diffeomorphism parameters $\vtheta^{(l)}$:
\begin{equation}
\label{eqn:digaf}
\vtheta^{(l)}(\bar{\mathbf{H}}^{(l)}, \mathbf{A}) = \textsc{Pool} \left(\textsc{GNN}_{\textsc{act}} (\bar{\mathbf{H}}^{(l)}, \mathbf{A})\right),
\end{equation}
where \textsc{Pool} is a graph-wise pooling operation, such as max or mean pooling.
The resulting vector $\vtheta^{(l)} \in \mathbb{R}^{\gN_\gP-1}$, which is dependent on the tessellation size $\gN_\gP$,  is then used to compute the output of the $l$-th layer, $\mathbf{H}^{(l)}$, as described in \Cref{eq:application_digraf}. We note that \Cref{eqn:digaf} yields a different $\vtheta^{(l)}$ for every input graph and features pair $(\bar{\mathbf{H}}^{(l)}, \mathbf{A})$, which implies the graph-adaptivity of \ourmethod. Furthermore, since $\textsc{GNN}_{\textsc{act}}$ is trained with the other network parameters in an end-to-end fashion, \ourmethod\ is also adaptive to the task of interest.  In \Cref{sec:appendix:implementation}, we provide and discuss the implementation details of $\textsc{GNN}_{\textsc{act}}$ and \textsc{Pool}.

\paragraph{Variants of \ourmethod.} \Cref{eqn:digaf} describes an approach to introduce graph-adaptivity to $\vtheta^{(l)}$ using $\textsc{GNN}_{\textsc{act}}$. An alternative approach is to directly optimize the parameters $\vtheta^{(l)} \in \mathbb{R}^{\gN_\gP -1}$, without using an additional GNN. Note that in this case, input and graph-adaptivity are sacrificed in favor of a computationally lighter solution. We denote this variant of our method by \ourmethodshared. Considering this variant is important because it allows us to: (i) offer a middle-ground solution in terms of computational effort, and (ii) it allows us to directly quantify the contribution of graph-adaptivity in \ourmethod. In \Cref{sec:experiments}, we compare the performance of the methods. 
\paragraph{Velocity Field Regularization.}
To ensure the smoothness of the velocity field, which will encourage training stability \cite{weber2023regularization}, we incorporate a regularization term in the learning procedure of $\vtheta^{(l)}$. Namely, we follow the Gaussian smoothness prior on the CPA velocity field from \citet{freifeld2015highly}, which was shown to be effective in maintaining smooth transformations. The regularization term is defined as follows: 
\begin{equation}     
\label{eqn:cpa_regularization}
\mathcal{R}(\{\vtheta^{(l)}\}_{l=1}^L) = \sum_{l=1}^{L}\vtheta^{{(l)}^\top} \Sigma_{\textsc{CPA}}^{-1} \vtheta^{(l)},
\end{equation}
 where $\Sigma_{\textsc{CPA}}$ represents the covariance of a zero-mean Gaussian smoothness prior defined as in~\citet{freifeld2015highly}. 
We further maintain the boundedness of $\vtheta^{(l)}$ by employing a hyperbolic tangent function (Tanh). In this way, $\vtheta^{(l)}$ remains in $[-1, 1]$ when applied in $T^{(l)}$ in \Cref{eq:ourmethod}, ensuring that the velocity field parameters remain bounded, encouraging the overall training stability of the model.

\subsection{Properties of \ourmethod}
\label{sec:theoritical_analysis}
In this section, we focus on understanding the theoretical properties of \ourmethod, highlighting the compelling attributes that establish it as a performant activation function for GNNs.

\textbf{\ourmethod\ yields differentiable activations.} By construction, \ourmethod\ learns a diffeomorphism, which is differentiable by definition. Being differentiable everywhere is considered beneficial as it allows for smooth weight updates during backpropagation, preventing the zigzagging effect in the optimization process~\citep{szandala2021review}. 

\textbf{\ourmethod\ is bounded within the input-output domain $\Omega$.} 
% The diffeomorphism $T(\cdot; \vtheta)$ in \ourmethod\ inherently maintains the boundedness of the input-output domain $\Omega$, ensuring that the transformed values remain within the specified range. In \cref{sec:remark:T_within_Omega}, we pointed out that $T(\cdot; \vtheta)$ is a $\Omega \to \Omega$ transformation~\citep{freifeld2015highly,freifeld2017transformations}.
We point out in \Cref{sec:remark:T_within_Omega} that the diffeomorphism $T^{(l)}(\cdot; \vtheta^{(l)})$ is a $\Omega \to \Omega$ transformation. Any diffeomorphism is continuous, and by the extreme value theorem, $T^{(l)}(\cdot; \vtheta^{(l)})$ is bounded in $\Omega$. This prevents the activation values from becoming excessively large, a property linked to faster convergence~\citep{dubey2022activation}.

\textbf{\ourmethod\ can learn to be zero-centered.} Benefiting from its flexibility, \ourmethod\ has the capacity to learn activation functions that are inherently zero-centered. As an input-adaptive activation function governed by a parameters vector $\vtheta^{(l)}$, \ourmethod\ can be adjusted through $\vtheta^{(l)}$ to maintain a zero-centered nature. This property is associated with accelerated convergence in neural network training \citep{dubey2022activation}. 

\textbf{\ourmethod\ is efficient.} 
\ourmethod\ exhibits 
linear computational complexity, and can further achieve sub-linear running times via parallelization in practice \cite{freifeld2017transformations}. Moreover, with the existence of a closed-form solution for $f^{\vtheta^{(l)}}$ and its gradient in the 1D case~\citep{freifeld2015highly}, the computations of CPAB can be done efficiently.
Additionally, the measured runtimes, detailed in \Cref{sec:app:computations}, underscore the complexity comparability of \ourmethod\ with other graph activation functions.

In addition to the above properties, which follow from our design choice of learning diffeomorphisms through the CPAB framework, we briefly present the following properties, which are formalized and proven in \Cref{sec:appendix:proofs}.

\textbf{\ourmethod\ is permutation equivariant.} We demonstrate in \Cref{theo:permutataion} that \ourmethod\ exhibits permutation equivariance to node numbering, ensuring that its behavior remains consistent regardless of the ordering of the graph nodes, which is a key desired property in designing GNN components \cite{bronstein2021geometric}. %holds true for points located both within and outside the defined domain, underscoring the robustness and adaptability of our method. 

\textbf{\ourmethod\ is Lipschitz continuous.} We show in \Cref{prop:lipschitz_digraf} that \ourmethod\ is Lipschitz continuous and derive its Lipschitz constant. Since it is also bounded, we can combine the two results, which leads us to the following proposition: %We show in \Cref{prop:digraf_bounded} that the maximal difference of \ourmethod\ is bounded within the given domain $\Omega$, by proving that $T(\cdot; \vtheta)$ in \ourmethod\ is Lipschitz continuous within $\Omega$. Notably, \ourmethod\ preserves the identity transformation outside the domain, maintaining a $1$-Lipschitz property. We can organize the statement formally as:

\begin{restatable}[The boundedness of $T(\cdot; \vtheta^{(l)})$ in \ourmethod]{proposition}{Boundness}
    \label{prop:digraf_bounded} 
    Given a bounded domain $\Omega=[a,b] \subset \mathbb{R}$ where $a<b$, and any two arbitrary points $x, y \in \Omega$, the maximal difference of a diffeomorphism $T(\cdot; \vtheta^{(l)})$ with parameter $\vtheta^{(l)}$ in \ourmethod\ is bounded as follows:
    \begin{equation}
        \label{eq:boundednessDigraf}
        |T(x; \vtheta^{(l)}) - T(y; \vtheta^{(l)})| \leq \min(|b-a|, |x-y|\exp(C_{v^{\vtheta^{(l)}}}))
    \end{equation}
    where $C_{v^{\vtheta^{(l)}}}$ is the Lipschitz constant of the CPA velocity field $v^{\vtheta^{(l)}}$.
\end{restatable}

\textbf{\ourmethod\ extends commonly used activation functions.} CPAB~\citep{freifeld2015highly, freifeld2017transformations}, which is used as a framework to learn the diffeomorphism in \ourmethod, is capable of learning and representing a wide range of diffeomorphic functions. When used as an activation function, the transformation $T^{(l)}(\cdot; \vtheta^{(l)})$ in \ourmethod\ adapts to the specific graph and task by learning different $\vtheta^{(l)}$ parameters, rather than having a fixed diffeomorphism. Examples of popular and commonly used diffeomorphisms utilized as activations include Sigmoid, Tanh, Softplus, and ELU, as we show in \Cref{sec:appendix:proofs}. Extending this approach is our \ourmethod\, that learns the diffeomorphism during training rather than selecting a pre-defined function. 

\section{Experiments}
\label{sec:experiments}
In this section, we conduct an extensive set of experiments to demonstrate the effectiveness of \ourmethod\ as a graph activation function. 
Our experiments seek to address the following questions:
\begin{enumerate}[label=(Q\arabic*),leftmargin=*]
    \item Does \ourmethod\ consistently improve the performance of GNNs compared to existing activation functions on a broad set of downstream tasks?
    \item To what extent is graph-adaptivity in \ourmethod\ beneficial when compared to our baseline of \ourmethodshared\ and existing activation functions that lack adaptivity? 
    \item Compared with other graph-adaptive activation functions, how does the added flexibility offered by \ourmethod\ impact downstream performance?
    \item How do the considered activation functions compare in terms of training convergence?
\end{enumerate}

\xhdr{Baselines} We compare \ourmethod\ with three categories of relevant and competitive baselines:
\begin{enumerate*}[label=(\arabic*),leftmargin=*]
    \item \emph{Standard Activation Functions}, namely Identity, \revision{Sigmoid~\citep{rumelhart1986learning}}, ReLU~\citep{fukushima1969visual}, LeakyReLU~\citep{Maas2013RectifierNI}, Tanh~\citep{hochreiter1997long}, GeLU~\citep{hendrycks2023gaussian}, and ELU~\citep{clevert2016fast} to estimate the benefit of learning activation functions parameters;
    \item \emph{Learnable Activation Functions}, specifically PReLU~\citep{he2015delving}, Maxout~\citep{pmlr-v28-goodfellow13} and Swish~\citep{ramachandran2017searching}, to assess the value of graph-adaptivity; and 
    \item \emph{Graph Activation Functions}, such as Max~\citep{iancu2020graph}, Median~\citep{iancu2020graph} and GReLU~\citep{zhang2022graph}, to evaluate the effectiveness of \ourmethod's design in capturing graph structure and the blueprint flexibility of \ourmethod\ as discussed in \Cref{sec:methods}.
\end{enumerate*}

All baselines are integrated into GCN~\citep{kipf2017semisupervised} for node tasks and GIN~\citep{xu2019how} (GINE \citep{hu2020strategies} where edge features are available) for graph tasks, to ensure fair and meaningful comparisons, isolating the impact of other design choices.
We provide additional details on the experimental settings and datasets in \Cref{sec:appendix:experiment}, as well as additional experiments, including ablation studies, in \Cref{sec:app:additional_results}.

\begin{table*}[t]
    \centering
    \scriptsize
    \caption{Comparison of node classification accuracy (\%) $\uparrow$ on different datasets using various baselines with \ourmethod. The top three methods are marked by \textbf{\textcolor{red}{First}}, \textbf{\textcolor{violet}{Second}}, \textbf{Third}.}
    \scriptsize
    \label{tab:GReLU}
    \begin{tabular}{l  ccccc }
        \toprule
        Method $\downarrow$ / Dataset $\rightarrow$ & 
        \textsc{Blog Catalog} &
        \textsc{Flickr}  &
        \textsc{CiteSeer} &
        \textsc{Cora} &
        \textsc{PubMed} \\
        \midrule  

        \textbf{\textsc{Standard Activations}} \\
               $\,$  \textsc{GCN} + Identity & 74.8$\pm$0.5 &
        53.5$\pm$1.1 &
        \third{69.1$\pm$1.6} &
        80.5$\pm$1.2 &
        77.6$\pm$2.1\\
        $\,$\textsc{GCN} + Sigmoid~\citep{rumelhart1986learning} & 39.7$\pm$4.5 & 18.3$\pm$1.2 & 27.9$\pm$2.1 & 32.1$\pm$2.3 & 52.8$\pm$6.6 \\
        $\,$       \textsc{GCN} + ReLU \citep{kipf2017semisupervised} & 
        72.1$\pm$1.9 &
        50.7$\pm$2.3 &
        67.7$\pm$2.3 & 
        79.2$\pm$1.4 & 
        77.6$\pm$2.2 \\
        $\,$ \textsc{GCN} + LeakyReLU~\citep{Maas2013RectifierNI}        & 
        72.6$\pm$2.1 & 51.0$\pm$2.0 & 68.4$\pm$1.8 & 79.4$\pm$1.6 & 76.8$\pm$1.6 \\
        $\,$ \textsc{GCN} + Tanh~\citep{hochreiter1997long}        & 
        73.9$\pm$0.5 &
        51.3$\pm$1.5 &
        \third{69.1$\pm$1.4} &
        80.5$\pm$1.3 & \third{77.9$\pm$2.1} \\
        $\,$ \textsc{GCN} + GeLU~\citep{hendrycks2023gaussian}        & 
        75.8$\pm$0.5 &
        \third{56.1$\pm$1.3} &
        67.8$\pm$1.7 &
        79.3$\pm$1.9 & 77.1$\pm$2.7 \\
        $\,$ \textsc{GCN} + ELU~\citep{clevert2016fast} & 74.8$\pm$0.5 & 53.4$\pm$1.1 & \third{69.1$\pm$1.7} &
        \third{80.7$\pm$1.2} & 77.5$\pm$2.2\\
        \midrule
    \textbf{\textsc{Learnable Activations}} \\
            $\,$ \textsc{GCN} + PReLU~\citep{he2015delving}     & 
74.8$\pm$0.4 & 
53.2$\pm$1.5 &
\second{69.2$\pm$1.5}  & 
80.5$\pm$1.2   & 77.6$\pm$2.1 \\
        $\,$ \textsc{GCN} + Maxout~\citep{pmlr-v28-goodfellow13}   & 
         72.4$\pm$1.4 & 54.0$\pm$1.8 & 68.5$\pm$2.2  & 79.8$\pm$1.5     & 77.3$\pm$2.9 \\
         $\,$ \textsc{GCN} + Swish~\citep{ramachandran2017searching} & \third{76.0$\pm$0.7} & 55.7$\pm$1.4 & 67.7$\pm$1.8 & 79.2$\pm$1.1 & 77.3$\pm$2.8 \\
        \midrule
        \textbf{\textsc{Graph Activations}} \\
        $\,$ \textsc{GCN} + Max \citep{iancu2020graph} & 72.0$\pm$1.0 & 47.5$\pm$0.9 & 59.7$\pm$2.9 & 76.0$\pm$1.8 & 75.0$\pm$1.4 \\
        $\,$ \textsc{GCN} + Median~\citep{iancu2020graph} & \second{77.7$\pm$0.7} & \second{58.3$\pm$0.6} & 61.3$\pm$2.7 & 77.1$\pm$1.1 & 75.7$\pm$2.5 \\
        $\,$       \textsc{GCN} + GReLU \citep{zhang2022graph}
        &
        73.7$\pm$1.2 &
        54.4$\pm$1.6 & 
        68.5$\pm$1.9 &
        \second{81.8$\pm$1.8}& 
        \second{78.9$\pm$1.7} \\
         \midrule
          \textsc{GCN} + \ourmethodshared\ & 
          80.8$\pm$0.6 &
          68.6$\pm$1.8 &
          69.2$\pm$2.1 &
         81.5$\pm$1.1 & 78.3$\pm$1.6 \\
          \textsc{GCN} + \ourmethod\ & 
          \first{81.6$\pm$0.8} &
          \first{69.6$\pm$0.6} &
          \first{69.5$\pm$1.4} &
          \first{82.8$\pm$1.1}& \first{79.3$\pm$1.4}  \\

        \bottomrule

    \end{tabular}
    %}
\end{table*}

\subsection{Node Classification}
\label{sec:node_class}

Our results are summarized in~\Cref{tab:GReLU}, where we consider the \textsc{BlogCatalog}~\citep{Yang_2023},  \textsc{Flickr}~\citep{Yang_2023}, \textsc{CiteSeer}~\citep{sen2008collective}, \textsc{Cora}~\citep{mccallum2000automating}, and \textsc{PubMed}~\citep{namata2012query} datasets. As can be seen from the Table,
\ourmethod\ consistently outperforms all standard activation functions, as well as all the learnable activation functions. 
Additionally, \ourmethod\  outperforms other graph-adaptive activation functions. We attribute this positive performance gap to the ability of \ourmethod\ to learn complex non-linearities due to its diffeomorphism-based blueprint, compared to piecewise linear or pre-defined functions as in other methods.
Finally, we compare the performance of  \ourmethod\ and \ourmethodshared.  
We remark that in this experiment, we are operating in a transductive setting, as the data consists of a single graph, implying that both \ourmethod\ and \ourmethodshared\ are adaptive in this case. Still, we see that \ourmethod\ slightly outperforms the \ourmethodshared\, and we attribute this performance gain to the GNN layers within \ourmethod\, that are (i) explicitly graph-aware, and (ii) can facilitate the learning of better diffeomorphism parameters  $\vtheta^{(l)}$ (\Cref{eqn:digaf}) due to the added complexity.

\subsection{Graph Classification and Regression}
\label{sec:graph_level_task}
\begin{wraptable}[21]{r}{0.4\textwidth}
\centering
\vspace{-12pt}
\caption{Comparison on \textsc{ZINC-12k} under the \textsc{500k} parameter budget. The top three methods are \textbf{\textcolor{red}{First}}, \textbf{\textcolor{violet}{Second}}, \textbf{Third}.}
\scriptsize
\begin{tabular}{l c}
    \toprule
        Method & \textsc{ZINC (MAE $\downarrow$)} \\
        \midrule
                        \textbf{\textsc{Standard Activations}} \\
             $\,$ \textsc{GIN} + Identity & 0.2460$\pm$0.0214\\
             $\,$\textsc{GIN} + Sigmoid~\citep{rumelhart1986learning} & 0.3839$\pm$0.0058 \\
        $\,$ \textsc{GIN} + ReLU~\citep{xu2019how}        & \third{0.1630$\pm$0.0040} \\
                $\,$ \textsc{GIN} + LeakyReLU~\citep{Maas2013RectifierNI}        & 0.1718$\pm$0.0042\\
        $\,$ \textsc{GIN} + Tanh~\citep{hochreiter1997long}        & 0.1797$\pm$0.0064\\
        $\,$ \textsc{GIN} + GeLU~\citep{hendrycks2023gaussian}        & 0.1896$\pm$0.0023\\
        $\,$ \textsc{GIN} + ELU~\citep{clevert2016fast} & 0.1741$\pm$0.0089\\
        \midrule
                \textbf{\textsc{Learnable Activations}} \\
            $\,$ \textsc{GIN} + PReLU~\citep{he2015delving}     & 0.1798 $\pm$0.0067\\
        $\,$ \textsc{GIN} + Maxout~\citep{pmlr-v28-goodfellow13}  & \second{0.1587$\pm$0.0057}\\ 
        $\,$ \textsc{GIN} + Swish~\citep{ramachandran2017searching} & 0.1636$\pm$0.0039\\
        \midrule\textbf{\textsc{Graph Activations}} \\
        $\,$ \textsc{GIN} + Max~\citep{iancu2020graph} & 0.1661$\pm$0.0035\\
        $\,$ \textsc{GIN} + Median~\citep{iancu2020graph} & 0.1715$\pm$0.0050\\
        $\,$ \textsc{GIN} + GReLU \cite{zhang2022graph} & 0.3003$\pm$0.0086 \\
  
        \midrule

                  \textsc{GIN} + \ourmethodshared\ & {0.1382$\pm$0.0080} \\

        \textsc{GIN} + \ourmethod\ & \textbf{\textcolor{red}{0.1302$\pm$0.0090}} \\
        \bottomrule
    \end{tabular}

    \label{tab:zinc}
\end{wraptable}

\textbf{\textsc{ZINC-12k}}. In \Cref{tab:zinc} we present results on the \textsc{ZINC-12k}~\citep{ZINCdataset,G_mez_Bombarelli_2018,dwivedi2023benchmarking} dataset for the regression of constrained solubility of molecules. We note that \ourmethodGNN\ achieves an MAE of $0.1302$, surpassing the best-performing activation on this dataset, Maxout, by $0.0285$, which translates to a relative improvement of $\sim 18\%$.

\textbf{OGB.} We evaluate \ourmethod\ on 4 datasets from the OGB benchmark \citep{hu2020ogb}, namely, \textsc{molesol}, \textsc{moltox21}, \textsc{molbace}, and \textsc{molhiv}. The results are reported in \Cref{tab:ogb}, where it is noted that \ourmethod\ achieves significant improvements compared to standard, learnable, and graph-adaptive activation functions. For instance, \ourmethod\ obtains a ROC-AUC score of $80.28$\% on \textsc{molhiv}, an absolute improvement of $4.7$\% over the best performing activation function (ReLU).

\textbf{TUDatasets.}. In addition to the aforementioned datasets, we evaluate \ourmethod\ on the popular TUDatasets \citep{morris2020tudataset}.  We present results on \textsc{MUTAG}, \textsc{PTC}, \textsc{PROTEINS}, \textsc{NCI1} and \textsc{NCI109} in \Cref{tab:tud_datasets} in \Cref{sec:app:additional_results}.
The results show that \ourmethod\ is always within the top-three performing activations across all datasets. 
As an example, on \textsc{PROTEINS} dataset, we see an absolute improvement of $1.1\%$ over the best-performing activation functions (Maxout and GReLU).

\begin{table*}[t]
\centering
\scriptsize
%\footnotesize
%\scriptsize
\caption{A comparison of \ourmethod\ to natural baselines, standard, and graph activation layers on OGB datasets, demonstrating the advantage of our approach. The top three methods are marked by \textbf{\textcolor{red}{First}}, \textbf{\textcolor{violet}{Second}}, \textbf{Third}.
}
\begin{tabular}{l  c c c c}
    \toprule
        \multirow{2}*{Method $\downarrow$ / Dataset $\rightarrow$} & \textsc{molesol} & \textsc{moltox21} & \textsc{molbace} &\textsc{molhiv} \\
        & \textsc{RMSE $\downarrow$} & \textsc{ROC-AUC $\uparrow$} &\textsc{ROC-AUC $\uparrow$} &\textsc{ROC-AUC $\uparrow$} \\
        \midrule
                \textbf{\textsc{Standard Activations}} \\
         $\,$ \textsc{GIN} + Identity & 1.402$\pm$0.036 & 74.51$\pm$0.44 & 72.69$\pm$2.93 & 75.12$\pm$0.77\\
         $\,$ \textsc{GIN} + Sigmoid~\citep{rumelhart1986learning} & 0.884$\pm$0.043 & 69.15$\pm$0.52 & 68.70$\pm$3.68 & 73.87$\pm$0.80 \\
        $\,$ \textsc{GIN} + ReLU~\citep{xu2019how}        & 1.173$\pm$0.057 & 74.91$\pm$0.51& 72.97$\pm$4.00 & \second{75.58$\pm$1.40} \\
                       $\,$ \textsc{GIN} + LeakyReLU~\citep{Maas2013RectifierNI}       & 1.219$\pm$0.055 & 74.60$\pm$1.10 & 73.40$\pm$3.19 &
                       74.75$\pm$1.20\\
        $\,$ \textsc{GIN} + Tanh~\citep{hochreiter1997long}  & 1.190$\pm$0.044 & 74.93$\pm$0.61  & 74.92$\pm$2.47 &\third{75.22}$\pm$2.03\\
        $\,$ \textsc{GIN} + GeLU~\citep{hendrycks2023gaussian}        & 1.147$\pm$0.050 & 74.29$\pm$0.59 & 75.59$\pm$3.32 & 74.15$\pm$0.79\\
        $\,$ \textsc{GIN} + ELU~\citep{clevert2016fast} & 1.104$\pm$0.038 & 75.08$\pm$0.62 & 76.10$\pm$3.29 & 75.09$\pm$0.65 \\
        \midrule
                \textbf{\textsc{Learnable Activations}} \\
        $\,$ \textsc{GIN} + PReLU~\citep{he2015delving}     & \third{1.098$\pm$0.062} & 74.51$\pm$0.92 & 76.16$\pm$2.28 &
        73.56$\pm$1.63\\
        $\,$ \textsc{GIN} + Maxout~\citep{pmlr-v28-goodfellow13} & 1.109$\pm$0.045 & 75.14$\pm$0.87 & 76.83$\pm$3.88  &
        72.75$\pm$2.10\\ 
        $\,$ \textsc{GIN} + Swish~\citep{ramachandran2017searching} & 1.113$\pm$0.066 & 73.31$\pm$1.01 & \third{77.23$\pm$2.35} & 72.95$\pm$0.64 \\
        \midrule
        
        \textbf{\textsc{Graph Activations}} \\
        $\,$ \textsc{GIN} + Max~\citep{iancu2020graph} & 1.199$\pm$0.070 & \second{75.50$\pm$0.77} & 77.04$\pm$2.81 & 73.44$\pm$2.08\\
        $\,$ \textsc{GIN} + Median~\citep{iancu2020graph} &\second{1.049$\pm$0.038}& 74.39$\pm$0.90  & \second{77.26$\pm$2.74} & 72.80$\pm$2.21 \\
         $\,$ \textsc{GIN} + GReLU~\citep{zhang2022graph} & 1.108$\pm$0.066 & \third{75.33$\pm$0.51} & 75.17$\pm$2.60 & 73.45$\pm$1.62 \\

        \midrule

         \textsc{GIN} + \ourmethodshared\ & {0.9011$\pm$0.047} & 76.37$\pm$0.49 & 78.90$\pm$1.41 & 79.19$\pm$1.36 \\

         \textsc{GIN} + \ourmethod\ & \first{0.8196$\pm$0.051} & \first{77.03$\pm$0.59} & \first{80.37$\pm$1.37} & \first{80.28$\pm$1.44} \\
        \bottomrule
\end{tabular}
\label{tab:ogb}
\end{table*}

\subsection{Convergence Analysis}
\label{sec:results_convergence}

Besides improved downstream performance, another important aspect of activation functions is their contribution to training convergence \citep{dubey2022activation}. We therefore present the training curves of \ourmethod\ as well as the rest of the considered baselines to gain insights into their training convergence. Results for representative datasets are presented in Figure \ref{fig:convergence}, where \ourmethod\ achieves similar or better training convergence than other methods, while also demonstrating better generalization abilities due to its better performance.

\begin{figure}[t]
    \centering
    \begin{subfigure}[b]{.32\linewidth}
        \centering
        \includegraphics[width=\linewidth]{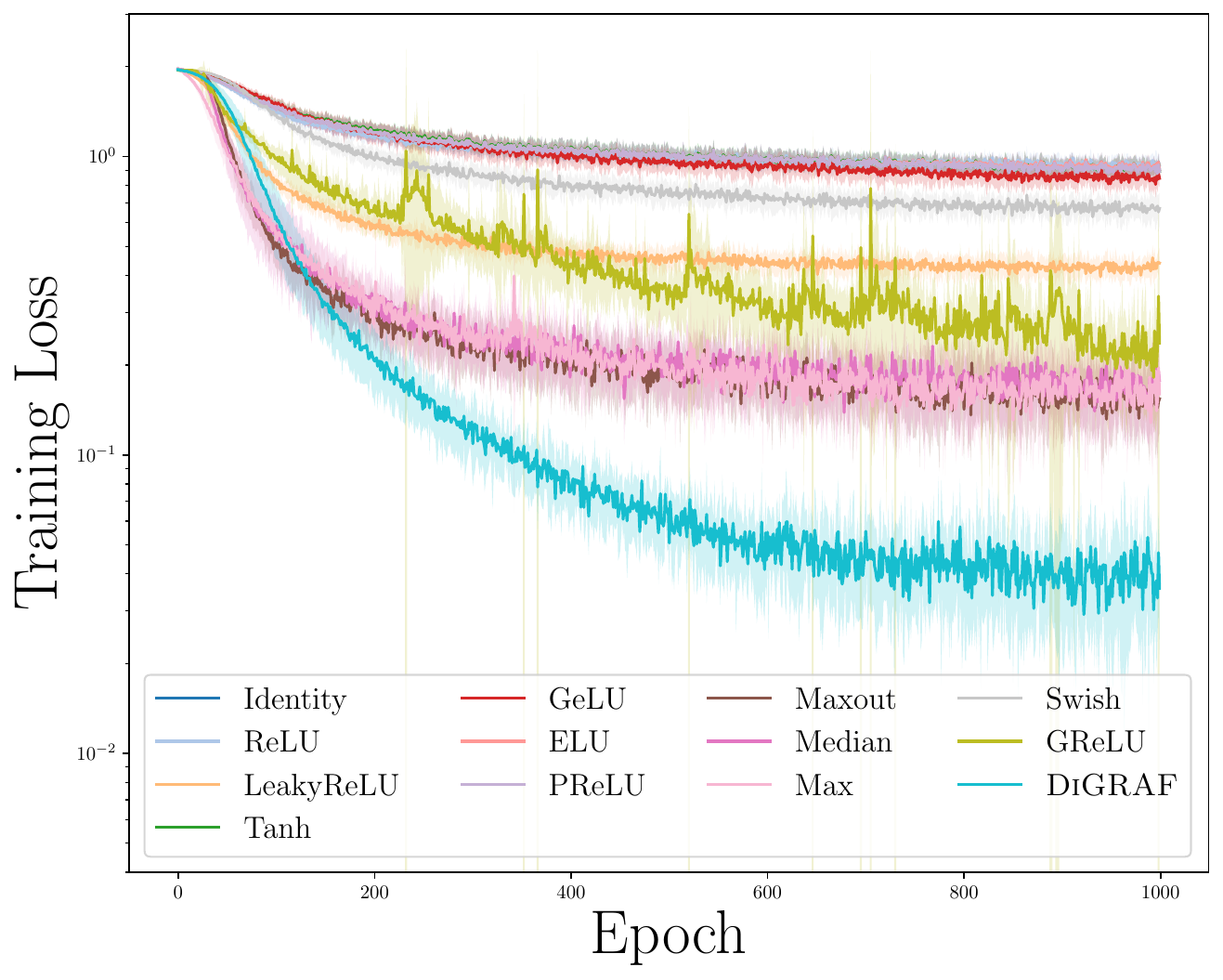}
        \caption{$\textsc{Cora}$}
        \label{fig:convergence_blog}
    \end{subfigure}
    \hfill
    \begin{subfigure}[b]{.32\linewidth}
        \centering
        \includegraphics[width=\linewidth]{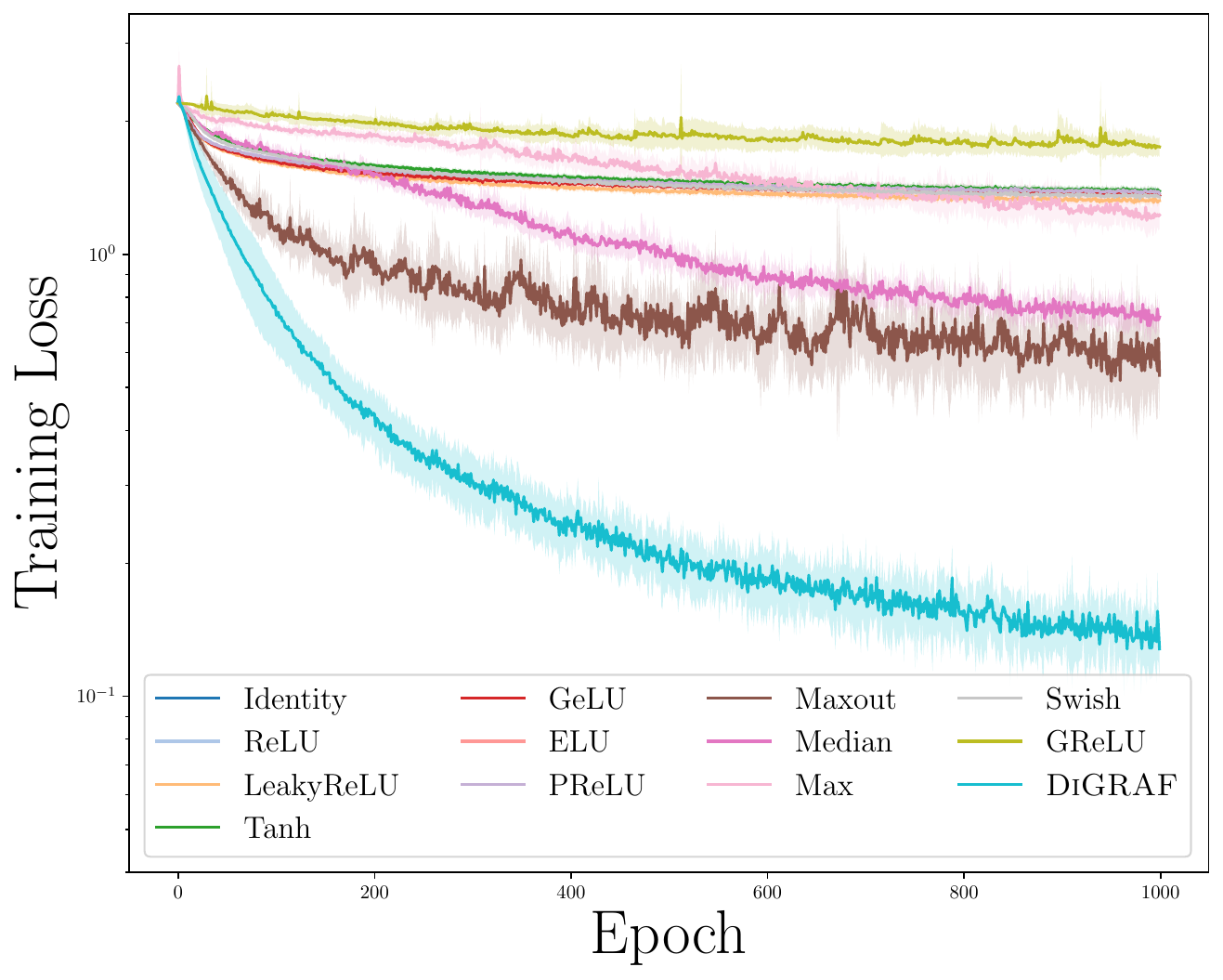}
        \caption{$\textsc{Flickr}$}
        \label{fig:convergence_flickr}
    \end{subfigure}
    \hfill
    \begin{subfigure}[b]{.32\linewidth}
        \centering
        \includegraphics[width=\linewidth]{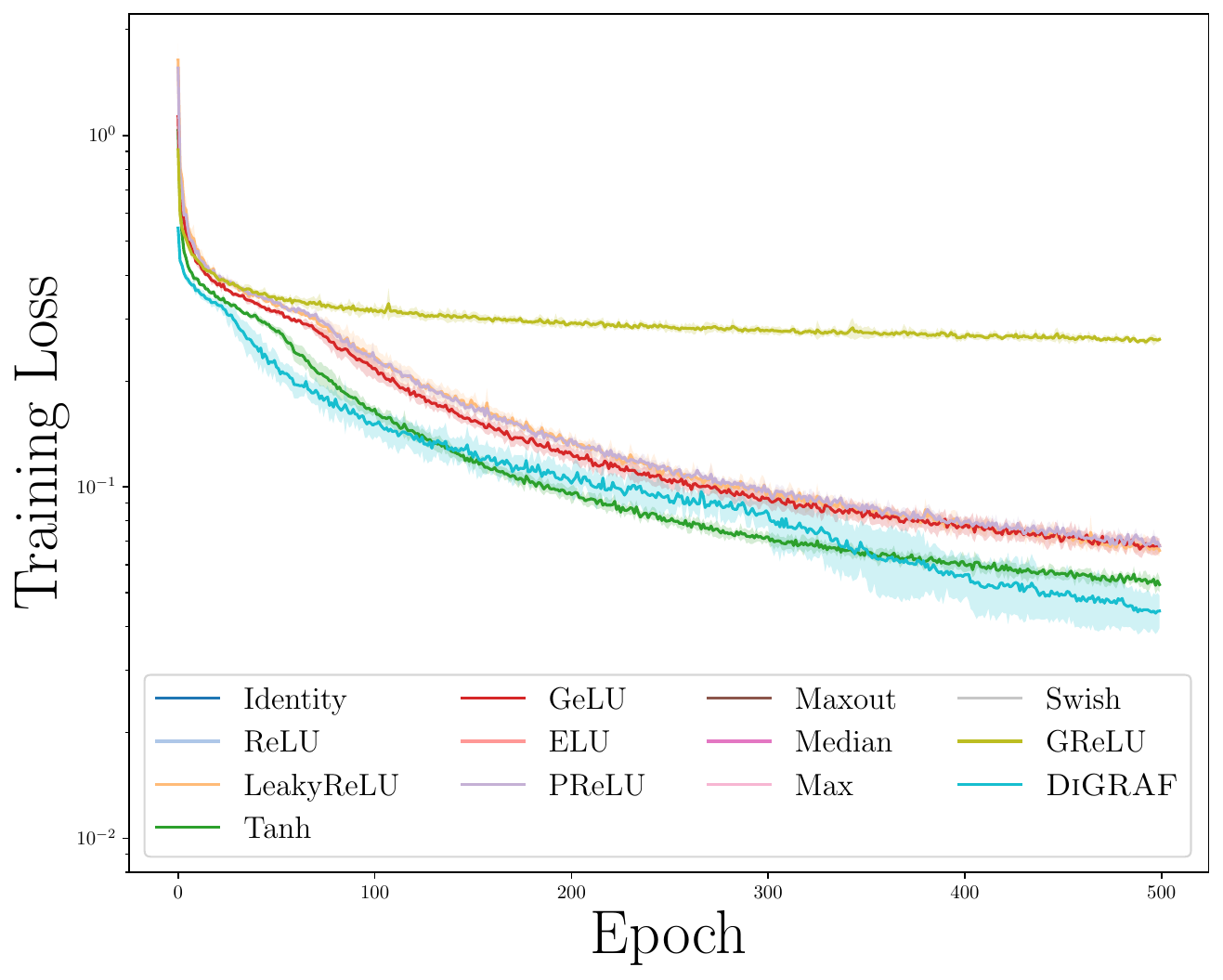}
        \caption{$\textsc{ZINC-12k}$}
        \label{fig:convergence_zinc}
    \end{subfigure}
    \hfill
    \caption{Convergence analysis of \ourmethod\ compared to baseline activation functions. The plot illustrates the training loss over epochs, showcasing the overall faster convergence of \ourmethod. 
    }
    \label{fig:convergence}
    \vspace{-10pt}
\end{figure}

\subsection{Discussion}
Our extensive experiments span across
15 different datasets and benchmarks, consisting of both node- and graph-level tasks. Our key takeaways are as follows:

\begin{enumerate}[label=(A\arabic*),leftmargin=*]
    \item \textbf{Overall Performance of \ourmethod:} The performance offered by \ourmethod\ is consistent and on par with or better than other activation functions, across all datasets. These results establish \ourmethod\ as a highly effective approach for learning graph activation functions.

    \item \textbf{Benefit of Graph-Adaptivity:} \ourmethod\  outperforms the learnable (although not graph-adaptive) activation functions such as PReLU, Maxout, and Swish, as well as our non-graph adaptive baseline \ourmethodshared, on all considered datasets.  
    This observation highlights the crucial role of graph-adaptivity in activation functions for GNNs.

    \item \textbf{The Benefit of Blueprint Flexibility:} \ourmethod\ consistently outperforms other graph-adaptive activation functions like Max, Median, and GReLU. 
    We tie this positive performance gap to the ability of \ourmethod\ to model complex non-linearities due to its diffeomorphism-based blueprint, compared to piecewise linear or pre-defined functions as in other methods.

    \item \textbf{Convergence of \ourmethod:} As shown in \Cref{sec:results_convergence}, in addition to overall better downstream  performance, \ourmethod\ allows to achieve better training convergence.
\end{enumerate}

In summary, compared with 12 well-known activation functions used in GNNs, and across multiple datasets and benchmarks, \ourmethod\ demonstrates a superior, learnable, flexible, and versatile graph-adaptive activation function,  highlighting it as a strong approach for designing and learning graph activation functions.

%\vspace{-5pt}
\section{Conclusions}
\label{sec:conclusion}
In this work, we introduced \ourmethod, a novel activation function designed for graph-structured data. Our approach leverages Continuous Piecewise-Affine Based (CPAB) transformations to integrate a graph-adaptive mechanism, allowing \ourmethod\ to adapt to the unique structural features of input graphs. 
We show that \ourmethod\ exhibits several desirable properties for an activation function, including differentiability, boundedness within a defined interval, and computational efficiency. Furthermore, we demonstrated that \ourmethod\ maintains stability under input perturbations and is permutation equivariant, therefore suitable for graph-based applications.
Our extensive experiments on diverse datasets and tasks demonstrate that \ourmethod\ consistently outperforms traditional, learnable, and existing graph-specific activation functions.

\paragraph{Limitations and Broader Impact.} 
While \ourmethod\ demonstrates consistent superior performance compared to existing activation functions, there remain areas for potential improvement. For instance, the current formulation is limited to learning activation functions that belong to the class of diffeomorphisms, which, despite encompassing a wide range of functions, might not be optimal. 
By improving the performance on real-world tasks like molecule property prediction, and offering faster training convergence, we envision a positive societal impact by \ourmethod\ in drug discovery and 
in achieving a lower carbon footprint.

\subsubsection*{Acknowledgments}
\revision{BR acknowledges support from the National Science Foundation (NSF) awards, CCF-1918483, CAREER IIS-1943364 and CNS-2212160, Amazon Research Award, AnalytiXIN, and the Wabash Heartland Innovation Network (WHIN), Ford, NVidia, CISCO, and Amazon. Computing infrastructure was supported in part by CNS-1925001 (CloudBank). This work was supported in part by AMD under the AMD HPC Fund program. ME is funded by the Blavatnik-Cambridge fellowship, the Cambridge Accelerate Programme for Scientific Discovery, and the Maths4DL EPSRC Programme.  The authors thank Shahaf Finder, Ron Shapira-Weber, and Oren Freifeld for the discussions on CPAB.} 

\bibliography{example_paper}

\begin{thebibliography}{93}
\providecommand{\natexlab}[1]{#1}
\providecommand{\url}[1]{\texttt{#1}}
\expandafter\ifx\csname urlstyle\endcsname\relax
  \providecommand{\doi}[1]{doi: #1}\else
  \providecommand{\doi}{doi: \begingroup \urlstyle{rm}\Url}\fi

\bibitem[Allassonni{\`e}re et~al.(2010)Allassonni{\`e}re, Kuhn, and Trouv{\'e}]{allassonniere2010construction}
St{\'e}phanie Allassonni{\`e}re, Estelle Kuhn, and Alain Trouv{\'e}.
\newblock Construction of bayesian deformable models via a stochastic approximation algorithm: a convergence study.
\newblock \emph{Bernoulli}, 2010.

\bibitem[Allassonni{\`e}re et~al.(2015)Allassonni{\`e}re, Durrleman, and Kuhn]{allassonniere2015bayesian}
St{\'e}phanie Allassonni{\`e}re, Stanley Durrleman, and Estelle Kuhn.
\newblock Bayesian mixed effect atlas estimation with a diffeomorphic deformation model.
\newblock \emph{SIAM Journal on Imaging Sciences}, 8\penalty0 (3):\penalty0 1367--1395, 2015.

\bibitem[Apicella et~al.(2021)Apicella, Donnarumma, Isgr{\`o}, and Prevete]{apicella2021survey}
Andrea Apicella, Francesco Donnarumma, Francesco Isgr{\`o}, and Roberto Prevete.
\newblock A survey on modern trainable activation functions.
\newblock \emph{Neural Networks}, 138:\penalty0 14--32, 2021.

\bibitem[Bevilacqua et~al.(2023)Bevilacqua, Eliasof, Meirom, Ribeiro, and Maron]{bevilacqua2023efficient}
Beatrice Bevilacqua, Moshe Eliasof, Eli Meirom, Bruno Ribeiro, and Haggai Maron.
\newblock Efficient subgraph gnns by learning effective selection policies.
\newblock \emph{arXiv preprint arXiv:2310.20082}, 2023.

\bibitem[Bianchi et~al.(2020)Bianchi, Grattarola, and Alippi]{bianchi2020spectral}
Filippo~Maria Bianchi, Daniele Grattarola, and Cesare Alippi.
\newblock Spectral clustering with graph neural networks for graph pooling.
\newblock In \emph{International conference on machine learning}, pages 874--883. PMLR, 2020.

\bibitem[Biewald(2020)]{wandb}
Lukas Biewald.
\newblock Experiment tracking with weights and biases, 2020.
\newblock URL \url{https://www.wandb.com/}.
\newblock Software available from wandb.com.

\bibitem[Brody et~al.(2022)Brody, Alon, and Yahav]{brody2022how}
Shaked Brody, Uri Alon, and Eran Yahav.
\newblock How attentive are graph attention networks?
\newblock In \emph{International Conference on Learning Representations}, 2022.

\bibitem[Bronstein et~al.(2021)Bronstein, Bruna, Cohen, and Veli{\v{c}}kovi{\'c}]{bronstein2021geometric}
Michael~M Bronstein, Joan Bruna, Taco Cohen, and Petar Veli{\v{c}}kovi{\'c}.
\newblock Geometric deep learning: Grids, groups, graphs, geodesics, and gauges.
\newblock \emph{arXiv preprint arXiv:2104.13478}, 2021.

\bibitem[Chelly et~al.(2024)Chelly, Finder, Ifergane, and Freifeld]{chelly2024ditac}
Irit Chelly, Shahaf~E Finder, Shira Ifergane, and Oren Freifeld.
\newblock Trainable highly-expressive activation functions.
\newblock In \emph{European Conference on Computer Vision}, 2024.

\bibitem[Chen et~al.(2020)Chen, Wei, Huang, Ding, and Li]{chen2020simple}
Ming Chen, Zhewei Wei, Zengfeng Huang, Bolin Ding, and Yaliang Li.
\newblock Simple and deep graph convolutional networks.
\newblock In \emph{International conference on machine learning}, pages 1725--1735. PMLR, 2020.

\bibitem[Clevert et~al.(2015)Clevert, Unterthiner, and Hochreiter]{clevert2015fast}
Djork-Arn{\'e} Clevert, Thomas Unterthiner, and Sepp Hochreiter.
\newblock Fast and accurate deep network learning by exponential linear units (elus).
\newblock \emph{arXiv preprint arXiv:1511.07289}, 2015.

\bibitem[Clevert et~al.(2016)Clevert, Unterthiner, and Hochreiter]{clevert2016fast}
Djork-Arné Clevert, Thomas Unterthiner, and Sepp Hochreiter.
\newblock Fast and accurate deep network learning by exponential linear units (elus), 2016.

\bibitem[Daubechies et~al.(2022)Daubechies, DeVore, Foucart, Hanin, and Petrova]{daubechies2022nonlinear}
Ingrid Daubechies, Ronald DeVore, Simon Foucart, Boris Hanin, and Guergana Petrova.
\newblock Nonlinear approximation and (deep) relu networks.
\newblock \emph{Constructive Approximation}, 55\penalty0 (1):\penalty0 127--172, 2022.

\bibitem[De~Ryck et~al.(2021)De~Ryck, Lanthaler, and Mishra]{de2021approximation}
Tim De~Ryck, Samuel Lanthaler, and Siddhartha Mishra.
\newblock On the approximation of functions by tanh neural networks.
\newblock \emph{Neural Networks}, 143:\penalty0 732--750, 2021.

\bibitem[Detlefsen et~al.(2018)Detlefsen, Freifeld, and Hauberg]{detlefsen2018deep}
Nicki~Skafte Detlefsen, Oren Freifeld, and S{\o}ren Hauberg.
\newblock Deep diffeomorphic transformer networks.
\newblock In \emph{Proceedings of the IEEE conference on computer vision and pattern recognition}, pages 4403--4412, 2018.

\bibitem[Dubey et~al.(2022)Dubey, Singh, and Chaudhuri]{dubey2022activation}
Shiv~Ram Dubey, Satish~Kumar Singh, and Bidyut~Baran Chaudhuri.
\newblock Activation functions in deep learning: A comprehensive survey and benchmark.
\newblock \emph{Neurocomputing}, 503:\penalty0 92--108, 2022.

\bibitem[Dwivedi et~al.(2022)Dwivedi, Luu, Laurent, Bengio, and Bresson]{dwivedi2022graph}
Vijay~Prakash Dwivedi, Anh~Tuan Luu, Thomas Laurent, Yoshua Bengio, and Xavier Bresson.
\newblock Graph neural networks with learnable structural and positional representations.
\newblock In \emph{International Conference on Learning Representations}, 2022.

\bibitem[Dwivedi et~al.(2023)Dwivedi, Joshi, Luu, Laurent, Bengio, and Bresson]{dwivedi2023benchmarking}
Vijay~Prakash Dwivedi, Chaitanya~K Joshi, Anh~Tuan Luu, Thomas Laurent, Yoshua Bengio, and Xavier Bresson.
\newblock Benchmarking graph neural networks.
\newblock \emph{Journal of Machine Learning Research}, 24\penalty0 (43):\penalty0 1--48, 2023.

\bibitem[Elfwing et~al.(2018)Elfwing, Uchibe, and Doya]{elfwing2018sigmoid}
Stefan Elfwing, Eiji Uchibe, and Kenji Doya.
\newblock Sigmoid-weighted linear units for neural network function approximation in reinforcement learning.
\newblock \emph{Neural networks}, 107:\penalty0 3--11, 2018.

\bibitem[Eliasof et~al.(2023)Eliasof, Frasca, Bevilacqua, Treister, Chechik, and Maron]{eliasof2023graph}
Moshe Eliasof, Fabrizio Frasca, Beatrice Bevilacqua, Eran Treister, Gal Chechik, and Haggai Maron.
\newblock Graph positional encoding via random feature propagation.
\newblock In \emph{International Conference on Machine Learning}, pages 9202--9223. PMLR, 2023.

\bibitem[Eliasof et~al.(2024)Eliasof, Bevilacqua, Sch{\"o}nlieb, and Maron]{eliasof2024granola}
Moshe Eliasof, Beatrice Bevilacqua, Carola-Bibiane Sch{\"o}nlieb, and Haggai Maron.
\newblock Granola: Adaptive normalization for graph neural networks.
\newblock \emph{arXiv preprint arXiv:2404.13344}, 2024.

\bibitem[Fey and Lenssen(2019)]{fey2019fast}
Matthias Fey and Jan~Eric Lenssen.
\newblock Fast graph representation learning with pytorch geometric, 2019.

\bibitem[Frasca et~al.(2022)Frasca, Bevilacqua, Bronstein, and Maron]{frasca2022understanding}
Fabrizio Frasca, Beatrice Bevilacqua, Michael~M Bronstein, and Haggai Maron.
\newblock Understanding and extending subgraph gnns by rethinking their symmetries.
\newblock In \emph{Advances in Neural Information Processing Systems}, 2022.

\bibitem[Freifeld et~al.(2015)Freifeld, Hauberg, Batmanghelich, and Fisher]{freifeld2015highly}
Oren Freifeld, Soren Hauberg, Kayhan Batmanghelich, and John~W Fisher.
\newblock Highly-expressive spaces of well-behaved transformations: Keeping it simple.
\newblock In \emph{Proceedings of the IEEE International Conference on Computer Vision}, pages 2911--2919, 2015.

\bibitem[Freifeld et~al.(2017)Freifeld, Hauberg, Batmanghelich, and Fisher]{freifeld2017transformations}
Oren Freifeld, S{o}ren Hauberg, Kayhan Batmanghelich, and Jonn~W Fisher.
\newblock Transformations based on continuous piecewise-affine velocity fields.
\newblock \emph{IEEE transactions on pattern analysis and machine intelligence}, 39\penalty0 (12):\penalty0 2496--2509, 2017.

\bibitem[Fukushima(1969)]{fukushima1969visual}
Kunihiko Fukushima.
\newblock Visual feature extraction by a multilayered network of analog threshold elements.
\newblock \emph{IEEE Transactions on Systems Science and Cybernetics}, 5\penalty0 (4):\penalty0 322--333, 1969.

\bibitem[Gao and Ji(2019)]{gao2019graph}
Hongyang Gao and Shuiwang Ji.
\newblock Graph u-nets.
\newblock In \emph{international conference on machine learning}, pages 2083--2092. PMLR, 2019.

\bibitem[Gilmer et~al.(2017)Gilmer, Schoenholz, Riley, Vinyals, and Dahl]{gilmer2017neural}
Justin Gilmer, Samuel~S Schoenholz, Patrick~F Riley, Oriol Vinyals, and George~E Dahl.
\newblock Neural message passing for quantum chemistry.
\newblock In \emph{International conference on machine learning}, pages 1263--1272. PMLR, 2017.

\bibitem[Goodfellow et~al.(2013)Goodfellow, Warde-Farley, Mirza, Courville, and Bengio]{pmlr-v28-goodfellow13}
Ian Goodfellow, David Warde-Farley, Mehdi Mirza, Aaron Courville, and Yoshua Bengio.
\newblock Maxout networks.
\newblock In Sanjoy Dasgupta and David McAllester, editors, \emph{Proceedings of the 30th International Conference on Machine Learning}, volume~28 of \emph{Proceedings of Machine Learning Research}, pages 1319--1327, Atlanta, Georgia, USA, 17--19 Jun 2013. PMLR.

\bibitem[Gronwall(1919)]{gronwall1919note}
Thomas~Hakon Gronwall.
\newblock Note on the derivatives with respect to a parameter of the solutions of a system of differential equations.
\newblock \emph{Annals of Mathematics}, pages 292--296, 1919.

\bibitem[Gómez-Bombarelli et~al.(2018)Gómez-Bombarelli, Wei, Duvenaud, Hernández-Lobato, Sánchez-Lengeling, Sheberla, Aguilera-Iparraguirre, Hirzel, Adams, and Aspuru-Guzik]{G_mez_Bombarelli_2018}
Rafael Gómez-Bombarelli, Jennifer~N. Wei, David Duvenaud, José~Miguel Hernández-Lobato, Benjamín Sánchez-Lengeling, Dennis Sheberla, Jorge Aguilera-Iparraguirre, Timothy~D. Hirzel, Ryan~P. Adams, and Alán Aspuru-Guzik.
\newblock Automatic chemical design using a data-driven continuous representation of molecules.
\newblock \emph{ACS Central Science}, 4\penalty0 (2):\penalty0 268–276, January 2018.
\newblock ISSN 2374-7951.
\newblock \doi{10.1021/acscentsci.7b00572}.

\bibitem[Haber and Ruthotto(2017)]{haber2017stable}
Eldad Haber and Lars Ruthotto.
\newblock Stable architectures for deep neural networks.
\newblock \emph{Inverse problems}, 34\penalty0 (1):\penalty0 014004, 2017.

\bibitem[He et~al.(2015)He, Zhang, Ren, and Sun]{he2015delving}
Kaiming He, Xiangyu Zhang, Shaoqing Ren, and Jian Sun.
\newblock Delving deep into rectifiers: Surpassing human-level performance on imagenet classification.
\newblock In \emph{Proceedings of the IEEE international conference on computer vision}, pages 1026--1034, 2015.

\bibitem[Hendrycks and Gimpel(2023)]{hendrycks2023gaussian}
Dan Hendrycks and Kevin Gimpel.
\newblock Gaussian error linear units (gelus), 2023.

\bibitem[Hochreiter and Schmidhuber(1997)]{hochreiter1997long}
Sepp Hochreiter and J{\"u}rgen Schmidhuber.
\newblock Long short-term memory.
\newblock \emph{Neural computation}, 9\penalty0 (8):\penalty0 1735--1780, 1997.

\bibitem[Hu et~al.(2020{\natexlab{a}})Hu, Fey, Zitnik, Dong, Ren, Liu, Catasta, and Leskovec]{hu2020ogb}
Weihua Hu, Matthias Fey, Marinka Zitnik, Yuxiao Dong, Hongyu Ren, Bowen Liu, Michele Catasta, and Jure Leskovec.
\newblock Open graph benchmark: Datasets for machine learning on graphs.
\newblock \emph{arXiv preprint arXiv:2005.00687}, 2020{\natexlab{a}}.

\bibitem[Hu et~al.(2020{\natexlab{b}})Hu, Liu, Gomes, Zitnik, Liang, Pande, and Leskovec]{hu2020strategies}
Weihua Hu, Bowen Liu, Joseph Gomes, Marinka Zitnik, Percy Liang, Vijay Pande, and Jure Leskovec.
\newblock Strategies for pre-training graph neural networks.
\newblock In \emph{International Conference on Learning Representations}, 2020{\natexlab{b}}.

\bibitem[Iancu et~al.(2020)Iancu, Ruiz, Ribeiro, and Isufi]{iancu2020graph}
Bianca Iancu, Luana Ruiz, Alejandro Ribeiro, and Elvin Isufi.
\newblock Graph-adaptive activation functions for graph neural networks.
\newblock In \emph{2020 IEEE 30th International Workshop on Machine Learning for Signal Processing (MLSP)}, pages 1--6. IEEE, 2020.

\bibitem[Jain et~al.(2017)Jain, Kar, et~al.]{jain2017non}
Prateek Jain, Purushottam Kar, et~al.
\newblock Non-convex optimization for machine learning.
\newblock \emph{Foundations and Trends{\textregistered} in Machine Learning}, 10\penalty0 (3-4):\penalty0 142--363, 2017.

\bibitem[Jumper et~al.(2021)Jumper, Evans, Pritzel, Green, Figurnov, Ronneberger, Tunyasuvunakool, Bates, {\v{Z}}{\'\i}dek, Potapenko, et~al.]{jumper2021highly}
John Jumper, Richard Evans, Alexander Pritzel, Tim Green, Michael Figurnov, Olaf Ronneberger, Kathryn Tunyasuvunakool, Russ Bates, Augustin {\v{Z}}{\'\i}dek, Anna Potapenko, et~al.
\newblock Highly accurate protein structure prediction with alphafold.
\newblock \emph{Nature}, 596\penalty0 (7873):\penalty0 583--589, 2021.

\bibitem[Khalife and Basu(2023)]{khalife2023power}
Sammy Khalife and Amitabh Basu.
\newblock On the power of graph neural networks and the role of the activation function, 2023.

\bibitem[Kim and Oh(2021)]{kim2021how}
Dongkwan Kim and Alice Oh.
\newblock How to find your friendly neighborhood: Graph attention design with self-supervision.
\newblock In \emph{International Conference on Learning Representations}, 2021.

\bibitem[Kipf and Welling(2017)]{kipf2017semisupervised}
Thomas~N. Kipf and Max Welling.
\newblock Semi-supervised classification with graph convolutional networks.
\newblock In \emph{International Conference on Learning Representations}, 2017.

\bibitem[Kreuzer et~al.(2021)Kreuzer, Beaini, Hamilton, L{\'e}tourneau, and Tossou]{kreuzer2021rethinking}
Devin Kreuzer, Dominique Beaini, Will Hamilton, Vincent L{\'e}tourneau, and Prudencio Tossou.
\newblock Rethinking graph transformers with spectral attention.
\newblock \emph{Advances in Neural Information Processing Systems}, 34:\penalty0 21618--21629, 2021.

\bibitem[Kunc and Kl{\'e}ma(2024)]{kunc2024three}
Vladim{\'\i}r Kunc and Ji{\v{r}}{\'\i} Kl{\'e}ma.
\newblock Three decades of activations: A comprehensive survey of 400 activation functions for neural networks.
\newblock \emph{arXiv preprint arXiv:2402.09092}, 2024.

\bibitem[Lee et~al.(2019)Lee, Lee, and Kang]{lee2019self}
Junhyun Lee, Inyeop Lee, and Jaewoo Kang.
\newblock Self-attention graph pooling.
\newblock In \emph{International conference on machine learning}, pages 3734--3743. PMLR, 2019.

\bibitem[Li et~al.(2018)Li, Han, and Wu]{li2018deeper}
Qimai Li, Zhichao Han, and Xiao-Ming Wu.
\newblock Deeper insights into graph convolutional networks for semi-supervised learning.
\newblock In \emph{Proceedings of the AAAI conference on artificial intelligence}, volume~32, 2018.

\bibitem[Liew et~al.(2016)Liew, Khalil-Hani, and Bakhteri]{liew2016bounded}
Shan~Sung Liew, Mohamed Khalil-Hani, and Rabia Bakhteri.
\newblock Bounded activation functions for enhanced training stability of deep neural networks on visual pattern recognition problems.
\newblock \emph{Neurocomputing}, 216:\penalty0 718--734, 2016.
\newblock ISSN 0925-2312.
\newblock \doi{https://doi.org/10.1016/j.neucom.2016.08.037}.

\bibitem[Liu et~al.(2024)Liu, Wang, Vaidya, Ruehle, Halverson, Soljačić, Hou, and Tegmark]{liu2024kan}
Ziming Liu, Yixuan Wang, Sachin Vaidya, Fabian Ruehle, James Halverson, Marin Soljačić, Thomas~Y. Hou, and Max Tegmark.
\newblock Kan: Kolmogorov-arnold networks, 2024.

\bibitem[Maas(2013)]{Maas2013RectifierNI}
Andrew~L. Maas.
\newblock Rectifier nonlinearities improve neural network acoustic models.
\newblock In \emph{International conference on machine learning}, 2013.

\bibitem[Maiti(2024)]{maiti2024adact}
Ritabrata Maiti.
\newblock Adact: Learning to optimize activation function choice through adaptive activation modules.
\newblock In \emph{The Second Tiny Papers Track at ICLR 2024}, 2024.

\bibitem[Martinez et~al.(2022)Martinez, Viles, and Olaizola]{martinez2022closed}
I{~n}igo Martinez, Elisabeth Viles, and Igor~G Olaizola.
\newblock Closed-form diffeomorphic transformations for time series alignment.
\newblock In \emph{International Conference on Machine Learning}, pages 15122--15158. PMLR, 2022.

\bibitem[McCallum et~al.(2000)McCallum, Nigam, Rennie, and Seymore]{mccallum2000automating}
Andrew~Kachites McCallum, Kamal Nigam, Jason Rennie, and Kristie Seymore.
\newblock Automating the construction of internet portals with machine learning.
\newblock \emph{Information Retrieval}, 3:\penalty0 127--163, 2000.

\bibitem[Mishra et~al.(2021)Mishra, Chandra, and Ghose]{mishra2021non}
Akash Mishra, Pravin Chandra, and Udayan Ghose.
\newblock A non-monotonic activation function for neural networks validated on benchmark tasks.
\newblock In \emph{Modern Approaches in Machine Learning and Cognitive Science: A Walkthrough: Latest Trends in AI, Volume 2}, pages 319--327. Springer, 2021.

\bibitem[Morris et~al.(2019)Morris, Ritzert, Fey, Hamilton, Lenssen, Rattan, and Grohe]{morris2019weisfeiler}
Christopher Morris, Martin Ritzert, Matthias Fey, William~L Hamilton, Jan~Eric Lenssen, Gaurav Rattan, and Martin Grohe.
\newblock Weisfeiler and leman go neural: Higher-order graph neural networks.
\newblock In \emph{Proceedings of the AAAI conference on artificial intelligence}, volume~33, pages 4602--4609, 2019.

\bibitem[Morris et~al.(2020)Morris, Kriege, Bause, Kersting, Mutzel, and Neumann]{morris2020tudataset}
Christopher Morris, Nils~M. Kriege, Franka Bause, Kristian Kersting, Petra Mutzel, and Marion Neumann.
\newblock Tudataset: A collection of benchmark datasets for learning with graphs, 2020.

\bibitem[Morris et~al.(2023)Morris, Lipman, Maron, Rieck, Kriege, Grohe, Fey, and Borgwardt]{morris2023wl}
Christopher Morris, Yaron Lipman, Haggai Maron, Bastian Rieck, Nils~M. Kriege, Martin Grohe, Matthias Fey, and Karsten Borgwardt.
\newblock Weisfeiler and leman go machine learning: The story so far.
\newblock \emph{Journal of Machine Learning Research}, 24\penalty0 (333):\penalty0 1--59, 2023.

\bibitem[Namata et~al.(2012)Namata, London, Getoor, Huang, and Edu]{namata2012query}
Galileo Namata, Ben London, Lise Getoor, Bert Huang, and U~Edu.
\newblock Query-driven active surveying for collective classification.
\newblock In \emph{10th international workshop on mining and learning with graphs}, volume~8, page~1, 2012.

\bibitem[Noutahi et~al.(2019)Noutahi, Beaini, Horwood, Gigu{\`e}re, and Tossou]{noutahi2019towards}
Emmanuel Noutahi, Dominique Beaini, Julien Horwood, S{\'e}bastien Gigu{\`e}re, and Prudencio Tossou.
\newblock Towards interpretable sparse graph representation learning with laplacian pooling.
\newblock \emph{arXiv preprint arXiv:1905.11577}, 2019.

\bibitem[Nwankpa et~al.(2018)Nwankpa, Ijomah, Gachagan, and Marshall]{nwankpa2018activation}
Chigozie Nwankpa, Winifred Ijomah, Anthony Gachagan, and Stephen Marshall.
\newblock Activation functions: Comparison of trends in practice and research for deep learning.
\newblock \emph{arXiv preprint arXiv:1811.03378}, 2018.

\bibitem[Opschoor et~al.(2020)Opschoor, Petersen, and Schwab]{opschoor2020deep}
Joost~AA Opschoor, Philipp~C Petersen, and Christoph Schwab.
\newblock Deep relu networks and high-order finite element methods.
\newblock \emph{Analysis and Applications}, 18\penalty0 (05):\penalty0 715--770, 2020.

\bibitem[Panigrahi et~al.(2019)Panigrahi, Shetty, and Goyal]{panigrahi2019effect}
Abhishek Panigrahi, Abhishek Shetty, and Navin Goyal.
\newblock Effect of activation functions on the training of overparametrized neural nets.
\newblock \emph{arXiv preprint arXiv:1908.05660}, 2019.

\bibitem[Paszke et~al.(2019)Paszke, Gross, Massa, Lerer, Bradbury, Chanan, Killeen, Lin, Gimelshein, Antiga, Desmaison, Köpf, Yang, DeVito, Raison, Tejani, Chilamkurthy, Steiner, Fang, Bai, and Chintala]{paszke2019pytorch}
Adam Paszke, Sam Gross, Francisco Massa, Adam Lerer, James Bradbury, Gregory Chanan, Trevor Killeen, Zeming Lin, Natalia Gimelshein, Luca Antiga, Alban Desmaison, Andreas Köpf, Edward Yang, Zach DeVito, Martin Raison, Alykhan Tejani, Sasank Chilamkurthy, Benoit Steiner, Lu~Fang, Junjie Bai, and Soumith Chintala.
\newblock Pytorch: An imperative style, high-performance deep learning library, 2019.

\bibitem[Price et~al.(2024)Price, Daultry~Ball, Lam, Jones, and Tanner]{price2024deep}
Ilan Price, Nicholas Daultry~Ball, Samuel~CH Lam, Adam~C Jones, and Jared Tanner.
\newblock Deep neural network initialization with sparsity inducing activations.
\newblock \emph{arXiv e-prints}, pages arXiv--2402, 2024.

\bibitem[Puny et~al.(2023)Puny, Lim, Kiani, Maron, and Lipman]{puny2023equivariant}
Omri Puny, Derek Lim, Bobak Kiani, Haggai Maron, and Yaron Lipman.
\newblock Equivariant polynomials for graph neural networks.
\newblock In \emph{International Conference on Machine Learning}, pages 28191--28222. PMLR, 2023.

\bibitem[Ramachandran et~al.(2017)Ramachandran, Zoph, and Le]{ramachandran2017searching}
Prajit Ramachandran, Barret Zoph, and Quoc~V Le.
\newblock Searching for activation functions.
\newblock \emph{arXiv preprint arXiv:1710.05941}, 2017.

\bibitem[Ramp{\'a}{\v{s}}ek et~al.(2022)Ramp{\'a}{\v{s}}ek, Galkin, Dwivedi, Luu, Wolf, and Beaini]{rampavsek2022recipe}
Ladislav Ramp{\'a}{\v{s}}ek, Michael Galkin, Vijay~Prakash Dwivedi, Anh~Tuan Luu, Guy Wolf, and Dominique Beaini.
\newblock Recipe for a general, powerful, scalable graph transformer.
\newblock \emph{Advances in Neural Information Processing Systems}, 35:\penalty0 14501--14515, 2022.

\bibitem[Reiser et~al.(2022)Reiser, Neubert, Eberhard, Torresi, Zhou, Shao, Metni, van Hoesel, Schopmans, Sommer, et~al.]{reiser2022graph}
Patrick Reiser, Marlen Neubert, Andr{\'e} Eberhard, Luca Torresi, Chen Zhou, Chen Shao, Houssam Metni, Clint van Hoesel, Henrik Schopmans, Timo Sommer, et~al.
\newblock Graph neural networks for materials science and chemistry.
\newblock \emph{Communications Materials}, 3\penalty0 (1):\penalty0 93, 2022.

\bibitem[Rumelhart et~al.(1986)Rumelhart, Hinton, and Williams]{rumelhart1986learning}
David~E Rumelhart, Geoffrey~E Hinton, and Ronald~J Williams.
\newblock Learning representations by back-propagating errors.
\newblock \emph{nature}, 323\penalty0 (6088):\penalty0 533--536, 1986.

\bibitem[Scardapane et~al.(2018)Scardapane, Van~Vaerenbergh, Comminiello, and Uncini]{scardapane2018improving}
Simone Scardapane, Steven Van~Vaerenbergh, Danilo Comminiello, and Aurelio Uncini.
\newblock Improving graph convolutional networks with non-parametric activation functions.
\newblock In \emph{2018 26th European Signal Processing Conference (EUSIPCO)}, pages 872--876. IEEE, 2018.

\bibitem[Scarselli et~al.(2008)Scarselli, Gori, Tsoi, Hagenbuchner, and Monfardini]{scarselli2008graph}
Franco Scarselli, Marco Gori, Ah~Chung Tsoi, Markus Hagenbuchner, and Gabriele Monfardini.
\newblock The graph neural network model.
\newblock \emph{IEEE transactions on neural networks}, 20\penalty0 (1):\penalty0 61--80, 2008.

\bibitem[Schwab et~al.(2023)Schwab, Stein, and Zech]{schwab2023deep}
Christoph Schwab, Andreas Stein, and Jakob Zech.
\newblock Deep operator network approximation rates for lipschitz operators.
\newblock \emph{arXiv preprint arXiv:2307.09835}, 2023.

\bibitem[Sen et~al.(2008)Sen, Namata, Bilgic, Getoor, Galligher, and Eliassi-Rad]{sen2008collective}
Prithviraj Sen, Galileo Namata, Mustafa Bilgic, Lise Getoor, Brian Galligher, and Tina Eliassi-Rad.
\newblock Collective classification in network data.
\newblock \emph{AI magazine}, 29\penalty0 (3):\penalty0 93--93, 2008.

\bibitem[Shapira~Weber et~al.(2019)Shapira~Weber, Eyal, Skafte, Shriki, and Freifeld]{shapira2019diffeomorphic}
Ron~A Shapira~Weber, Matan Eyal, Nicki Skafte, Oren Shriki, and Oren Freifeld.
\newblock Diffeomorphic temporal alignment nets.
\newblock \emph{Advances in Neural Information Processing Systems}, 32, 2019.

\bibitem[Sterling and Irwin(2015)]{ZINCdataset}
Teague Sterling and John~J. Irwin.
\newblock {ZINC 15} -- ligand discovery for everyone.
\newblock \emph{Journal of Chemical Information and Modeling}, 55\penalty0 (11):\penalty0 2324--2337, 11 2015.
\newblock \doi{10.1021/acs.jcim.5b00559}.

\bibitem[Sundaramoorthi and Yezzi(2018)]{NEURIPS2018_68148596}
Ganesh Sundaramoorthi and Anthony Yezzi.
\newblock Variational pdes for acceleration on manifolds and application to diffeomorphisms.
\newblock In S.~Bengio, H.~Wallach, H.~Larochelle, K.~Grauman, N.~Cesa-Bianchi, and R.~Garnett, editors, \emph{Advances in Neural Information Processing Systems}, volume~31. Curran Associates, Inc., 2018.

\bibitem[Szanda{\l}a(2021)]{szandala2021review}
Tomasz Szanda{\l}a.
\newblock Review and comparison of commonly used activation functions for deep neural networks.
\newblock \emph{Bio-inspired neurocomputing}, pages 203--224, 2021.

\bibitem[Veli{\v{c}}kovi{\'c} et~al.(2017)Veli{\v{c}}kovi{\'c}, Cucurull, Casanova, Romero, Lio, and Bengio]{velivckovic2017graph}
Petar Veli{\v{c}}kovi{\'c}, Guillem Cucurull, Arantxa Casanova, Adriana Romero, Pietro Lio, and Yoshua Bengio.
\newblock Graph attention networks.
\newblock \emph{arXiv preprint arXiv:1710.10903}, 2017.

\bibitem[Wang et~al.(2024)Wang, Liu, Zhou, Yi, Tan, and Ma]{wang2024continuous}
Hexiang Wang, Fengqi Liu, Qianyu Zhou, Ran Yi, Xin Tan, and Lizhuang Ma.
\newblock Continuous piecewise-affine based motion model for image animation.
\newblock \emph{arXiv preprint arXiv:2401.09146}, 2024.

\bibitem[Wang et~al.(2020)Wang, Li, Ma, Montufar, Zhuang, and Fan]{wang2020haar}
Yu~Guang Wang, Ming Li, Zheng Ma, Guido Montufar, Xiaosheng Zhuang, and Yanan Fan.
\newblock Haar graph pooling.
\newblock In \emph{International conference on machine learning}, pages 9952--9962. PMLR, 2020.

\bibitem[Weber and Freifeld(2023)]{weber2023regularization}
Ron~Shapira Weber and Oren Freifeld.
\newblock Regularization-free diffeomorphic temporal alignment nets.
\newblock In \emph{International Conference on Machine Learning}, pages 30794--30826. PMLR, 2023.

\bibitem[Wu et~al.(2022)Wu, Sun, Zhang, Xie, and Cui]{wu2022graph}
Shiwen Wu, Fei Sun, Wentao Zhang, Xu~Xie, and Bin Cui.
\newblock Graph neural networks in recommender systems: a survey.
\newblock \emph{ACM Computing Surveys}, 55\penalty0 (5):\penalty0 1--37, 2022.

\bibitem[Xu et~al.(2015)Xu, Wang, Chen, and Li]{xu2015empirical}
Bing Xu, Naiyan Wang, Tianqi Chen, and Mu~Li.
\newblock Empirical evaluation of rectified activations in convolutional network, 2015.

\bibitem[Xu et~al.(2019)Xu, Hu, Leskovec, and Jegelka]{xu2019how}
Keyulu Xu, Weihua Hu, Jure Leskovec, and Stefanie Jegelka.
\newblock How powerful are graph neural networks?
\newblock In \emph{International Conference on Learning Representations}, 2019.

\bibitem[Xu et~al.(2023)Xu, Zhang, and Yang]{xu2023tafs}
Zhen Xu, Xiaojin Zhang, and Qiang Yang.
\newblock Tafs: Task-aware activation function search for graph neural networks.
\newblock 2023.

\bibitem[Yang et~al.(2023)Yang, Shi, Xiao, Yang, Bhowmick, and Liu]{Yang_2023}
Renchi Yang, Jieming Shi, Xiaokui Xiao, Yin Yang, Sourav~S. Bhowmick, and Juncheng Liu.
\newblock Pane: scalable and effective attributed network embedding.
\newblock \emph{The VLDB Journal}, 32\penalty0 (6):\penalty0 1237–1262, March 2023.
\newblock ISSN 0949-877X.
\newblock \doi{10.1007/s00778-023-00790-4}.

\bibitem[Ying et~al.(2018)Ying, You, Morris, Ren, Hamilton, and Leskovec]{ying2018hierarchical}
Zhitao Ying, Jiaxuan You, Christopher Morris, Xiang Ren, Will Hamilton, and Jure Leskovec.
\newblock Hierarchical graph representation learning with differentiable pooling.
\newblock \emph{Advances in neural information processing systems}, 31, 2018.

\bibitem[Zhang et~al.(2023{\natexlab{a}})Zhang, Feng, Du, He, and Wang]{zhang2023complete}
Bohang Zhang, Guhao Feng, Yiheng Du, Di~He, and Liwei Wang.
\newblock A complete expressiveness hierarchy for subgraph gnns via subgraph weisfeiler-lehman tests.
\newblock In \emph{International Conference on Machine Learning}, 2023{\natexlab{a}}.

\bibitem[Zhang et~al.(2023{\natexlab{b}})Zhang, Luo, Wang, and He]{zhang2023rethinking}
Bohang Zhang, Shengjie Luo, Liwei Wang, and Di~He.
\newblock Rethinking the expressive power of gnns via graph biconnectivity.
\newblock \emph{arXiv preprint arXiv:2301.09505}, 2023{\natexlab{b}}.

\bibitem[Zhang and Fletcher(2016)]{zhang2016bayesian}
Miaomiao Zhang and P~Thomas Fletcher.
\newblock Bayesian statistical shape analysis on the manifold of diffeomorphisms.
\newblock \emph{Algorithmic Advances in Riemannian Geometry and Applications: For Machine Learning, Computer Vision, Statistics, and Optimization}, pages 1--23, 2016.

\bibitem[Zhang et~al.(2021)Zhang, Liang, Liu, and Tang]{zhang2021graph}
Xiao-Meng Zhang, Li~Liang, Lin Liu, and Ming-Jing Tang.
\newblock Graph neural networks and their current applications in bioinformatics.
\newblock \emph{Frontiers in genetics}, 12:\penalty0 690049, 2021.

\bibitem[Zhang et~al.(2022)Zhang, Zhu, Meng, Koniusz, and King]{zhang2022graph}
Yifei Zhang, Hao Zhu, Ziqiao Meng, Piotr Koniusz, and Irwin King.
\newblock Graph-adaptive rectified linear unit for graph neural networks.
\newblock In \emph{Proceedings of the ACM Web Conference 2022}, pages 1331--1339, 2022.

\bibitem[Zhang et~al.(2019)Zhang, Bu, Ester, Zhang, Yao, Yu, and Wang]{zhang2019hierarchical}
Zhen Zhang, Jiajun Bu, Martin Ester, Jianfeng Zhang, Chengwei Yao, Zhi Yu, and Can Wang.
\newblock Hierarchical graph pooling with structure learning.
\newblock \emph{arXiv preprint arXiv:1911.05954}, 2019.

\end{thebibliography}

\clearpage
\newpage

%%%%%%%%%%%%%%%%%%%%%%%%%%%%%%%%%%%%%%%%%%%%%%%%%%%%%%%%%%%%

\appendix
\section{Additional Related Work}

\paragraph{Graph Neural Networks.}
Graph Neural Networks \cite{scarselli2008graph} (GNNs) have emerged as a transformative approach in machine learning, notably following the popularity of the message-passing scheme~\citep{gilmer2017neural}. GNNs enable effective learning from graph-structured data, and can be applied to different tasks, ranging from social network analysis \citep{kipf2017semisupervised} to bioinformatics \citep{jumper2021highly}. In recent years, various GNN architectures were proposed, aiming to address various aspects, from alleviating oversmoothing~\citep{chen2020simple}, concerning attention mechanisms in the message passing scheme \citep{velivckovic2017graph,brody2022how,kim2021how}, or focusing on the expressive power of the architectures \citep{morris2023wl,frasca2022understanding,zhang2023complete,zhang2023rethinking,puny2023equivariant,bevilacqua2023efficient}, given that message-passing based architectures are known to be bounded by the WL graph isomorphism test~\citep{xu2019how,morris2019weisfeiler}. 

Despite advancements, the poor performance of deep GNNs has led to a preference for shallow architectures GCNs \citep{li2018deeper}. To enhance performance, techniques such as pooling functions have been proposed, introducing generalization by reducing feature map sizes~\citep{zhang2019hierarchical}. Methods such as HGP-SL~\citep{zhang2019hierarchical}, GraphUNet~\citep{gao2019graph}, and LaPool~\citep{noutahi2019towards} introduce pooling layers specifically designed for GNNs. Beyond node feature, the importance of graph structure and positional features is increasingly recognized, with advancements such as GraphGPS~\citep{rampavsek2022recipe} and SAN~\citep{kreuzer2021rethinking} integrating positional and structural encodings through attention-based mechanisms.

\paragraph{Evaluation of Rectified Activation Functions.}
Rectified activation functions, represented by the Rectified Linear Unit (ReLU), have been widely applied and studied in various neural network architectures due to their simplicity and effectiveness~\citep{nwankpa2018activation, apicella2021survey, kunc2024three}. The prevalent assumption that ReLU's performance is predominantly due to its sparsity is critically examined by \citet{xu2015empirical}, suggesting introducing a non-zero slope in the negative part can significantly enhance network performance. Extending this, \citet{price2024deep} investigates sparsity-inducing activation functions, such as the shifted ReLU, in network initialization and early stages of training. These functions can mitigate overfitting and boost model generalization capabilities. Conversely, it was shown that in overparameterized networks, smoother activation functions, like Tanh and Swish, can enhance the convergence rate, in contrast to the non-smooth characteristics of ReLU~\citep{panigrahi2019effect}. However, the fixed nature of ReLU and many of its variants restricts their ability to adapt the input, resulting in limited power to capture dynamics in learning. 

\paragraph{Advancements in Learnable Activation Functions.}
Recent research has increasingly focused on adaptive and learnable activation functions, which are optimized alongside the learning process of the network. The AdAct framework~\citep{maiti2024adact} introduces learnability by combining multiple activation functions into a single module with learnable weighting coefficients. However, these coefficients are fixed after training, limiting the framework's adaptability to varying inputs.  A concurrent work by~\citet{liu2024kan} introduces Kolmogorov-Arnold Networks (KAN), a novel architecture that diverges from traditional Multi-Layer Perceptron (MLP) configurations, which applies activation functions to network edges instead of nodes. Unlike our current work, which focuses only on the design of activation functions for GNNs, their research extends beyond this scope and considers a fundamental architecture design. Finally, the recently proposed TAFS~\citep{xu2023tafs} learns a task-adaptive (but not graph-adaptive) activation function for GNNs through a bi-level optimization.

\section{Implementation Details of \ourmethod}
\label{sec:appendix:implementation}
\paragraph{Multiple Graphs in one Batch.} Consider a set of graphs $S = \{G_1, G_2, \cdots , G_B\}$ with a batch size of $B$. Let $N_S = N_1 + N_2 + \cdots + N_B$ represent the cumulative number of nodes across the graph dataset. The term $N_{\text{max}} \overset{\Delta}{=} \max (N_1, N_2, \cdots, N_B)$ denotes the largest node count present in any single graph within $S$.

To create a unified feature matrix for $S$ that encompasses all graphs in the batch, we standardize the dimension by padding each feature matrix $\mathbf{X}_i \in \mathbb{R}^{N_i \times C},\ i \in [B]$ for graph $G_i \in S$ from $N_i$ to $N_{\text{max}}$ with zeros. The combined feature matrix $\mathbf{X}_S$ is constructed by concatenating the transposed feature matrices $\mathbf{X}_i^\top$ $\forall i \in [B]$, resulting in a matrix that lies in the domain $\mathbb{R}^{(B \cdot C) \times N_{\text{max}}}$. This matrix is permutation invariant; while relabeling nodes changes the row indices, it does not affect the overall transformation process. Therefore, \ourmethod\ can handle multiple graphs in a batch. In practice, to avoid the overhead of padding, we use the batching support from Pytorch-Geometric \citep{fey2019fast}.

\paragraph{Implementation Details of $\textsc{GNN}_{\textsc{act}}$.} In \cref{sec:methods:models}, we examined two distinct approaches to learn the diffeomorphism parameters $\vtheta^{(l)}$, either directly or through $\textsc{GNN}_{\textsc{act}}$. As shown in \cref{sec:appendix:thetaRelations}, $\vtheta^{(l)}$ determines the velocity field $v^{\vtheta^{(l)}}$. Predicting a graph-dependent $\vtheta^{(l)}$ adds graph-adaptivity to the activation function $T^{(l)}$. In \ourmethod\ we achieve this by employing another GNN $\textsc{GNN}_{\textsc{act}}$,  described below.

The backbone of $\textsc{GNN}_{\textsc{act}}$ utilizes the same structure as the primary network layers $\textsc{GNN}_{\textsc{layer}}^{(l)}$, that is, GCN \citep{kipf2017semisupervised} or GIN \citep{xu2019how}. It is important to note, that while $\textsc{GNN}_{\textsc{act}}$ has a similar structure to the primary network GNN \revision{with ReLU activation function}, it has its own set of learnable weights, and it is shared among the layers, unlike the primary GNN layers $\textsc{GNN}_{\textsc{layer}}^{(l)}$. The hidden dimensions and the number of layers of $\textsc{GNN}_{\textsc{act}}$ are hyperparameters. The weight parameters of $\textsc{GNN}_{\textsc{act}}$ are trained concurrently with the main network weights. 
As described in \Cref{eqn:digaf}, after the computation of  $\textsc{GNN}_{\textsc{act}}$,  a pooling layer denoted by \textsc{Pool} is placed to aggregate node features. This aggregation squashes the node dimension such that the output is not dependent on the specific order of nodes, and it yields the vector of parameters $\vtheta^{(l)}$.  

\paragraph{Rescaling $\mathbf{\bar{H}}^{(l)}$.} Following the implementation of \citet{freifeld2015highly}, the default 1D domain for CPAB is set as $[0, 1]$. To enhance the flexibility of $T^{(l)}$ and ensure its adaptability across various input datasets, \ourmethod\ extends the domain to $\Omega = [a, b] \subset \mathbb{R}$ with $a<b$ as shown in \cref{sec:methods:CPAB_definition}. To match the two domains, we rescale the intermediate feature matrix $\mathbf{\bar{H}}^{(l)}$ from $\Omega$ to the unit interval $[0, 1]$ before passing it to $T^{(l)}$. Let $r = \frac{b-a}{2}$, then rescaling is performed using the function $f(x) = (x + r)/(2r)$. Data points outside this range will retain their original value, effectively acting as an identity function outside the domain $\Omega$.

\paragraph{Training Loss Function.} 
As described in \Cref{eqn:cpa_regularization}, we employ a regularization term for the velocity field to maintain the smoothness of the activation function. To control the strength of regularization, we introduce a hyperparameter $\lambda$. We denote $\mathcal{L}_{\textsc{task}}$ as the loss function of the downstream task (i.e. cross-entropy loss in case of classification and mean absolute error in case of regression tasks), and the overall training loss of \ourmethod, denoted as $\mathcal{L}_{\textsc{total}}$ is given as
\begin{equation}
    \label{eqn:total_loss}
    \mathcal{L}_{\textsc{total}} = \mathcal{L}_{\textsc{task}} + \lambda \, \mathcal{R}(\{\vtheta^{(l)}\}_{l=1}^L).
\end{equation}

\section{Overview of CPA Velocity Fields and CPAB Transformations}

\label{sec:appendix:thetaRelations}
 In this Section, we drop the layer notations $l$ for simplicity.
In \cref{sec:methods:CPAB_definition}, we introduce the concept of a diffeomorphism on a closed interval in \cref{sec:def:diffeomorphism}, which can be learned through the integration of a Continuous Piecewise Affine (CPA) velocity field. As detailed in \cref{sec:def:CPA_Velocity}, the velocity field $v^{\vtheta}$ is governed by the parameter $\vtheta$ and the tessellation $\gP$. We now discuss how the velocity fields are computed following the methodologies presented by~\citet{freifeld2015highly, freifeld2017transformations} and highlight the relations between $v^{\vtheta}$, $\vtheta$ and $\gP$. We start by formally defining the tessellation on $\Omega$:

\begin{definition}[Tessellation of a closed interval~\citep{freifeld2015highly}] \label{sec:def:tessellation}
A tessellation $\gP$ of size $\gN_\gP$ subintervals of a closed interval $\Omega = [a, b]$ in $\sR$ is a partitioning $\{[x_i, x_{i+1}]\}_{i=0}^{\gN_\gP-1}$ that satisfies the following properties:
\begin{enumerate}[label=(\arabic*)]
    \item $x_0 = a$ and $x_{\gN_{\gP}} = b$
    \item Each point $x \in \Omega$ lies in at least one subinterval $[x_i, x_{i+1}]$
    \item The intersection of any two subintervals $[x_i, x_{i+1}]$ and $[x_{i+1}, x_{i+2}]$ is exactly $\{x_{i+1}\}$
    \item $\bigcup\limits_{i=0}^{\gN_\gP-1} [x_i, x_{i+1}] = \Omega $
\end{enumerate}

\end{definition}

The vector of parameters $\vtheta$ is linked to the subintervals in $\gP$, whose dimension is determined by the number of intervals $\mathcal{N}_{\mathcal{P}}$. Similar to \citet{freifeld2015highly}, we impose boundary constraints that mandate the velocity at the boundary of the tessellation to be zero, i.e., $v^{\vtheta}(0) = v^{\vtheta}(1) = 0$. This boundary condition allows us to compose the diffeomorphism in the domain $\Omega$ with an identity function for any values outside the domain. 
Under this constraint, the degrees of freedom (number of parameters) for $\theta$ is $\mathcal{N}_{\mathcal{P}}-1$.

\iffalse
With $\vtheta$ and $\gP$ defined, we proceed to compute the velocity field $v^{\vtheta}$. We first introduce a membership function $\gamma: \Omega \mapsto \{0, \cdots, \mathcal{N}_\mathcal{P}-1\}$ that assigns any real number $x \in \Omega$ to the index $p \in \{0, \cdots, \mathcal{N}_\mathcal{P} - 1\}$ of $\mathrm{cell}_p$ in $\mathcal{P}$ containing $x$. This function $\gamma$ is defined as $\gamma: x \mapsto \min \{p: x \in cell_p\}$. 
Following the framework laid out by~\citet{freifeld2015highly, freifeld2017transformations}, the velocity field is expressed mathematically as:
\begin{equation}
    \label{eq:velocity}
    v^{\vtheta}(x) = \mathbf{D}_{\gamma(x), \vtheta}\,\tilde{\vx}
\end{equation}
Here, $\tilde{\vx} = \begin{bmatrix}
    x \\ 1
\end{bmatrix}$, and the coefficient matrix $\mathbf{D}_{\gamma(x), \vtheta}$ corresponds to the cell indexed $p = \gamma(x)$, varying with the parameter vector $\vtheta \in \sR^{\gN_\gP-1}$. This matrix is a constituent of the vector $\mathbf{D}_\vtheta \overset{\Delta}{=} (\mathbf{D}_{0, \vtheta}, \mathbf{D}_{1, \vtheta}, \cdots, \mathbf{D}_{\gN_\gP - 2, \vtheta})$. The computation of each coefficient $\mathbf{D}_{\gamma(x), \vtheta}$ is described by:
\begin{equation}
    \label{eq:velocity:coeff}
    \mathbf{D}_{\gamma(x), \vtheta} = \sum_{j=0}^{\mathcal{N}_{\mathcal{P}}-2} \vtheta_j \cdot \mathbf{b}_j
\end{equation}
where $\mathbf{b}_j \in \mathbb{R}^{2 \times \mathcal{N}_{\mathcal{P}}}$ represents an orthonormal basis vector associated with the null space defined by $\mathcal{P}$. 
\fi

The velocity field is then defined as follows:
\begin{definition}[Relation between $\vtheta$ and ${v}^{\vtheta}$, taken from \citet{freifeld2017transformations}] \label{sec:def:theta_velocity}
Given a tessellation $\gP$ with $\gN_\gP$ intervals on a closed domain $\Omega=[a, b]$, as defined in \cref{sec:def:tessellation}. Given a parameter $\vtheta \in \mathbb{R}^{\gN_\gP - 1}$ and an arbitrary point $x$ within the domain, a continuous piecewise-affine velocity field ${v}^{\vtheta}$ can present as follows:

\begin{equation}
    \label{eq:velocity_theta}
    {v}^{\vtheta}(x) = \sum_{j=0}^{\gN_{\gP} - 2} \vtheta_j \mathbf{b}_j \Tilde{x},
\end{equation}
where $\{\mathbf{b}_j\}_{j=0}^{\gN_{\gP} - 2}$ is an orthonormal basis of the space of velocity fields  $\mathcal{V}$, such that ${v}^{\vtheta} \in \mathcal{V}$, and $\tilde{\vx} = \begin{bmatrix}
    x \\ 1
\end{bmatrix}$.
\end{definition}

\revision{The orthonormal basis $\{\mathbf{b}_j\}_{j=0}^{\gN_{\gP} - 2}$ for the velocity field can be obtained through Singular Value Decomposition of $\mathbf{L}$. Note that $\mathbf{L}$ is a matrix constraining the coefficients of each continuous piecewise-affine velocity function by ensuring that the velocity value at the shared endpoints is the same \citep{martinez2022closed}. Let $vec(\mathbf{A})$ be a column vector containing the coefficients for each interval, for instance, $[a_0, b_0, a_1, b_1]^T$ for consecutive intervals interval $0$ and interval $1$. The shared endpoint is $x_1$. To achieve the constrain, we have the equation $a_0 * x_1 + b_0 + a_1 * (-x_1) + b_1 * (-1) = 0$. In this example, the constrain matrix $\mathbf{L}$ is $\mathbf{L} = [x_1, 1, -x_1, -1]$. By generalizing the previous example, the constraint can be expressed as $\mathbf{L} \ast vec(\mathbf{A}) = \overrightarrow{0}$. As the endpoints are decided by the tessellation setup, we can build the constrain matrix $\mathbf{L}$ without knowing $vec(\mathbf{A})$. And thus, orthonormal basis $\{\mathbf{b}_j\}_{j=0}^{\gN_{\gP} - 2}$ can be computed by giving tessellation setup.}

\revision{
\begin{proposition}[\ourmethod\ has a closed form solution] \label{sec:def:closed_form}
Equation \ref{eq:cpab_integration} can be expressed as an equivalent ODE. By allowing $x$ to vary and fixing $t$, the solution to this ODE can be written as a composition of a finite number of solutions $\psi$:
$$
    \phi^\theta (x, t) = (\psi^{t_m}_{\theta, c_m} \circ \psi^{t_{m-1}}_{\theta, c_{m-1}} \circ \cdots \circ \psi^{t_2}_{\theta, c_2} \circ \psi^{t_1}_{\theta, c_1})(x)
$$
Here $m$ represents the number of cells visited. Given $x, \theta$, time $t$, and the smallest cell index containing $x$, $c$, we can compute each $\psi^{t_i}_{\theta, c_i}(x), i \in \{1, …, m\}$ from $\psi^{t_1}_{\theta, c_1}(x)$ to $\psi^{t_m}_{\theta, c_m}(x)$. In other words, \ourmethod\ has a closed-form solution.
\end{proposition}
\begin{proof}    
The proof follows the steps in \citet{martinez2022closed}.
\Cref{eq:cpab_integration} can be expressed as the equivalent ODE: $\frac{d\phi^\theta (x, t)}{d t} = v^\theta (\phi^\theta (x, t))$. By allowing $x$ to vary and fixing $t$, the solution to this ODE can be written as a composition of a finite number of solutions $\psi$:
$$
    \phi^\theta (x, t) = (\psi^{t_m}_{\theta, c_m} \circ \psi^{t_{m-1}}_{\theta, c_{m-1}} \circ \cdots \circ \psi^{t_2}_{\theta, c_2} \circ \psi^{t_1}_{\theta, c_1})(x)
$$
Given $x, \theta$, time $t$, and a function $\gamma$ that returns the smallest cell index containing $x$, namely $c = \gamma(x)$, then we can compute each $\psi^{t_i}_{\theta, c_i}(x), i \in \{1, …, m\}$ and use them in the closed form solution for the ODE. The cell boundary $x_c$ is determined based on the velocity value $v(x)$ at point $x$. If $v(x) \geq 0$, $x_c$ is the largest point in the interval; otherwise, it is the smallest point.
In this setup, a cell is a 1D interval with two endpoints. At the hitting time $t_{hit}$, $\psi_c^{\theta}(x, t_{hit})$ is

$$
    \psi_c^{\theta}(x, t_{hit}) = x_c,
$$

where $t^\theta_{hit} = \frac{1}{a^\theta_c} \log \left( \frac{a^\theta_c x_c + b^\theta_c}{a^\theta_c x + b^\theta_c}\right) $. The CPAB velocity field is continuous piecewise-affine, and for each interval with index $c$, it has coefficients $a^\theta_c$ (slope) and $b^\theta_c$ (bias). These can be computed given $\theta$. If $t^\theta_{hit} > t$, then $\phi^\theta (x, t) = \psi_c (x, t)$. Otherwise, we repeat the process with updated values $t = t-t^\theta_{hit}$, $x = x_c$, and $c$ adjusted based on the sign of $v(x)$.

This iterative process continues until convergence, with an upper bound for $m$ being $max(c_1, N_P - c_1 + 1)$, where $c_1$ refers to the first visited cell index, and $N_P$ is the number of closed intervals in the space $\Omega$. With the above steps, we can precisely compute each $\psi^{t_i}_{\theta, c_i}(x), i \in \{1, …, m\}$ from $\psi^{t_1}_{\theta, c_1}(x)$ to $\psi^{t_m}_{\theta, c_m}(x)$ following the equation $\phi^\theta (x, t) = (\psi^{t_m}_{\theta, c_m} \circ \cdots \circ \psi^{t_1}_{\theta, c_1})(x)$. This allows us to determine the exact solution for $\phi^\theta (x, t)$.
\end{proof}

}

\section{Properties and Proofs}
\label{sec:appendix:proofs}
\revision{We present a summary of the properties offered by our  \ourmethod\, that are absent in  general-purpose activation functions or existing graph activations in \Cref{tab:summary_properties}.}
\begin{table*}[t]
    \centering
    \scriptsize
    \caption{A summary of the properties of activation functions.  --  means not studied in the corresponding paper. $^*$ Median-of-medians algorithm can achieve linear time complexity on average.}
    \label{tab:summary_properties}
    \begin{tabular}{l|cccccc}
    \toprule
         Act $\downarrow$ / Prop $\rightarrow$ & Boundedness & Differentiability & Linear Complexity & Permutation Equiv. & Lipschitz Cont. & Graph Adap. \\
         \midrule 
         ReLU~\citep{xu2019how} & $\times$ & $\times$ & $\checkmark$ &$\checkmark$&$\checkmark$ & $\times$\\
         Tanh~\citep{hochreiter1997long} &$\checkmark$ & $\checkmark$& $\checkmark$&$\checkmark$ & $\checkmark$& $\times$ \\
         PReLU~\citep{he2015delving} &$\times$ & $\times$ & $\checkmark$&$\checkmark$ &$\checkmark$ & $\times$\\
         Swish~\citep{ramachandran2017searching} &$\times$ &$\checkmark$ &$\checkmark$ &$\checkmark$ &$\checkmark$ &$\times$ \\
         Max~\citep{iancu2020graph} & $\times$&$\times$ &$\checkmark$ &$\checkmark$ &$\checkmark$ &$\checkmark$ \\
         Median~\citep{iancu2020graph} &$\times$ & $\times$ &$\checkmark^*$ &$\checkmark$ &-- & $\checkmark$\\
         GReLU~\citep{zhang2022graph} & $\times$&$\times$ &$\checkmark$ &$\checkmark$ &-- & $\checkmark$\\
         \ourmethod & $\bm{\checkmark}$& $\bm{ \checkmark}$&$\bm{\checkmark}$ & $\bm{\checkmark}$& $\bm{\checkmark}$&$\bm{\checkmark}$ \\
         \bottomrule
    \end{tabular}
\end{table*}

Similar to \Cref{sec:appendix:thetaRelations}, for simplicity, in this Section, we drop the layer notations $l$. 

In this section, we present the propositions and proofs for the properties outlined in \cref{sec:theoritical_analysis}. We begin by remarking that as shown in \cref{sec:theoritical_analysis},  \ourmethod\ is bounded within the domain $\Omega = [a, b]$, where $a < b$ by construction. We then present \Cref{prop:lipConstV} that outlines the Lipschitz constant of the velocity field $v^\vtheta$, followed by \Cref{prop:lipschitz_digraf}, showing that \ourmethod\ is also Lipschitz continuous, and provide an upper bound on its Lipschitz constant.

\begin{proposition}[The Lipschitz Constant of $v^{\vtheta}$]
    \label{prop:lipConstV}
    Given two arbitrary points $x, y \in \mathbb{R}$, and velocity field parameters $\vtheta \in \mathbb{R}^{\gN_\gP - 1}$ that define the continuous piecewise-affine velocity field $v^{\vtheta}$, there exists a Lipschitz constant $C_{v^{\vtheta}} = \sum_{j=0}^{\gN_{\gP} - 2} |\vtheta_j|$ such that 
    \begin{equation}
        \left| v^{\vtheta}(x) - v^{\vtheta}(y) \right| \leq C_{v^{\vtheta}} \|(\tilde{\vx}-\tilde{\vy}) \|_2,
    \end{equation}
    where $| \cdot|$ and $\| \cdot\|_2$ denote the absolute value of a scalar and the $\ell_2$ norm of a vector, respectively.
\end{proposition}

\begin{proof} 

First, we note that it was shown in~\citet{freifeld2015highly, freifeld2017transformations} that $v^{\vtheta}$ is Lipschitz continuous, and now we provide a derivation of that Lipschitz constant. 
    Following \cref{sec:def:theta_velocity}, the velocity field $v^{\vtheta}$ is defined as ${v}^{\vtheta}(x) = \sum_{j=0}^{\gN_{\gP} - 2} \vtheta_j \mathbf{b}_j \Tilde{x}$, where $\{\mathbf{b}_j\}_{j=0}^{\gN_\gP - 2}$ is an orthonormal basis of the velocity space. By the definition of $v^{\vtheta}(x)$ and $v^{\vtheta}(y)$, we have the following:
    \begin{align}
        \left| v^{\vtheta}(x) - v^{\vtheta}(y) \right| &=     \left|\sum_{j=0}^{\gN_{\gP} - 2} \vtheta_j \mathbf{b}_j \tilde{\vx} -    \sum_{j=0}^{\gN_{\gP} - 2} \vtheta_j \mathbf{b}_j \Tilde{y} \right| \\ \label{eq:eq16}
        &=     \left|\sum_{j=0}^{\gN_{\gP} - 0} \vtheta_j \mathbf{b}_j (\tilde{\vx}-\tilde{\vy}) \right| \\ \label{eq:eq17}
        &\leq     \sum_{j=0}^{\gN_{\gP} - 2} |\vtheta_j| \| \mathbf{b}_j \|_2 \| (\tilde{\vx}-\tilde{\vy}) \|_2  \\ \label{eq:eq18}
        & =  \sum_{j=0}^{\gN_{\gP} - 2} |\vtheta_j|  \| (\tilde{\vx}-\tilde{\vy}) \|_2 \\
        &= \|(\tilde{\vx}-\tilde{\vy}) \|_2  \sum_{j=0}^{\gN_{\gP} - 2} |\vtheta_j|  \\
        &= C_{v^{\vtheta}} \|(\tilde{\vx}-\tilde{\vy}) \|_2,
    \end{align}
    where the transition between \Cref{eq:eq16} and \cref{eq:eq17} follows from the triangle inequality, and the transition between \cref{eq:eq17} and \cref{eq:eq18} follows from $\mathbf{b}_j$ being an orthonormal vector.
    
    From the derivation above, and the fact that we know from \citet{freifeld2015highly, freifeld2017transformations} that the velocity field is Lipschitz continuous, we conclude that the Lipschitz constant $C_{v^{\vtheta}}$ of $v^{\vtheta}$ reads $C_{v^{\vtheta}}=\sum_{j=0}^{\gN_{\gP} - 2} |\vtheta_j|$.
\end{proof}

Given the Lipschitz constant $C_{v^{\vtheta}}$ for $v^{\vtheta}$, we proceed to demonstrate that the transformation $T(\cdot; \vtheta)$ in \ourmethod\ is Lipschitz continuous, as well as bounding its Lipschitz constant. 
\begin{proposition}[The Lipschitz Constant of \ourmethod]
    \label{prop:lipschitz_digraf}
    The diffeomorphic function $T(\cdot; \vtheta)$ in \ourmethod\ is defined in \cref{eq:TandPhi} for a given set of weights $\vtheta$, which in turn define the velocity field $v^{\vtheta}$. Let $x, y \in \mathbb{R}$ be two arbitrary points, then the following inequality is satisfied:
    \begin{equation}
        | T(x; \vtheta) - T(y; \vtheta) | \leq |x-y|\exp(C_{v^{\vtheta}})
    \end{equation}
    where $C_{v^{\vtheta}}$ is the Lipschitz constant of $v^{\vtheta}$.
\end{proposition}

\begin{proof}
    We begin by substituting $T(\cdot; \vtheta)$ with \cref{eq:cpab_diffeo} and \cref{eq:TandPhi}. Utilizing \Cref{prop:lipConstV}, we then establish an upper bound for $| T(x; \vtheta) - T(y; \vtheta) |$ as follows:
    \begin{align}
        \label{eq:lipDigraf}
        | T(x; \vtheta) - T(y; \vtheta) | &= | x + \int\limits_{0}^{1} v^\vtheta\big(\phi^\vtheta(x, \tau)\big)\, d\tau  - y - \int\limits_{0}^{1} v^\vtheta\big(\phi^\vtheta(y, \tau)\big)\, d\tau      | &  \\
        &\leq | x - y | + C_{v^{\vtheta}} \int\limits_{0}^1 \left|(\phi^\vtheta(x, \tau) - \phi^\vtheta(y, \tau)) \right|  \\
        &\leq |x - y|\exp(C_{v^{\vtheta}}),
    \end{align}
    where $ C_{v^{\vtheta}}$ is the Lipschitz constant of $v^{\vtheta}$ (\Cref{prop:lipConstV}) and the last transition follows from Grönwall's inequality \citep{gronwall1919note}. Consequently, the Lipschitz constant of \ourmethod\ is bounded from above by $\exp(C_{v^{\vtheta}})$.
\end{proof}

Now that we established that $T(\cdot; \vtheta)$  is Lipschitz continuous and presented an upper bound, we investigate what is the maximal difference in the output of $T(\cdot; \vtheta)$ with respect to two arbitrary inputs $x, y \in \Omega$, and whether it can be bounded. To address this, we present the following remark:
\begin{remark}
    \label{sec:remark:T_within_Omega}
    Given a bounded domain $\Omega=[a, b],\ a < b$, by construction, the diffeomorphism $T(\cdot; \vtheta)$ with parameter $\vtheta$ in \ourmethod, as in \Cref{eq:ourmethod}, is a $\Omega \to \Omega$ transformation~\citep{freifeld2015highly, freifeld2017transformations}. Therefore, by the max value theorem, the maximal output discrepancy for arbitrary $x, y \in \Omega$ is $|b-a|$, i.e., $         |T(x; \vtheta) - T(y; \vtheta)| \leq |b-a|$.
\end{remark}

Combining the  \Cref{prop:lipConstV}, \Cref{prop:lipschitz_digraf}  and \Cref{sec:remark:T_within_Omega}, we  formalize and prove the following proposition:
\Boundness*

\begin{proof}

     In \cref{prop:lipschitz_digraf} we presented an upper bound on the Lipschitz constant of $T(\cdot; \vtheta)$, and in \ref{sec:remark:T_within_Omega} we also presented an upper bound on the maximal difference between the application of $T(\cdot; \vtheta)$ on two inputs $x, y$. 
     Combining the two bounds, we get the following inequality:
     \begin{equation}
         |T(x; \vtheta) - T(y; \vtheta)| \leq \min(|b-a|, |x-y|\exp(C_{v^{\vtheta}})).
     \end{equation}
\end{proof}

The result in \cref{prop:digraf_bounded} gives us a tighter upper bound on the boundedness of the transformation $T(\cdot; \vtheta)$ in our \ourmethod\, that is related both to the hyperparameters $a, b$, as well as the learned velocity field parameters $\vtheta$.

Next, we discuss another property outlined in \cref{sec:theoritical_analysis},  demonstrating that \ourmethod\ is permutation equivariant -- a desirable property when designing a GNN component \citep{bronstein2021geometric}.

\begin{proposition}[ \ourmethod\ is permutation equivariant.]
        \label{theo:permutataion}
    Consider a graph encoded by the adjacency matrix $\mathbf{A} \in \mathbb{R}^{N \times N}$, where $N$ is the number of nodes. Let $\bar{\mathbf{H}}^{(l)} \in \mathbb{R}^{N \times C}$ be the intermediate node features at layer $l$, before the element-wise application of our \ourmethod.
    Let $\mathbf{P}$ be an $N \times N$ permutation matrix. Then,
    \begin{equation}
        \ourmethod(\mathbf{P} \bar{\mathbf{H}}^{(l)}, \vtheta^{(l)}_P) = \mathbf{P}\: \ourmethod(\bar{\mathbf{H}}^{(l)},  \vtheta^{(l)})
    \end{equation}
    where $\vtheta^{(l)}_P$ and $\vtheta^{(l)}$ are obtained by feeding $\mathbf{P} \bar{\mathbf{H}}^{(l)}$ and $ \bar{\mathbf{H}}^{(l)}$, respectively, to \Cref{eqn:digaf}.
\end{proposition}

\iffalse
\begin{restatable}
{proposition}{Permutation}
    \label{theo:permutataion}
    Consider a graph $G = (\mathbf{A}, \mathbf{X})$, where $\mathbf{X} \in \mathbb{R}^{N \times C}$ denotes the node feature matrix and $\mathbf{A} \in \mathbb{R}^{N \times N}$ represents the adjacency matrix. Let $\bar{\mathbf{H}}^{(l)} \in \mathbb{R}^{N \times C}$ be the intermediate node features at layer $l$, before the element-wise application of our \ourmethod.
    Let $\mathbf{P}$ be an $N \times N$ permutation matrix. Then,
    \begin{equation}
        \ourmethod(\mathbf{P} \bar{\mathbf{H}}^{(l)}, \vtheta^{(l)}_P) = \mathbf{P}\: \ourmethod(\bar{\mathbf{H}}^{(l)},  \vtheta^{(l)})
    \end{equation}
    where $\vtheta^{(l)}_P$ and $\vtheta^{(l)}$ are obtained by $\mathbf{P} \bar{\mathbf{H}}^{(l)}$ and $ \bar{\mathbf{H}}^{(l)}$, respectively (\Cref{eqn:digaf}).
\end{restatable}
\fi
\begin{proof}

We break down the proof into two parts. First, we show that $\textsc{GNN}_{\textsc{act}}$ outputs the same $\vtheta$ under permutations, that is we show
\begin{equation*}
    \vtheta^{(l)}_P = \vtheta^{(l)}.
\end{equation*}
Second, we prove that the activation function $T^{(l)}$ is permutation equivariant, ensuring the overall method maintains this property.

To begin with, recall that \Cref{eqn:digaf} is composed by $\textsc{GNN}_{\textsc{act}}$, which is permutation equivariant, and by a pooling layer, which is permutation invariant.  Therefore their composition is permutation invariant, that is $\vtheta^{(l)}_P = \vtheta^{(l)}$. 

Prior to the activation function layer $T^{(l)}$, $\bar{\mathbf{H}}^{(l)}$ undergoes rescaling as described in \cref{sec:appendix:implementation}, which is permutation equivariant as it operates element-wise. Finally, since activation function $T^{(l)}$ acts element-wise, and given that $\vtheta$ remains unchanged, the related CPA velocity fields are identical, resulting in the same transformed output for each entry, despite the entries being permuted in $\mathbf{P} \bar{\mathbf{H}}^{(l)}$. Therefore,  \ourmethod\ is permutation equivariant.
\end{proof}

\subsection{Diffeomorphic Activation Functions}
In this section, we provide several examples of popular and well-known diffeomorphic functions, contributing to our motivation to utilize diffeomorphisms as a blueprint for learning graph activation functions. 
We remark that, differently from standard activation functions, our \ourmethod\ does not need to follow a predefined, fixed template, but can instead learn a diffeomorphism best suited for the task and input, as $T^{(l)}$ within CPAB can represent a wide range of diffeomorphisms~\citep{freifeld2015highly, freifeld2017transformations}. 

We recall that, as outlined in Section \ref{sec:def:diffeomorphism}, a function is classified as a diffeomorphism if it is \begin{enumerate*}[label=(\arabic*)]
    \item bijective,
    \item differentiable, and
    \item has a differentiable inverse.
\end{enumerate*} 

\textbf{Sigmoid.} We denote the Sigmoid activation function as $\sigma: \mathbb{R} \to (0, 1)$, defined by $$\sigma(x) = \frac{1}{1 + e^{-x}}.$$ To prove that $\sigma$ is a diffeomorphism, we first establish its bijectivity. Injectivity follows from observing that for any distinct points $x_1$ and $x_2$ in $\mathbb{R}$, $\sigma(x_1) = \frac{1}{1 + e^{-x_1}}$ can only equal $\sigma(x_2) = \frac{1}{1 + e^{-x_2}}$ if and only if $x_1 = x_2$. For surjectivity, we represent $x$ as a function of $y$, such that $y = \frac{1}{1 + e^{-x}} \implies x = -\ln \left( \frac{1 - y}{y} \right)$, ensuring that for every $y \in (0, 1)$ there is an element $x \in \mathbb{R}$ such that $\sigma(x)=y$.

To demonstrate differentiability, we examine the derivative of $\sigma$. The derivative $$\frac{d}{dx} \sigma(x) = \sigma(x)(1 - \sigma(x)),$$ which is continuous. Additionally, the inverse function $$\sigma^{-1}(y) = -\ln \left( \frac{1 - y}{y} \right)$$ is also bijective and differentiable. Thus, with all these requirements satisfied, $\sigma$ is indeed a diffeomorphism.

\textbf{Tanh.} The hyperbolic tangent function $$\tanh(x) = \frac{e^x - e^{-x}}{e^x + e^{-x}}$$ is a diffeomorphism from $\mathbb{R}$ to $(-1, 1)$. To establish this, we demonstrate that $\tanh$ is bijective and differentiable, with a differentiable inverse function. 

Firstly, $\tanh$ is injective because if $\tanh(x_1) = \tanh(x_2)$, then $x_1 = x_2$. It is also surjective because for any $y \in (-1, 1)$, there exists $x = \frac{1}{2} \ln \left( \frac{1 + y}{1 - y} \right)$ such that $\tanh(x) = y$. 

The derivative $$\frac{d}{dx}\tanh(x) = 1 - \tanh^2(x)$$ is continuous and positive. Additionally, the inverse function $$\tanh^{-1}(y) = \frac{1}{2} \ln \left( \frac{1 + y}{1 - y} \right)$$ is continuously differentiable. Therefore, $\tanh$ qualifies as a diffeomorphism.

\textbf{Softplus.} 
To establish the Softplus function $$\text{softplus}(x) = \ln(1 + e^x)$$ as a diffeomorphism from $\mathbb{R}$ to $(0, \infty)$, we first demonstrate its injectivity and surjectivity.

Assuming $\text{softplus}(x_1) = \text{softplus}(x_2)$, we obtain $e^{x_1} = e^{x_2}$, implying $x_1 = x_2$, hence establishing injectivity. For any $y \in (0, \infty)$, we find an $x \in \mathbb{R}$ such that $y = \ln(1 + e^x)$, ensuring surjectivity.

The derivative of the Softplus function, $$\frac{d}{dx}\text{softplus}(x) = \frac{e^x}{1 + e^x} = \sigma(x),$$ where $\sigma(x)$ is the Sigmoid function, known to be continuous and differentiable. Therefore, $\text{softplus}(x)$ is continuously differentiable.

Considering the inverse of the Softplus function, $$\text{softplus}^{-1}(y) = \ln(e^y - 1),$$ its derivative is $$\frac{d}{dy} \text{softplus}^{-1}(y) = \frac{e^y}{e^y - 1},$$ which is continuous for all $y > 0$, indicating that $\text{softplus}^{-1}(y)$ is continuously differentiable for all $y > 0$.
Therefore, we conclude that the softplus function qualifies as a diffeomorphism.

\textbf{ELU.} The ELU activation function~\citep{clevert2015fast} is defined as below:
\begin{align*}
  \text{ELU}(x) =
    \begin{cases}
      x & \text{if $x > 0$}\\
      \alpha(e^x - 1) & \text{if $x \leq 0$}
    \end{cases} 
\end{align*}
where $\alpha \in \mathbb{R}$ is a constant that scales the negative part of the function.
To demonstrate that ELU is bijective, we analyze its injectivity and surjectivity. For $x>0$, ELU acts as the identity function, which is inherently injective. For $x \leq 0$, $\alpha(e^{x_1} - 1) = \alpha(e^{x_2} - 1)$,  implies $x_1 = x_2$. The inverse function for ELU is given by:
\begin{equation*}
    \text{ELU}^{-1}(y) = \begin{cases}
    y  & \text{if $y > 0$} \\ 
    \ln(\frac{y}{\alpha} + 1) &  \text{if $y \leq 0$}
    \end{cases}
\end{equation*}
This inverse maps every value in the codomain back to a unique value in the domain, proving that ELU is surjective.

Next, we examine the continuity of ELU. At $x = 0$, $\text{ELU}(x=0) = \alpha (e^0-1)=0$. Next, we check the limits for both sides of $0$. For $x > 0$, $\lim_{x \to 0^+} \text{ELU}(x) = \lim_{x \to 0^+} x = 0$, while for $x \leq 0$, we have $\lim_{x \to 0^-} \text{ELU}(x) = \lim_{x \to 0^-} \alpha(e^{x} - 1) = 0$. Since both limits are equal, the ELU function is continuous at $x = 0$. For the derivative of ELU, i.e.,
\begin{align*}
        \frac{d}{dx} \text{ELU}(x) =
        \begin{cases}
          1 & \text{if $x > 0$}\\
          \alpha e^x & \text{if $x \leq 0$}
        \end{cases}
\end{align*}
at $x = 0$, we have $\frac{d}{dx}\text{ELU}(x) = \alpha e^0 = \alpha$. By setting $\alpha=1$, the derivative at $x=0$ matches the derivative for $x > 0$, making the derivative continuous. 

The derivative for the inverse function is
\begin{equation*}
    \frac{d}{dy}\text{ELU}^{-1}(y) = 
    \begin{cases}
        1 & \quad \text{if $y > 0$}\\
        \frac{1}{y + \alpha} & \quad \text{if $y \leq 0$}
    \end{cases}
\end{equation*}
which is also continuously differentiable. Hence, ELU is a diffeomorphism.

\section{Additional Results}\label{sec:app:additional_results}

\subsection{Function Approximation with CPAB}
\label{app:funcapprox}

\begin{figure}[h]
    \centering
\includegraphics[width=0.5\linewidth]{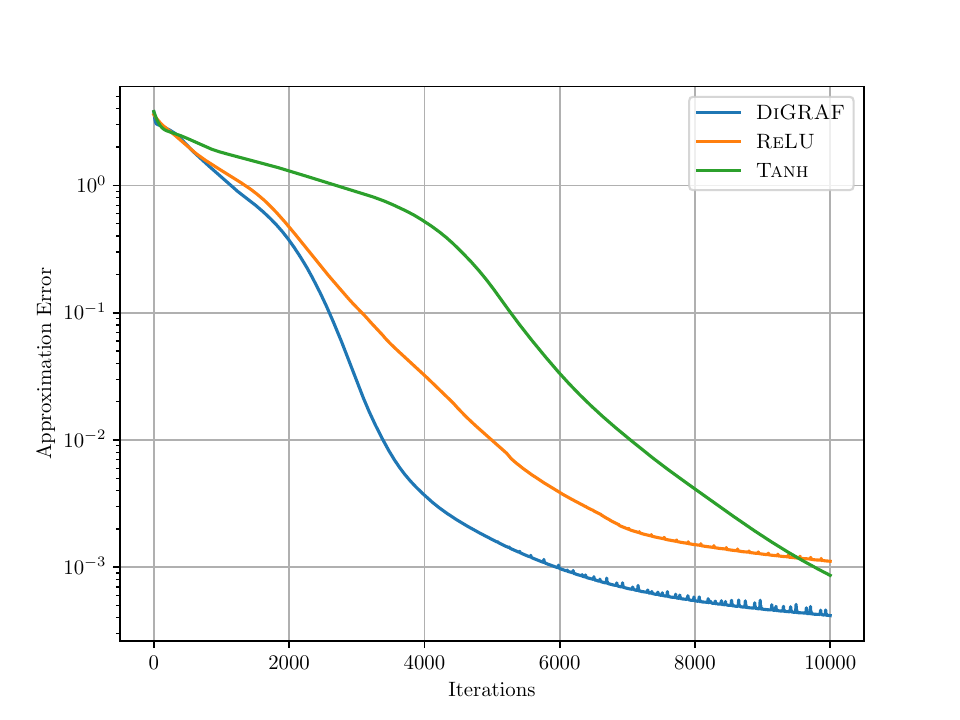}
    \caption{The approximation error of the Peaks function (\Cref{eq:peaks}) with ReLU, Tanh, and \ourmethod.}
    \label{fig:approx_peaks}
\end{figure}

The combination of learned linear layers together with non-linear functions such as ReLU and Tanh are well-known to yield good function approximations \citep{daubechies2022nonlinear,de2021approximation}. Therefore, when designing an activation function blueprint, i.e., the template by which the activation function is learned, it is important to consider its approximation power. In \Cref{sec:introduction} and in particular in \Cref{fig:approximate:flexibility}, we demonstrate the ability of the CPAB framework to approximate known activation functions. We now show additional evidence for the flexibility and power of CPAB as a framework for learning activation functions, leading to our \ourmethod. To this end, we consider the ability of a multilayer perceptron (MLP) with various activation functions (ReLU, Tanh, and \ourmethod) to approximate the well-known \emph{`peaks'} function that mathematically reads:
\begin{equation}
    \label{eq:peaks}
        g(x,y) =  3(1-x)^2  \exp(-(x^2) - (y+1)^2) \
       - 10(\frac{x}{5} - x^3 - y^5) \exp(-x^2-y^2) \
       - \frac{1}{3}\exp(-(x+1)^2 - y^  2).
\end{equation}
The peaks function in \Cref{eq:peaks} is often times used to measure the ability of methods to approximate functions \citep{haber2017stable}, where the input is point pairs $(x,y)\in \mathbb{R}^2$, and the goal is to minimize the mean-squared-error between the predicted function value $g$ and the actual function value $x$. Formally, we consider the following MLP:
\begin{equation}
    \label{eq:mlp}
    \hat{g}(x,y) = (\sigma( \sigma(\begin{bmatrix}
        x,y
    \end{bmatrix}W_1)W_2) W_3 ) ,
\end{equation}
where $\sigma$ is the activation of choice (ReLU, Tanh, or \ourmethod), and $W_1 \in \mathbb{R}^{2 \times 64}$, $W_2 \in \mathbb{R}^{64 \times 64}$, $W_3 \in \mathbb{R}^{64 \times 1}$ are the trainable parameter matrices of the linear layers in the MLP. The goal, as discussed above, is to minimize the loss $ \| \hat{g}(x,y) - g(x,y) \|_2 $, for data triplets $\left(x_i,y_i,g(x_i,y_i)\right)$ sampled from the peaks function. In our experiment, we sample 50,000 points, and report the obtained approximation error in terms of MSE in \Cref{fig:approx_peaks}. As can be seen, our \ourmethod, based on the CPAB framework, allows to obtain a significantly lower approximation error, up to 10 times lower (better) than ReLU, and 3 times better than Tanh. This example further motivates us to harness CPAB as the blueprint of \ourmethod.

\subsection{Results on TUDatasets}
Our results are summarized in~\Cref{tab:tud_datasets}, where we consider \textsc{MUTAG}, \textsc{PTC}, \textsc{PROTEINS}, \textsc{NCI1} and \textsc{NCI109} datasets from the TU repository~\citep{morris2020tudataset}. As can be seen from the table, \ourmethod\ is consistently among the top-3 best-performing activation functions, and it consistently outperforms other graph-adaptive activation functions. These results support our design choices for \ourmethod\ and the flexibility offered by CPAB diffeomorphisms.
\begin{table*}[t]
    \centering
    \scriptsize
    \caption{Graph classification accuracy (\%) $\uparrow$ on TUDatasets. 
    The top three methods are marked by \textbf{\textcolor{red}{First}}, \textbf{\textcolor{violet}{Second}}, \textbf{Third}.}
    \label{tab:tud_datasets}
    \begin{tabular}{l ccccc }
        \toprule
        Method $\downarrow$ / Dataset $\rightarrow$ & 
        \textsc{MUTAG} &
        \textsc{PTC} &
        \textsc{PROTEINS} &
        \textsc{NCI1} &
        \textsc{NCI109} % &

        \\
        \midrule  

        \textbf{\textsc{Standard Activations}} \\
        $\,$  \textsc{GIN} + Identity & 91.4$\pm$5.6 & 66.2$\pm$5.5 & 75.9$\pm$3.2 & 82.8$\pm$2.0 & 82.8$\pm$1.3\\
        $\,$ \textsc{GIN} + Sigmoid~\citep{rumelhart1986learning} & 90.9$\pm$5.5 & 65.3$\pm$4.8 & 75.0$\pm$5.0 & 82.6$\pm$1.4 & 81.2$\pm$1.6 \\
        $\,$       \textsc{GIN} + ReLU \citep{xu2019how} & 
        89.4$\pm$5.6 & 
        64.6$\pm$7.0 & 
        76.2$\pm$2.8 &
        82.7$\pm$1.7 &
        82.2$\pm$1.6 % &

        \\
       $\,$ \textsc{GIN} + LeakyReLU~\citep{Maas2013RectifierNI}        & 90.9$\pm$5.7 & 65.0$\pm$9.0 & 76.2$\pm$4.4 & 
       \second{83.5$\pm$2.2} &
       \third{82.9$\pm$2.0}\\
        $\,$ \textsc{GIN} + Tanh~\citep{hochreiter1997long}    & \first{92.5$\pm$7.9 }   & 65.1$\pm$6.5 & 75.9$\pm$4.3 &
        83.2$\pm$2.6 &
        \second{{83.0$\pm$2.6}}\\
        $\,$ \textsc{GIN} + GeLU~\citep{hendrycks2023gaussian} & 90.9$\pm$7.0        & 65.4$\pm$7.9 & 76.6$\pm$2.8 &
        \second{83.5$\pm$1.4} &
        \third{82.9$\pm$1.6}\\
        $\,$ \textsc{GIN} + ELU~\citep{clevert2016fast} & \first{92.5$\pm$5.6} & 65.4$\pm$7.5 & 75.4$\pm$2.7 & 83.3$\pm$2.0 & 82.6$\pm$1.7\\
        \midrule
        \textbf{\textsc{Learnable Activations}} \\
        $\,$ \textsc{GIN} + PReLU~\citep{he2015delving} &91.7$\pm$6.7 & \third{66.9$\pm$7.0} & \third{76.7$\pm$3.5} & 82.9$\pm$2.6 & 82.3$\pm$1.8     \\
        $\,$ \textsc{GIN} + Maxout~\citep{pmlr-v28-goodfellow13} &  91.5$\pm$7.5  & 66.8$\pm$8.3 & \second{76.8$\pm$4.0} & 83.3$\pm$2.9 & \second{83.0$\pm$3.0}  \\ 
        $\,$ \textsc{GIN} + Swish~\citep{ramachandran2017searching} & 90.4$\pm$4.8 & 65.1$\pm$6.3 & 76.2$\pm$4.2 & \third{83.4$\pm$1.4} & \third{82.9$\pm$3.0}\\
        \midrule 
        
        \textbf{\textsc{Graph Activations}} \\
         $\,$ \textsc{GIN} + Max \citep{iancu2020graph} & 90.9$\pm$7.1 & \second{67.7$\pm$9.2} & 75.9$\pm$3.1 & 83.3$\pm$2.0 & 82.7$\pm$1.9\\
        $\,$ \textsc{GIN} + Median \citep{iancu2020graph} & \third{92.0$\pm$6.6} & \second{67.7$\pm$4.5} & 75.0$\pm$4.3 & \first{83.6$\pm$1.9} & 82.8$\pm$1.8\\
        $\,$ \textsc{GIN} + GReLU \citep{zhang2022graph} & \third{92.0$\pm$7.3} & 64.9$\pm$6.6 & \second{76.8$\pm$3.5} & 
        82.8$\pm$2.5 &
82.4$\pm$2.2
        \\
        \midrule
        
                  \textsc{GIN} + \ourmethodshared\ & 92.0$\pm$5.6 & \first{68.9$\pm$7.5} & 77.2$\pm$3.6 & 
                  83.0$\pm$1.3 & 82.9$\pm$2.2\\

         \textsc{GIN} + \ourmethod\  & 
         \second{92.1$\pm$7.9} & 
         68.6$\pm$7.4 & 
         \first{77.9$\pm$3.4} & \third{83.4$\pm$1.2} & \textbf{\textcolor{red}{83.3$\pm$1.9}}
         \\

        \bottomrule

    \end{tabular}
\end{table*}
\revision{

\subsection{Comparison with Different GNN Architectures}

\begin{table*}[t]
\centering
\scriptsize
\caption{\revision{Different \textsc{GNN} architectures (\textcolor{red}{GCN}, \textcolor{blue}{GAT}, \textcolor{violet}{GIN}, \textcolor{orange}{SAGE}) coupled with ReLU, \ourmethodshared\ and \ourmethod\ activation functions. The top performing model is marked with the corresponding color for each architecture.}}
\label{tab:gnn_structures}
\begin{tabular}{lcccccccc}
\toprule
    Activation & Model &\textsc{Cora} $\uparrow$ & \textsc{CiteSeer} $\uparrow$ &  \textsc{PubMed} $\uparrow$ & \textsc{Blog} & \textsc{Flickr} $\uparrow$  & \textsc{ZINC} $\downarrow$ &  \textsc{molhiv} $\uparrow$ \\
    & & & & & \textsc{Catalog} $\uparrow$ & & \\
\midrule
   \multirow{4}{*}{ReLU} &  GCN & 79.2$\pm$1.4 & 67.7$\pm$2.3 & 77.6$\pm$2.2 & 72.1$\pm$1.9 & 50.7$\pm$2.3 & 0.3674$\pm$0.0111 & 76.06$\pm$0.97 \\
    &  GAT & 78.0$\pm$2.1 & 63.6$\pm$1.9 & 77.0$\pm$1.7 & 74.2$\pm$1.8 & 55.5$\pm$1.1 & 0.3842$\pm$0.0070 & 76.00$\pm$0.82\\
    &  GIN & 67.1$\pm$3.0  & 58.8$\pm$2.2 & 68.4$\pm$2.7 & 72.6$\pm$2.5 & 43.1$\pm$2.6  & 0.1630$\pm$0.0040 & 75.58$\pm$1.40 \\
    &  SAGE &  78.5$\pm$1.6 & 67.4$\pm$1.8 & 76.2$\pm$1.8 & 84.9$\pm$3.1 & 43.5$\pm$2.8 & 0.4680$\pm$0.0030 & 77.46$\pm$0.91 \\
\midrule
    \multirow{2}{*}{\ourmethod} &  GCN & 81.5$\pm$1.1 & 69.2$\pm$2.1 & 78.3$\pm$1.6 & 80.8$\pm$0.6 & 68.6$\pm$1.8 & 0.3187$\pm$0.0083 & 76.62$\pm$1.20 \\
    &  GAT & \textcolor{blue}{81.0$\pm$2.1} & 69.3$\pm$1.7 & 78.2$\pm$2.0 & \textcolor{blue}{79.3$\pm$2.8} & 62.8$\pm$6.9 & 0.3309$\pm$0.0115 & 76.80$\pm$1.14 \\
  \multirow{2}{*}{\textsc{(W/O Adap.)}}  &  GIN & \textcolor{violet}{80.6$\pm$2.3} & 67.5$\pm$4.2 & 76.0$\pm$4.0 & 82.1$\pm$3.5 & 68.0$\pm$1.3 & 0.1382$\pm$0.0082 & 79.19$\pm$1.36 \\
    &  SAGE & 79.3$\pm$7.8 & 67.7$\pm$2.5 & 77.1$\pm$3.2 & 90.6$\pm$0.3 & 66.5$\pm$6.5 & 0.4442$\pm$0.0097 & 78.19$\pm$0.83 \\
\midrule
    \multirow{4}{*}{\ourmethod} &  GCN & \textcolor{red}{82.8$\pm$1.1} & \textcolor{red}{69.5$\pm$1.4} & \textcolor{red}{79.3$\pm$1.4} & \textcolor{red}{81.6$\pm$0.8} & \textcolor{red}{69.6$\pm$0.6} & \textcolor{red}{0.2830$\pm$0.0054} & \textcolor{red}{77.38$\pm$2.31} \\
    &  GAT & \textcolor{blue}{81.0$\pm$1.5} & \textcolor{blue}{69.4$\pm$2.4} & \textcolor{blue}{78.9$\pm$2.5} & \textcolor{blue}{79.3$\pm$4.2} & \textcolor{blue}{62.9$\pm$1.0} & \textcolor{blue}{0.2918$\pm$0.0133} & \textcolor{blue}{77.47$\pm$1.18} \\
    &  GIN & \textcolor{violet}{80.6$\pm$2.0} & \textcolor{violet}{68.9$\pm$3.8} & \textcolor{violet}{76.9$\pm$3.3} & \textcolor{violet}{83.0$\pm$4.1} & \textcolor{violet}{70.6$\pm$4.3} & \textcolor{violet}{0.1302$\pm$0.0090} & \textcolor{violet}{80.28$\pm$1.44} \\
    &  SAGE & \textcolor{orange}{79.9$\pm$6.9} & \textcolor{orange}{68.0$\pm$4.0} & \textcolor{orange}{77.3$\pm$3.3} & \textcolor{orange}{90.8$\pm$0.4} & \textcolor{orange}{69.0$\pm$4.8} & \textcolor{orange}{0.4147$\pm$0.0078} & \textcolor{orange}{79.32$\pm$0.74}\\
\bottomrule
% \end{tblr}
% \end{adjustbox}
\end{tabular}
\end{table*}

We now provide a comparison between  \ourmethod\ and the ReLU activation function coupled with GCN, GAT, GIN, and SAGE backbones in \cref{tab:gnn_structures}. Notably, \ourmethod\ consistently outperforms ReLU regardless of the backbone architecture.
}

\revision{
\subsection{Visualization of \ourmethod}
\begin{figure}
\begin{tabular}{cc}
  \includegraphics[width=0.5\linewidth]{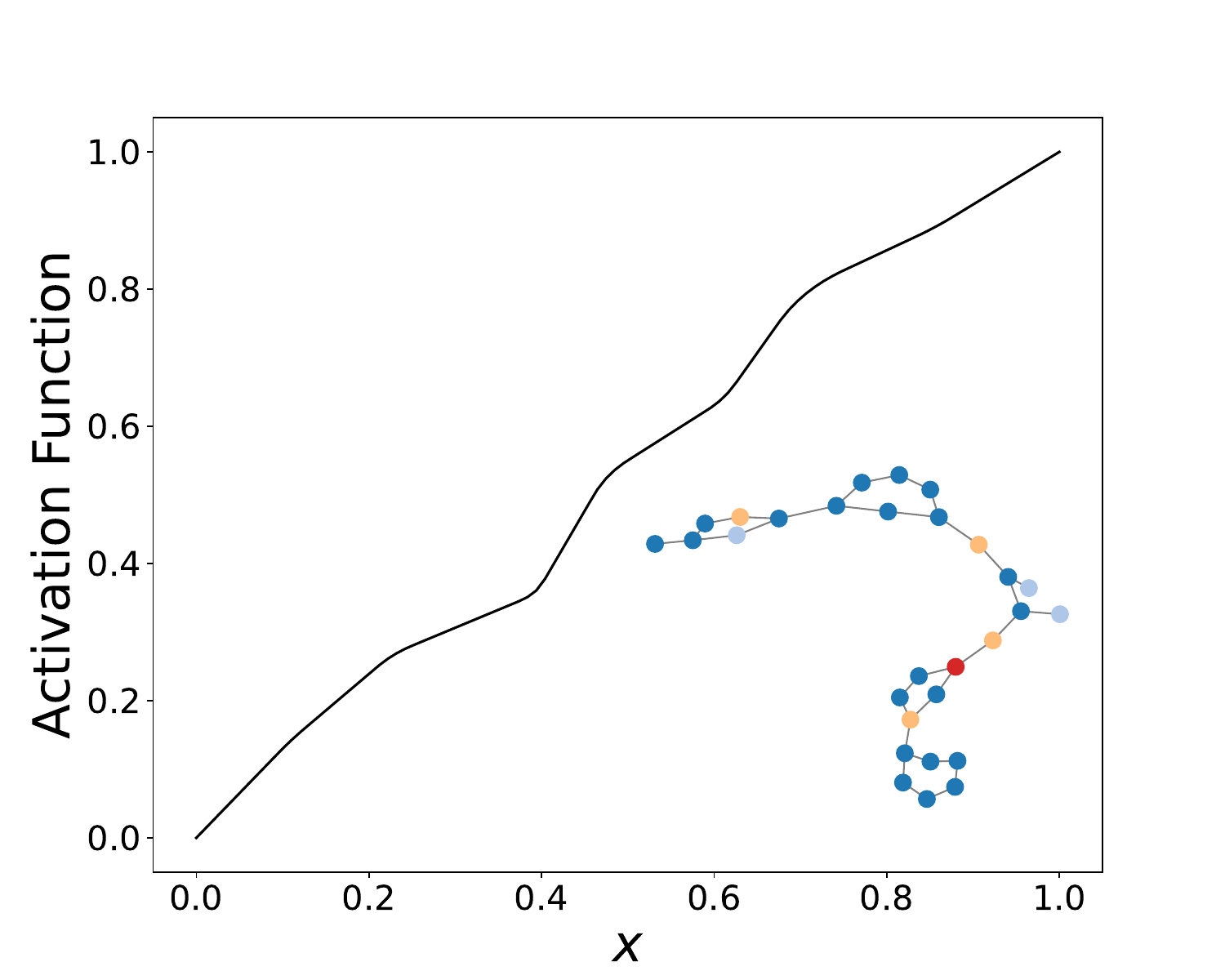} &   \includegraphics[width=0.5\linewidth]{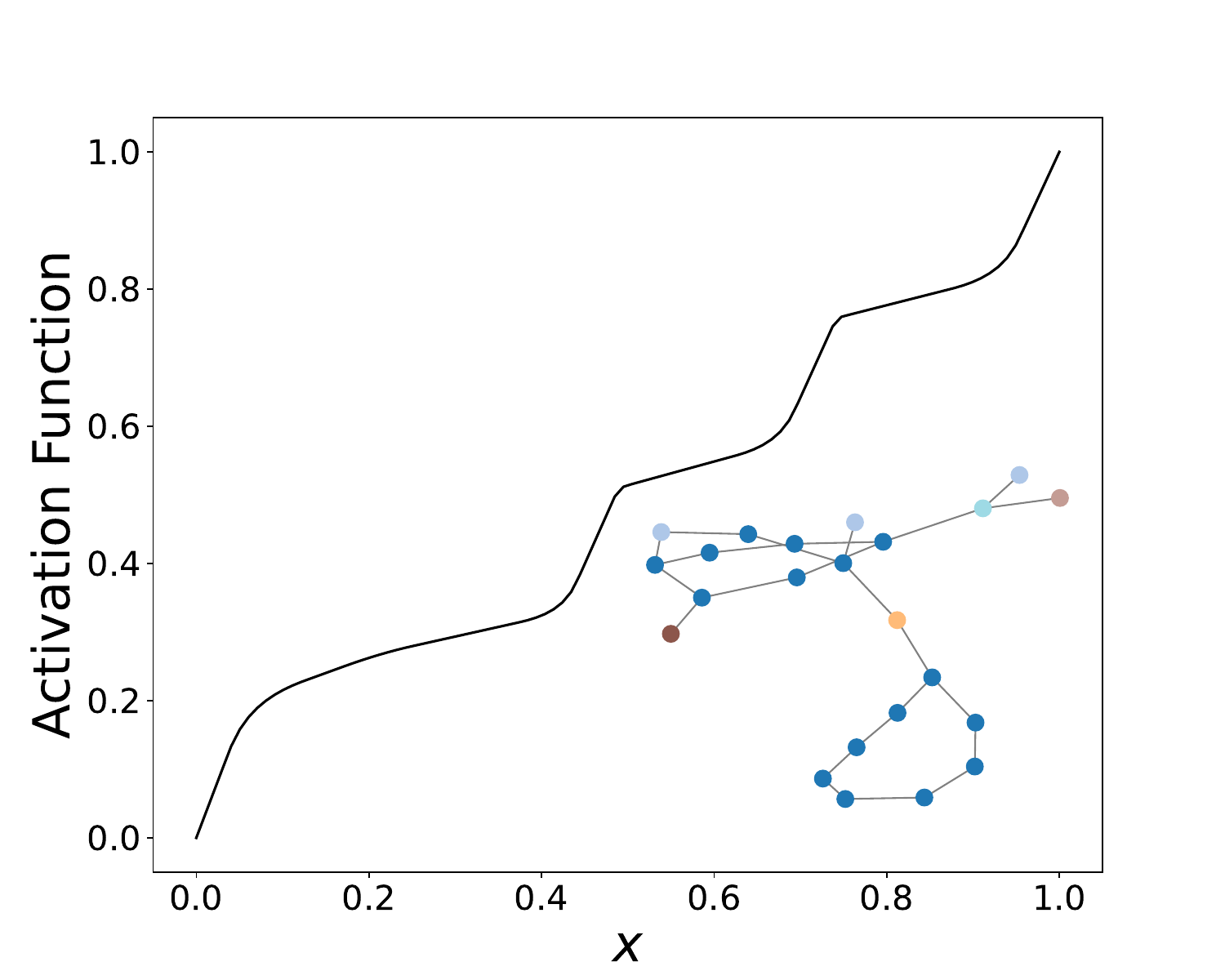} \\
(a) $\textsc{ZINC}$ Test Graph $9$  & (b) $\textsc{ZINC}$ Test Graph $141$ \\[6pt]
\end{tabular}
\caption{\revision{Activation function learned by \ourmethod\ after the last $\textsc{GNN}$ layer on two randomly selected graphs from \textsc{ZINC}. Different node colors indicate different node features. \ourmethod\ yields different activations for different graphs.}}
\label{fig:different_DiGRAF}
\end{figure}
To gain a qualitative understanding of the behavior of \ourmethod, we now illustrate the activation function learned by \ourmethod\ after the last GNN layer on different graphs. To this end, we randomly selected two graphs from the \textsc{ZINC} dataset, as shown in Figure \ref{fig:different_DiGRAF}. The original graphs are presented in the lower right, with each color representing a feature. Nodes with the same color share the same feature. The comparison of the figures demonstrates that for different graphs, with different features and structures, \ourmethod\ learns distinct activation functions, showing its adaptivity to the input graph.
}

\subsection{Parameter Count Comparison} $\textsc{GNN}_{\textsc{act}}$ is a core component of \ourmethod, which ensures graph-adaptivity by generating the parameters $\vtheta^{(l)}$ of the activation function conditioned on the input graph. While the benefits of graph-adaptive activation functions are evident from our experiments in \Cref{sec:experiments}, as \ourmethod\ consistently outperforms \ourmethodshared, the variant of our method that is not graph adaptive, it comes at the cost of additional parameters to learn $\textsc{GNN}_{\textsc{act}}$ (\Cref{eqn:digaf}). 
Specifically, because in all our experiments $\textsc{GNN}_{\textsc{act}}$ is composed of 2 layers and a hidden dimension of 64, \ourmethod\ adds at most approximately \textsc{20k} additional parameters. The number of added parameters in \ourmethodshared\ is significantly lower, counting at $\gN_{\gP} - 1$, where $\gN_{\gP}$ is the tessellation size. Note in our experiments, the tessellation size does not exceed 16. 
To further understand whether the improved performance of \ourmethod\ is due to the increased number of parameters, we conduct an additional experiment using the ReLU activation function where we increase the number of parameters of the model and compare the performances. In particular, we consider following settings: \begin{enumerate*}[label=(\arabic*)]
    \item The standard variant (GIN + ReLU), 
    \item The variant obtained by doubling the number of layers, and
    \item The variant is obtained by doubling the number of hidden channels. 
\end{enumerate*}

We present the results of the experiment described above on the \textsc{ZINC-12k} and \textsc{molhiv} datasets in \Cref{tab:cap_comparision}. 
 We observed that adding
more parameters to the ReLU baseline does not produce significant performance improvements, even in cases where the baselines have $\sim$4 times more parameters than \ourmethod and its baseline. On the contrary, with \ourmethod\, significantly improved performance is obtained compared to the baselines. 

\begin{table*}[t]
\centering
\caption{Performance Comparison of \ourmethod\ with ReLU variants of increased parameter budget. The number of parameters is reported within the parenthesis adjacent to the metric. We use GINE \citep{hu2020strategies} as a backbone. Increasing the parameter count with ReLU does not yield significant improvements, and \ourmethod\ outperforms all variants, even those with a higher number of parameters. Note that, \ourmethodshared\ has only $\gN_\gP - 1$ additional parameters, where  $\gN_\gP$ is the tessellation size.}
\footnotesize
\begin{tabular}{lcc}
    \toprule
    Method $\downarrow$ / Dataset $\rightarrow$ & \textsc{ZINC (MAE $\downarrow$)} & \textsc{molhiv (Acc. \% $\uparrow$)}\\ %& 
    \midrule
    \textsc{GIN} + ReLU (standard) & 0.1630$\pm$0.0040 $(\sim \phantom{1}\textsc{308k})$ & 75.58$\pm$1.40 $(\sim \phantom{1}\textsc{63k})$ \\ %& 
    \textsc{GIN} + ReLU (double \#channels) & 0.1578$\pm$0.0014 $(\sim \textsc{1207k})$ & 75.73$\pm$0.71 $(\sim \textsc{240k})$ \\ %& 
    \textsc{GIN} + ReLU (double \#layers) & 0.1609$\pm$0.0033 $(\sim \phantom{1}\textsc{580k})$ & 75.78$\pm$0.43 $(\sim \textsc{116k})$\\ 
    \midrule
    \ourmethodshared & 0.1382$\pm$0.0080 $(\sim \textsc{308k})$ & 79.19$\pm$1.36 $(\sim \textsc{63k})$ \\ 
    \ourmethod & \textbf{0.1302$\pm$0.0090} $(\sim \textsc{333k})$ & \textbf{80.28$\pm$1.44} $(\sim \textsc{83k})$ \\ 
    \bottomrule
\end{tabular}

\label{tab:cap_comparision}
\end{table*}

\section{Ablation Studies}
We present the impact of several key components of \ourmethod, namely the tessellation size $\gN_\gP$, the depth of $\textsc{GNN}_{\textsc{act}}$ (\Cref{eqn:digaf}) and the regularization coefficient $\lambda$ of $\vtheta^{(l)}$ (\Cref{eqn:total_loss}). We choose a few representative datasets, i.e., 
$\textsc{Cora}$, $\textsc{CiteSeer}$, 
$\textsc{Flickr}$ and $\textsc{BlogCatalog}$ for which we use GCN~\citep{kipf2017semisupervised}; and $\textsc{ZINC-12k}$ and $\textsc{molhiv}$ for which we use GINE~\citep{hu2020strategies} as GNN respectively.
\label{sec:app:ablation}

\subsection{Tessellation Size}
Recall that the tessellation size $\gN_\gP$ determines the dimension of $\vtheta^{(l)} \in \mathbb{R}^{\gN_\gP-1}$ that parameterizes the velocity fields within \ourmethod. We study the effect of the tessellation size on the performance of \ourmethod\ in \Cref{fig:ablation:tess}. 
We can see that a small tessellation size is sufficient for good performance, and increasing its size results in marginal changes. \revision{This observation suggests that CPAB is highly flexible, and aligns with the conclusions in previous studies on different applications of CPAB~\citep{martinez2022closed}, which have shown that small sizes are sufficient in most cases.}

\begin{figure}[t]
    \centering
    \includegraphics[width=0.5\linewidth]{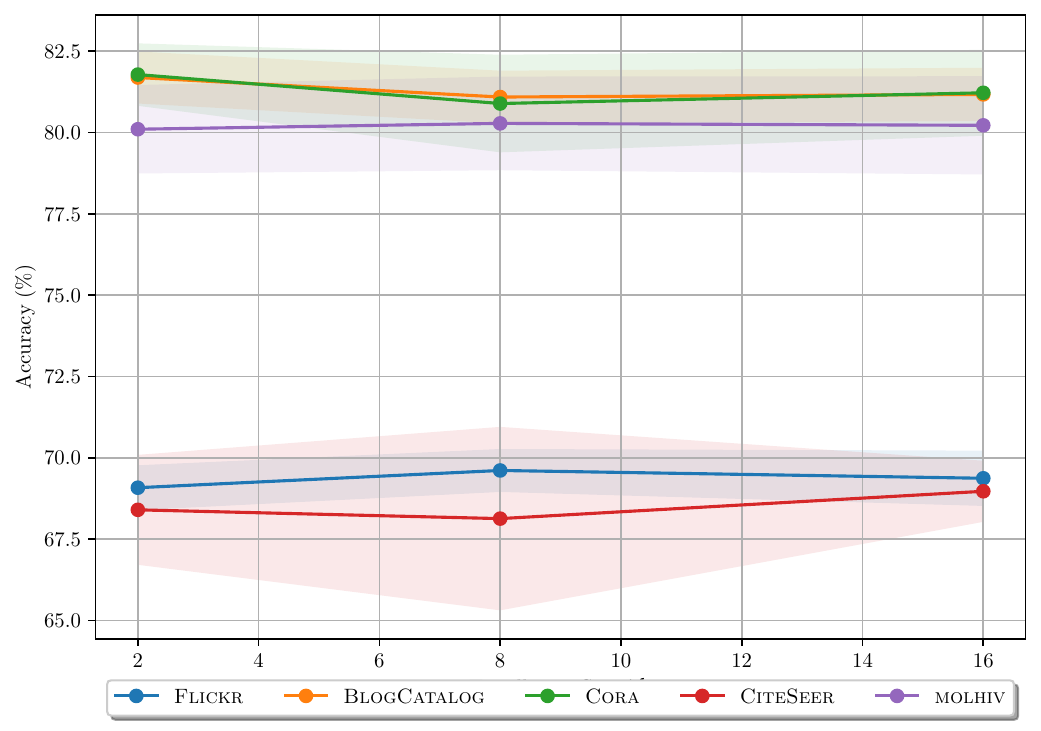}
    \caption{Impact of tessellation size $\gN_\gP$ on the performance of \ourmethod\ on $\textsc{Cora}$, $\textsc{CiteSeer}$, \textsc{Flickr},  \textsc{BlogCatalog}, and \textsc{molhiv} datasets. }
    \label{fig:ablation:tess}
\end{figure}

\subsection{Depth of $\textsc{GNN}_{\textsc{act}}$}
\ourmethod\ exhibits graph adaptivity by predicting  $\vtheta^{(l)} \in \mathbb{R}^{\gN_\gP-1}$ conditioned on the input graph through $\textsc{GNN}_{\textsc{act}}$. \Cref{tab:ablation:depth} shows the impact of the number of layers $L_{\textsc{act}}$ of $\textsc{GNN}_{\textsc{act}}$ on the performance of \ourmethod. In particular, we maintain a fixed architecture for \ourmethod\ and vary only $L_{\textsc{act}}$. The results show that increasing the depth of $\textsc{GNN}_{\textsc{act}}$ improves only marginally the performance of $\ourmethod$, demonstrating that the increased number of parameters is not the main factor of the better performance of \ourmethod. On the contrary, the flexibility and adaptivity offered by \ourmethod\ are the main factors of the improvements, as demonstrated by \ourmethod\ consistently outperforming \ourmethodshared\ and other activation functions (\Cref{sec:experiments}).

\begin{table*}[t]
    \centering
    \scriptsize
    \caption{Effect of depth of $\textsc{GNN}_{\textsc{act}}$ on \ourmethod.}
    \begin{tabular}{l ccc}
    \toprule
        Dataset & $L_{\textsc{act}} = 2$ & $L_{\textsc{act}} = 4$ & $L_{\textsc{act}} = 6$\\
        \midrule
        \textsc{Flickr} & 69.6$\pm$0.6 & 66.3$\pm$0.8 & 69.3$\pm$0.7\\
        \textsc{BlogCatalog} &
        81.0$\pm$0.5 & 81.1$\pm$0.48 & 81.6$\pm$0.8\\
        \textsc{molhiv} & 80.28$\pm$1.44 & 80.19$\pm$1.49 & 80.22$\pm$1.56\\
        \textsc{ZINC} & 0.1302$\pm$0.0090 & 0.1309±0.0084 & 0.1314$\pm$0.083\\
         \bottomrule
    \end{tabular}
    \label{tab:ablation:depth}
\end{table*}

\subsection{Regularization}

As discussed in \Cref{sec:methods:models}, the regularization enforces the smoothness of the velocity field. We investigate the impact of the value of the regularization coefficient $\lambda$ on \ourmethod\ (\Cref{eqn:total_loss}) in \Cref{tab:ablation:regularization}. The results reveal that the optimal value of $\lambda$ depends on the dataset of interest, with small positive values yielding generally good results across all datasets.

\begin{table*}[t]
    \centering
    \scriptsize
    \caption{Effect of velocity field regularization coefficient $\lambda$ on \ourmethod.}
    \begin{tabular}{l cccc}
    \toprule
        Dataset & $\lambda = 0.0$ & $\lambda = 0.001$ & $\lambda = 0.01$ &
        $\lambda = 1.0$\\
        \midrule
        \textsc{Flickr (\% Acc $\uparrow$)} & 
        69.1$\pm$0.9 & 68.7$\pm$0.7 &
        69.6$\pm$0.6 &
        69.0$\pm$0.9\\
        \textsc{BlogCatalog (\% Acc $\uparrow$)} &
        80.5$\pm$0.9 &
        81.0$\pm$0.8 &
        81.4$\pm$1.0 &
        81.6$\pm$0.8\\
        \textsc{molhiv (\% Acc $\uparrow$)} & 79.38$\pm$2.10 & 80.28$\pm$1.44 & 80.16$\pm$1.50 & 78.15$\pm$1.29\\
    \textsc{ZINC (MAE $\downarrow$)} & 0.1395$\pm$0.0102 & 0.1348$\pm$0.0093 & 0.1302$\pm$0.0090 & 0.1353$\pm$0.0071\\
         \bottomrule
    \end{tabular}
    \label{tab:ablation:regularization}
\end{table*}

\revision{
\subsection{Comparison of \ourmethod\ and \ourmethodshared\ with Equal Parameter Budget}

To demonstrate the efficacy of graph adaptivity provided by $\text{GNN}_{\textsc{act}}$, we conduct an experiment where we increase the number of layers and channels of $\text{GNN}_{\textsc{Layer}}$ in \ourmethodshared\ to match the total number of parameters in \ourmethod. As shown in \cref{tab:graph_adaptivity}, the increase in the number of parameters does not translate to better performance. Rather, the effective usage of the extra parameters as done by $\text{GNN}_{\textsc{act}}$ is the reason behind the performance boost offered by \ourmethod.

\begin{table*}[t]
    \centering
    \scriptsize
    \caption{Results on \textsc{ZINC} and \textsc{molhiv} datasets along with number of parameters in parenthesis.}
    \label{tab:graph_adaptivity}
    \begin{tabular}{lcc}
        \toprule
        \textbf{Method} & \textsc{ZINC (MAE)} $\downarrow$ & \textsc{MOLHIV (ROC AUC)} $\uparrow$ \\
        \midrule
        GIN + \ourmethodshared\ with larger $\text{GNN}_{\text{Layer}}$& 0.1388 $\pm$ 0.0071 \textsc{(337K)} & 79.22 $\pm$ 1.40 \textsc{(85K)} \\
        GIN + \ourmethodshared\ (Original) & 0.1382 $\pm$ 0.0086 \textsc{(308K)} & 79.19 $\pm$ 1.36 \textsc{(63K)} \\
       GIN + \ourmethod & 0.1302 $\pm$ 0.0094 \textsc{(333K)} & 80.28 $\pm$ 1.44 \textsc{(83K)} \\
        \bottomrule
    \end{tabular}
\end{table*}

}

\section{Experimental Details} \label{sec:appendix:experiment}
We implemented \ourmethod\ using PyTorch~\citep{paszke2019pytorch} (offered under BSD-3 Clause license) and the PyTorch Geometric library~\citep{fey2019fast} (offered under MIT license). All experiments were conducted on NVIDIA RTX A5000, NVIDIA GeForce RTX 4090, NVIDIA GeForce RTX 4070 Ti Super, NVIDIA GeForce GTX 1080 Ti, NVIDIA TITAN RTX and NVIDIA TITAN V GPUs. For hyperparameter tuning and model selection, we utilized the Weights and Biases (wandb) library~\citep{wandb}. We used the \texttt{difw} package~\citep{martinez2022closed, freifeld2017transformations, freifeld2015highly, shapira2019diffeomorphic} (offered under MIT license) for the diffeomorphic transformations based on the closed-form integration of CPA velocity functions.
In the following subsections, we present the experimental procedure, dataset details, and hyperparameter configurations for each task.

\paragraph{Hyperparameters.} The hyperparameters include the number of layers $L$ and embedding dimension $C$ of $\textsc{GNN}_{\textsc{layer}}^{(l)}$, learning rates and weight decay factors for both  $\textsc{GNN}_{\textsc{layer}}^{(l)}$ and $\textsc{GNN}_{\textsc{act}}$, dropout rate $p$, tessellation size $\mathcal{N}_{\mathcal{P}}$, and regularization coefficient $\lambda$. We additionally include the number of layers $L_{\textsc{act}}$ and embedding dimension $C_{\textsc{act}}$ of $\textsc{GNN}_{\textsc{act}}$.  We employed a combination of grid search and Bayesian optimization. All hyperparameters were chosen according to the best validation metric. \revision{For the baselines, we include only the applicable hyperparameters in our search space.}

\paragraph{Node Classification.} For each dataset, we train a 2-layer GCN~\citep{kipf2017semisupervised} as the backbone architecture, and integrate each of the activation functions into this model. Following~\citet{zhang2022graph}, we randomly choose 20 nodes from each class for training and select 1000 nodes for testing. For each activation function, we run the experiment 10 times with random partitions. We report the mean and standard deviation of node classification accuracy on the test set.
\Cref{tab:node_cls_stats} summarizes the statistics of the node classification datasets used in our experiments.  All models were trained for 1000 epochs with a fixed batch size of 32 using the Adam optimizer. \Cref{tab:hyperparameters_node_cls1,tab:hyperparameters_node_cls2} lists the hyperparameters and their search ranges or values.

\begin{table*}[t]
\centering
\scriptsize
\caption{Statistics of the node classification datasets~\citep{mccallum2000automating,sen2008collective,namata2012query,Yang_2023}.}
\begin{tabular}{lcccc}
\toprule
Dataset & \#nodes & \#edges & \#features & \#classes\\
\midrule
\textbf{\textsc{Planetoid}}~ \\
$\,$ \textsc{Cora} & 2,708 & 10,556 & 1,433 & 7\\
$\,$ \textsc{CiteSeer} & 3,327 & 9,104 & 3,703 & 6\\
$\,$ \textsc{PubMed} & 19,717 & 88,648 & 500 & 3\\
\textbf{\textsc{Social Networks}}~ \\
$\,$ \textsc{Flickr} & 7,575 & 479,476 & 12,047 & 9\\
$\,$ \textsc{BlogCatalog} & 5,196 & 343,486 & 8,189 & 6\\
\bottomrule
\end{tabular}
\label{tab:node_cls_stats}
\end{table*}

\begin{table*}[t]
\centering
\scriptsize
\caption{Hyperparameter configurations for the Planetoid datasets~\citep{mccallum2000automating,sen2008collective,namata2012query}.}

\begin{tabular}{lc}
\toprule
Hyperparameter & Search Range / Value\\
\midrule
Learning rate for $\textsc{GNN}_{\textsc{layer}}^{(l)}$ & $[10^{-5}, 10^{-4}, 10^{-3}, 5 \times 10^{-3}, 5 \times 10^{-2}]$\\
Learning rate for $\theta^{(l)}$ / $\textsc{GNN}_{\textsc{act}}$ & $[10^{-6}, 5 \times 10^{-6}, 10^{-5}, 10^{-4}, 10^{-3}, 5 \times 10^{-3}]$\\
Weight decay  & $[10^{-5}, 10^{-4}, 5 \times 10^{-3}, 0.0]$\\
$C$ & $[64, 128, 256]$\\
$C_{\textsc{act}}$ & $[64, 128]$\\
$L_{\textsc{act}}$ & $[2, 4]$\\
$p$ & $[0.0, 0.5]$\\
$\mathcal{N}_{\mathcal{P}}$ & $[2, 4, 8, 16]$\\
$\lambda$ & $[0.0, 10^{-3}, 10^{-2}, 1.0]$\\
\bottomrule
\end{tabular}
\label{tab:hyperparameters_node_cls1}
\end{table*}

\begin{table*}[t]
\centering
\scriptsize
\caption{Hyperparameter configurations for the social network datasets~\citep{Yang_2023}.}
\begin{tabular}{lc}
\toprule
Hyperparameter & Search Range / Value\\
\midrule
Learning rate for $\textsc{GNN}_{\textsc{layer}}^{(l)}$ &
$[10^{-5}, 10^{-4}, 5 \times 10^{-4}, 10^{-3}, 5 \times 10^{-2}, 10^{-2}]$\\
Learning rate for $\theta^{(l)}$ / $\textsc{GNN}_{\textsc{act}}$ & 
$[10^{-6}, 10^{-5}, 10^{-4}, 10^{-3}, 5 \times 10^{-3}, 10^{-2}, 5 \times 10^{-2}]$\\
Weight decay for $\textsc{GNN}_{\textsc{layer}}$ & 
$[10^{-5}, 10^{-4}, 5 \times 10^{-3}, 0.0]$\\
Weight decay for $\theta^{(l)}$ / $\textsc{GNN}_{\textsc{act}}$ & 
$[10^{-6}, 10^{-5}, 10^{-4}, 5 \times 10^{-3}, 0.0]$\\
$C$ & $[64, 128, 256]$\\
$C_{\textsc{act}}$ & $[16, 32, 64, 128]$\\
$L$ & $[2, 4]$\\
$L_{\textsc{act}}$ & $[2, 4]$\\
$p$ & $[0.0, 0.4, 0.5, 0.6, 0.7]$\\
$\mathcal{N}_{\mathcal{P}}$ & $[2, 4, 8, 16]$\\
$\lambda$ & $[0.0, 10^{-3}, 10^{-2}, 1.0]$\\
\bottomrule
\end{tabular}
\label{tab:hyperparameters_node_cls2}
\end{table*}

\paragraph{Graph Classification.}
The statistics of various datasets can be found in \Cref{tab:graph_cls_stats}. We consider the following setup:
\begin{itemize}
    \item\textbf{\textsc{ZINC-12k}:} We consider the splits provided in~\citet{dwivedi2022graph}. 
    We use the mean absolute error (MAE) both as the loss and evaluation metric and report the mean and standard deviation over the test set calculated using five different seeds. We use the Adam optimizer and decay the learning rate by 0.5 every 300 epochs, with a maximum of 1000 epochs.
    In all our experiments, we adhere to the \textsc{500k} parameter budget~\citep{dwivedi2022graph}. We use GINE~\citep{hu2020strategies} layers both for $\textsc{GNN}_{\textsc{layer}}^{(l)}$ and within $\textsc{GNN}_{\textsc{act}}$, and we fix $C_{\textsc{act}} = 64$ and $L_{\textsc{act}} = 2$. We report the hyperparameter search space for all the other hyperparameters in \Cref{tab:hyperparameters_combined}. 

    \item \textbf{TUDatasets}:  We follow the standard procedure prescribed in \citet{xu2019how} for evaluation. That is, we use a 10 fold cross-validation and report the mean and standard deviation of the accuracy at the epoch that yields the best validation performance on average. We use the Adam optimizer and train for a maximum of 350 epochs. We use GIN~\citep{xu2019how} layers both for $\textsc{GNN}_{\textsc{layer}}^{(l)}$ and within $\textsc{GNN}_{\textsc{act}}$, and we fix $L_{\textsc{act}} = 2$. We present the hyperparameter search space for all other parameters in \Cref{tab:hyperparameters_combined}. 

    \item \textbf{\textsc{OGB}:} We consider 4 datasets from the OGB repository, with one, namely \textsc{molesol}, being a regression problem, while the others are classification tasks. We run each experiment using five different seeds and report the mean and standard deviation of RMSE/ROC-AUC.  We use the Adam optimizer, decaying the learning rate by a factor of 0.5 every 100 epochs, and train for a maximum of 500 epochs. We use the GINE model with the encoders prescribed in \citet{hu2020ogb}  both for $\textsc{GNN}_{\textsc{layer}}^{(l)}$ and within $\textsc{GNN}_{\textsc{act}}$, and we set $C_{\textsc{act}} = 64$ and $L_{\textsc{act}} = 2$. We present the hyperparameter search space for all other parameters in \Cref{tab:hyperparameters_combined} 
\end{itemize}
% \Cref{tab:graph_cls_stats} presents the statistics of the graph classification datasets.

\begin{table*}[t]
\centering
\scriptsize
\caption{Statistics of the graph classification datasets~\citep{morris2019weisfeiler,hu2020ogb,dwivedi2022graph}.}
\begin{tabular}{lccccc}
\toprule
Dataset & \#graphs & \#nodes & \#edges & \#features & \#classes\\
\midrule
\textsc{ZINC-12k} & 12,000 & $\sim$23.2 & $\sim$49.8 & 1 & 1\\
\midrule
\textbf{TUDatasets} \\
$\,$ \textsc{MUTAG} & 188 & $\sim$17.9 & $\sim$39.6 & 7 & 2\\
$\,$ \textsc{PROTEINS} & 1,113 & $\sim$39.1 & $\sim$145.6 & 3 & 2\\
$\,$ \textsc{PTC} & 344 & $\sim$14.2 & $\sim$14.6 & 18 & 2\\
$\,$ \textsc{NCI1} & 4,110 & $\sim$29.8 & $\sim$32.3 & 37 & 2\\
$\,$ \textsc{NCI109} & 4,127 & $\sim$29.6 & $\sim$32.1 & 38 & 2\\
\midrule
\textbf{OGB} \\
$\,$ \textsc{molesol} & 1,128 & $\sim$13.3 & $\sim$13.7 & 9 & 1\\
$\,$ \textsc{moltox21} & 7,831 & $\sim$18.6 & $\sim$19.3 & 9 & 2\\
$\,$ \textsc{molbace} & 1,513 & $\sim$34.1 & $\sim$36.9 & 9 & 2\\
$\,$ \textsc{molhiv} & 41,127 & $\sim$25.5 & $\sim$27.5 & 9 & 2\\
\bottomrule
\end{tabular}
\label{tab:graph_cls_stats}
\end{table*}

\begin{table*}[t]
    \centering
    \scriptsize
    \caption{Hyperparameters and search ranges/values for TUDatasets~\citep{morris2019weisfeiler}, OGB~\citep{hu2020ogb}, and \textsc{ZINC-12k}~\citep{dwivedi2022graph} datasets.}
    \begin{tabular}{l c  c  c}
    \toprule
    Hyperparameter & TUDatasets & OGB & \textsc{ZINC}\\
    \midrule
    Learning rate for $\textsc{GNN}_{\textsc{layer}}^{(l)}$ & 
    \multicolumn{3}{c}{$[10^{-5}, 10^{-4}, 10^{-3}, 5 \times 10^{-3}]$}\\
    Learning rate for $\theta^{(l)}$/$ \textsc{GNN}_{\textsc{act}}$ & 
    \multicolumn{3}{c}{$[5 \times 10^{-6}, 10^{-5}, 10^{-4}, 10^{-3}, 5 \times 10^{-3}]$}\\
    Weight decay for $\textsc{GNN}_{\textsc{layer}}^{(l)}$ & 
    \multicolumn{3}{c}{$[10^{-5}, 10^{-4}, 5 \times 10^{-3}, 0.0]$}\\
    Weight decay for $\theta^{(l)}$/$ \textsc{GNN}_{\textsc{act}}$ & 
    \multicolumn{3}{c}{$[10^{-5}, 10^{-4}, 5 \times 10^{-3}, 0.0]$}\\
     $C$ & $[16, 32]$ & $[64, 128]$ & $[64, 128, 256]$\\
     $C_{\textsc{act}}$ & $[16, 32, 64, 128]$ & -- & --\\
     $L$ & $[4, 6]$ & $[2, 4, 6]$ & $[2, 4]$\\
    $p$ & \multicolumn{3}{c}{$[0.0, 0.5]$}\\
    $\mathcal{N}_{\mathcal{P}}$ & \multicolumn{3}{c}{$[2, 4, 8, 16]$}\\
    $\lambda$ & \multicolumn{3}{c}{$[0.0, 10^{-3}, 10^{-2}, 1.0]$}\\
    Graph pooling layer & \multicolumn{3}{c}{[sum, mean]}\\
    Batch size & $[32, 128]$ & $[64, 128]$ & $[64, 128]$\\
    \bottomrule
    \end{tabular}
    \label{tab:hyperparameters_combined}
\end{table*}

\section{Complexity and Runtimes}
\label{sec:app:computations}
\paragraph{Time Complexity.} We now provide an analysis of the time complexity of \ourmethod. Let us recall the following details: (i) As described in \Cref{eq:application_digraf}, \ourmethod\ is applied element-wise in parallel for each dimension of the output of $\textsc{GNN}_{\textsc{layer}}^{(l)}$. (ii) As described in \Cref{eqn:digaf}, we employ an additional GNN denoted by $\textsc{GNN}_{\textsc{act}}$ to compute $\vtheta^{(l)}$. In all our experiments, both the backbone GNN and $\textsc{GNN}_{\textsc{act}}$ are message-passing neural networks (MPNNs)~\citep{gilmer2017neural}. (iii) As described in Theorem 2 of \citet{freifeld2017transformations}, for 1-dimensional domain, there exists a closed form for $T^{(l)}(\cdot; \vtheta^{(l)})$, and the complexity for the CPAB computations are linear with respect to the tesselation size, which is a constant of up to 16 in our experiments. Therefore, using \ourmethod\ with any linear complexity (with respect to the number of nodes and edges) MPNN-based backbone maintains the linear complexity of the backbone MPNN. 
Put precisely, each MPNN layer has linear complexity in the number of nodes $|V|$ and $|E|$. We use $L_{\textsc{act}}$ layers in $\textsc{GNN}_{\textsc{act}}$, the computational complexity of a  \ourmethod\  layer is $\gO(L_{\textsc{act}} \cdot (|V| + |E|))$. Since we have $L$ layers in overall GNN, the computational complexity of an MPNN-based GNN coupled with \ourmethod\ is $\gO(L \cdot L_{\textsc{act}} \cdot (|V| + |E|))$. In our experiments, we fix the hyperparameter $L_{\textsc{act}} = 2$, resulting in $\gO(L \cdot (|V| + |E|))$ computational complexity in practice.

\revision{\paragraph{Memory Complexity.} \ourmethod\ uses $\textsc{GNN}_{\textsc{act}}$  which is an MPNN and hence has linear space complexity (with respect to the number of nodes and edges). CPAB computations require constant memory with respect to the graph size for a 1-dimensional domain due to the analytical implementation. We use $L$ layers in overall GNN and $L_{\textsc{act}}$ layers in $\textsc{GNN}_{\textsc{act}}$ resulting in a memory complexity of $\gO(L \cdot L_{\textsc{act}} \cdot (|V| + |E|))$. In our experiments, we fix the hyperparameter $L_{\textsc{act}} = 2$, resulting in $\gO(L \cdot (|V| + |E|))$ memory complexity in practice.} 

\paragraph{Runtimes.} Despite having linear computational complexity in the size of the graph, \ourmethod\ performs additional computations to obtain $\vtheta^{(l)}$ using $\textsc{GNN}_{\textsc{act}}$. To understand the impact of these computations, we measured the training and inference times of \ourmethod\ and present it in \Cref{tab:runtimes}. Specifically, we report the average time per batch and standard deviation of the same measured on an NVIDIA A5000 GPU, using a batch size of 128. For a fair comparison, we use the same number of layers, batch size, and channels in all methods. Additionally, for our \ourmethod, we set the number of layers within $\textsc{GNN}_\textsc{act}$ to $L_{\textsc{act}} = 2$, and the embedding dimension to $C_{\textsc{act}} = 64$. Our analysis indicates that while \ourmethod\ requires additional computational time, it yields significantly better performance. For example, compared to the best activation function on the dataset, namely Maxout, \ourmethod\ requires an additional $\sim 6.21$ms at inference, but results in a relative improvement in the performance of $\sim 17.95\%$. 

\revision{On the ZINC dataset, using GIN as the primary model, \ourmethod\ exhibits approximately 4.5 times slower training times and 3.5 times slower inference times compared to ReLU. \ourmethod\ demonstrates an inference time that is approximately 1.35 times faster than GReLU, while also achieving superior performance.}

\begin{table*}[t]
    \centering
    \scriptsize
    \caption{Batch runtimes on an NVIDIA RTX A5000 GPU of \ourmethod\ and other activation functions, with 4 GNN layers, batch size 128, 64 embedding dimensions, and $\textsc{GNN}_\textsc{act}$ with $L_{\textsc{act}} = 2$ layers and  $C_{\textsc{act}} = 64$ embedding dimension, on \textsc{ZINC-12k} dataset.}
    \begin{tabular}{l ccc}
    \toprule
    \multirow{2}{*}{Method} &  \multicolumn{3}{c}{\textsc{ZINC}} \\
    & Training time (ms) & Inference time (ms) & \textsc{(MAE $\downarrow$)} \\
    \midrule
    \textsc{GIN} + ReLU~\citep{xu2019how} & 4.18$\pm$0.10	&	2.47$\pm$0.08 & 0.1630$\pm$0.0040\\
    \midrule
    \textsc{GIN} + Maxout~\citep{pmlr-v28-goodfellow13} & 4.71$\pm$0.13	&	2.41$\pm$0.12 & 0.1587$\pm$0.0057\\
    \textsc{GIN} + Swish~\citep{ramachandran2017searching} & 4.55$\pm$0.12	&	2.30$\pm$0.24 & 0.1636$\pm$0.0039\\
    \midrule
    \textsc{GIN} + Max~\citep{iancu2020graph} &
    9.19$\pm$0.25	&	4.50$\pm$0.93 & 0.1661$\pm$0.0035\\
    \textsc{GIN} + Median~\citep{iancu2020graph} & 14.54$\pm$1.35	&	10.13$\pm$1.20 & 0.1715$\pm$0.0050\\
    \textsc{GIN} + GReLU~\citep{zhang2022graph} & 20.63$\pm$0.99	&	11.69$\pm$2.79 & 0.3003$\pm$0.0086\\
    \midrule
     \textsc{GIN} + \ourmethodshared & 13.76$\pm$0.65	&4.97$\pm$1.72 & 0.1382$\pm$0.0080\\
    \textsc{GIN} + \ourmethod & 
    19.37$\pm$1.28 & 8.62$\pm$0.18 & 0.1302$\pm$0.0090\\
        \bottomrule
    \end{tabular}
    \label{tab:runtimes}
\end{table*}

\newpage
\clearpage
\section*{NeurIPS Paper Checklist}

\begin{enumerate}

\item {\bf Claims}
    \item[] Question: Do the main claims made in the abstract and introduction accurately reflect the paper's contributions and scope?
    \item[] Answer: \answerYes{} % Replace by \answerYes{}, \answerNo{}, or \answerNA{}.
    \item[] Justification: See \Cref{sec:introduction} and 
    \Cref{sec:experiments}.
    %\ipsit{TODO: Introduction does not talk about limitations}
    %\justificationTODO{}
    % Clear

\item {\bf Limitations}
    \item[] Question: Does the paper discuss the limitations of the work performed by the authors?
    \item[] Answer: \answerYes{} % Replace by \answerYes{}, \answerNo{}, or \answerNA{}.
    \item[] Justification: See \Cref{sec:conclusion} and \Cref{sec:app:computations}.
    % \justificationTODO{}

\item {\bf Theory Assumptions and Proofs}
    \item[] Question: For each theoretical result, does the paper provide the full set of assumptions and a complete (and correct) proof?
    \item[] Answer: \answerYes{} % Replace by \answerYes{}, \answerNo{}, or \answerNA{}.
    \item[] Justification: We introduce the idea in \Cref{sec:theoritical_analysis}, with additional propositions and detailed proofs show in \Cref{sec:appendix:proofs}.
    

    \item {\bf Experimental Result Reproducibility}
    \item[] Question: Does the paper fully disclose all the information needed to reproduce the main experimental results of the paper to the extent that it affects the main claims and/or conclusions of the paper (regardless of whether the code and data are provided or not)?
    \item[] Answer: \answerYes{} % Replace by \answerYes{}, \answerNo{}, or \answerNA{}.
    \item[] Justification: See \Cref{sec:experiments}, \Cref{sec:app:additional_results}, \Cref{sec:appendix:implementation} and \Cref{sec:appendix:experiment}. %\justificationTODO{}

\item {\bf Open access to data and code}
    \item[] Question: Does the paper provide open access to the data and code, with sufficient instructions to faithfully reproduce the main experimental results, as described in supplemental material?
    \item[] Answer: \answerYes{} % Replace by \answerYes{}, \answerNo{}, or \answerNA{}.
    \item[] Justification: We provide extensive details on the implementation and evaluation of our method, and we released our code.
    %\justificationTODO{}

\item {\bf Experimental Setting/Details}
    \item[] Question: Does the paper specify all the training and test details (e.g., data splits, hyperparameters, how they were chosen, type of optimizer, etc.) necessary to understand the results?
    \item[] Answer: \answerYes{} % Replace by \answerYes{}, \answerNo{}, or \answerNA{}.
    \item[] Justification: See \Cref{sec:appendix:experiment} and \Cref{sec:experiments}. %\justificationTODO{}

\item {\bf Experiment Statistical Significance}
    \item[] Question: Does the paper report error bars suitably and correctly defined or other appropriate information about the statistical significance of the experiments?
    \item[] Answer: \answerYes{} % Replace by \answerYes{}, \answerNo{}, or \answerNA{}.
    \item[] Justification: The standard deviations for all the metrics have been presented. See \Cref{tab:GReLU}, \Cref{tab:zinc}, \Cref{tab:ogb}, \Cref{tab:tud_datasets},
    \Cref{tab:cap_comparision},
    \Cref{tab:ablation:depth}, \Cref{tab:ablation:regularization},  and \Cref{tab:runtimes}. Also see \Cref{sec:experiments}, \Cref{sec:app:additional_results}.

\item {\bf Experiments Compute Resources}
    \item[] Question: For each experiment, does the paper provide sufficient information on the computer resources (type of compute workers, memory, time of execution) needed to reproduce the experiments?
    \item[] Answer: \answerYes{} % Replace by \answerYes{}, \answerNo{}, or \answerNA{}.
    \item[] Justification: We discuss the computational resources in \Cref{sec:appendix:experiment} and \Cref{sec:app:computations}.
    %\justificationTODO{}
    
\item {\bf Code Of Ethics}
    \item[] Question: Does the research conducted in the paper conform, in every respect, with the NeurIPS Code of Ethics \url{https://neurips.cc/public/EthicsGuidelines}?
    \item[] Answer: \answerYes{} % Replace by \answerYes{}, \answerNo{}, or \answerNA{}.
    \item[] Justification: See \Cref{sec:appendix:experiment}. 

\item {\bf Broader Impacts}
    \item[] Question: Does the paper discuss both potential positive societal impacts and negative societal impacts of the work performed?
    \item[] Answer: \answerYes{} % Replace by \answerYes{}, \answerNo{}, or \answerNA{}.
    \item[] Justification: We discuss the societal impact in the conclusion, see \Cref{sec:conclusion}.
%\justificationTODO{}
    
\item {\bf Safeguards}
    \item[] Question: Does the paper describe safeguards that have been put in place for responsible release of data or models that have a high risk for misuse (e.g., pretrained language models, image generators, or scraped datasets)?
    \item[] Answer: \answerNA{} % Replace by \answerYes{}, \answerNo{}, or \answerNA{}.
    \item[] Justification: We don't scrape any datasets not we release any models of high risk for misuse. %\justificationTODO{}

\item {\bf Licenses for existing assets}
    \item[] Question: Are the creators or original owners of assets (e.g., code, data, models), used in the paper, properly credited and are the license and terms of use explicitly mentioned and properly respected?
    \item[] Answer: \answerYes{} % Replace by \answerYes{}, \answerNo{}, or \answerNA{}.
    
    \item[] Justification: We cited each datasets we used in \Cref{sec:experiments} and \Cref{sec:appendix:experiment}. The license of the packages we used is in \Cref{sec:appendix:experiment}.

\item {\bf New Assets}
    \item[] Question: Are new assets introduced in the paper well documented and is the documentation provided alongside the assets?
    \item[] Answer: \answerNA{} % Replace by \answerYes{}, \answerNo{}, or \answerNA{}.
    \item[] Justification: We don't release any new assets. %\justificationTODO{}

\item {\bf Crowdsourcing and Research with Human Subjects}
    \item[] Question: For crowdsourcing experiments and research with human subjects, does the paper include the full text of instructions given to participants and screenshots, if applicable, as well as details about compensation (if any)? 
    \item[] Answer: \answerNA{} % Replace by \answerYes{}, \answerNo{}, or \answerNA{}.
    \item[] Justification: The paper does not involve crowdsourcing nor research with human subjects.
%\justificationTODO{}

\item {\bf Institutional Review Board (IRB) Approvals or Equivalent for Research with Human Subjects}
    \item[] Question: Does the paper describe potential risks incurred by study participants, whether such risks were disclosed to the subjects, and whether Institutional Review Board (IRB) approvals (or an equivalent approval/review based on the requirements of your country or institution) were obtained?
    \item[] Answer: \answerNA{} % Replace by \answerYes{}, \answerNo{}, or \answerNA{}.
    \item[] Justification: We don't involve crowdsourcing nor research with human subjects. %\justificationTODO{}

\end{enumerate}

\end{document}